\documentclass[twoside,11pt]{article}

\usepackage[abbrvbib,preprint]{jmlr2e}
\usepackage{xcolor}
\usepackage{amsmath,amsfonts}
\usepackage{mathrsfs}
\usepackage{bm}
\usepackage{caption}
\usepackage{subcaption}
\usepackage{placeins}
\usepackage{tabularx,array,booktabs}
\renewcommand{\d}{\;\mathrm{d}}
\newcommand{\di}{\mathrm{d}}
\newcommand{\mat}[1]{\mathbf{#1}}

\newcommand{\E}{\mathbb{E}}
\newcommand{\R}{\mathbb{R}}
\newcommand{\Var}{\operatorname{Var}}
\newcommand{\argmin}{\operatorname*{argmin}}

\newcommand{\fs}{\mathcal Z}
\newcommand{\xspace}{\mathcal X}
\newcommand{\yspace}{\mathcal Y}
\newcommand{\hs}{\mathcal H}
\newcommand{\as}{\mathcal U}
\newcommand{\LR}{\varrho}
\newcommand{\bx}{\mat{x}}
\newcommand{\bX}{\mat{X}}

\newcommand{\cN}{\mathcal N}
\newcommand{\cX}{\mathcal X}

\newcommand{\pS}{p_S}

\newcommand{\op}{o_P}
\newcommand{\Op}{O_P}
\newcommand{\dto}{\xrightarrow{\;d\;}}

\newtheorem{assumption}{Assumption}
\newcommand{\ZH}{\mathcal Z_{\mathrm H}}
\newcommand{\dd}{\;\mathrm{d}}

\makeatletter
\def\ps@arxivtitle{%
  \let\@mkboth\@gobbletwo
  \def\@oddhead{}%
  \def\@evenhead{}%
  \def\@oddfoot{\hfill\small\rm\thepage\hfill}%
  \def\@evenfoot{\hfill\small\rm\thepage\hfill}}
\makeatother

\ShortHeadings{DA via Doubly-Anchored Distributionally Robust Optimization}{Kepplinger and Vidyashankar}
\firstpageno{1}

\begin{document}

\title{Domain Adaptation with Target Information via Doubly-Anchored Distributionally Robust Optimization}

\author{%
    \name David Kepplinger \textnormal{\small (corresponding author)} \email dkepplin@gmu.edu \\
    \addr Department of Statistics\\
    George Mason University\\
    Fairfax, VA 22030, USA
    \AND
    \name Anand N.\ Vidyashankar \email avidyash@gmu.edu \\
    \addr Department of Statistics\\
    George Mason University\\
    Fairfax, VA 22030, USA}

\maketitle
\thispagestyle{arxivtitle}

\begin{abstract}%
Domain Adaptation (DA) often lacks worst-case guarantees, whereas source-centered Distributionally Robust Optimization (DRO) ignores available target structure.
We introduce a doubly-anchored DRO framework that intersects two $\phi$-divergence balls centered at the source law and a source-completed target reference combining the target covariate law with the source conditional law.
A general dual representation yields adversarial bridge geometries for KL/KL, squared-Hellinger/squared-Hellinger (HD/HD), and asymmetric KL/HD pairings.
For symmetric pairings, classical barycentric identities give effective single-ball problems whose centers reweight the source law. When the total multiplier is positive, the relative target-side multiplier selects the tempering coordinate of a minimizing ball.
KL/HD instead produces a Lambert-$W$ bridge coupling exponential loss tilting with Hellinger attenuation.
Without assuming covariate shift, we establish population locality and uniform finite-sample target-risk bounds for estimated structural bridges.
The bounds separate conditional shift, bridge-to-target discrepancy, effective sample size under the oracle weights, and likelihood-ratio estimation error. Symmetric radius profiles characterize the admissible structural centers. We give finite-sample conditions that exclude both endpoints from minimizers of a deterministic risk-bound envelope.
For bridge-weighted Nadaraya--Watson regression, we establish uniform consistency for the source regression function and pointwise asymptotic normality. These results apply to the target regression function when the two regression functions coincide. A bounded-overlap simulation displays these quantities, and a Fashion-MNIST study illustrates the bridges in higher dimensions.
\end{abstract}

\begin{keywords}
  unsupervised domain adaptation,
  transfer learning, %
  distributionally robust optimization, %
  $\phi$-divergence, %
  robustness, %
  density ratio estimation, %
  non-parametric regression, %
  covariate shift
\end{keywords}

\section{Introduction}
The deployment of machine learning models in high-stakes environments is frequently bottlenecked by distribution shift.
While classical Domain Adaptation (DA) techniques seek to align feature representations across domains, they often lack worst-case robustness guarantees.
Conversely, standard Distributionally Robust Optimization (DRO) typically anchors its ambiguity set at the empirical source distribution \citep[e.g.,][]{BenTal2013,HuHong2012,MohajerinEsfahani2018,duchi2021}.
The direct DRO envelope bounds the target risk whenever the target law lies in that source-centered set.

In the source-supervised, unlabeled target setting considered here, source-anchored DRO thus fails to exploit the structural and distributional information available through the target covariate law.
To fill this gap, we introduce Robust Domain Adaptation via a doubly-anchored ambiguity set, a framework that directly integrates the available target-domain structure into the robust optimization geometry.

We observe labeled data from the source domain and unlabeled covariates from the target domain, potentially collected under different sampling and measurement conditions. The source supplies supervised information, while the target marginal supplies target-domain structure.
We therefore center one ambiguity ball at the source law $P_S$ and a second at the source-completed target reference law
$P_T^\circ=P_T^X\otimes P_S^{Y\mid X}$, with radii $\rho_S$ and $\rho_T$, respectively.
The borrowed conditional is only the supervised center of the target-side ball; departures of $P_T^{Y\mid X}$ from $P_S^{Y\mid X}$ are permitted through the target radius.
We assume that the true target law is close to both population anchors, so the target constraint restricts source-induced worst cases to target-compatible laws.
The second anchor has three consequences.
First, the structural bridges in Section~\ref{sec:loss-agnostic-density-bridges} place several likelihood-ratio weighting schemes on a common geometric footing. They span a continuum from source-only learning to raw importance weighting. Using Kullback--Leibler divergence (KL) for both constraints recovers the flattened weights $\LR^\alpha$ of \citet{Shimodaira2000,Sugiyama2007}. Using squared Hellinger distance (HD) for both recovers the square-root interpolated weights of \citet{Kimura2022}. When the sum of the optimal multipliers is positive, the relative target-side multiplier selects a coordinate on the corresponding path (Section~\ref{sec:symmetric-single-ball-geometry}).
Second, feasibility of the true target law makes the robust value an upper bound on target risk. Under this assumption, the target radius governs how far a symmetric structural center can move toward the source. The finite-sample bounds quantify the trade-off between bridge-to-target discrepancy and the effective sample size under the oracle weights (Section~\ref{sec:bridge-risk-bounds}).
Third, comparing the loss-aware bridges with their loss-agnostic counterparts shows how the supervised loss and the source--target structure interact through the chosen geometry: they factorize the KL/KL and HD/HD bridges and are coupled through the Lambert-$W$ function for KL/HD (Section~\ref{sec:bridge-geometries}).

A fundamental requirement for our doubly-anchored bridge is adequate overlap between the source and target covariate laws.
We assume that $P_S^X$ and $P_T^X$ are mutually absolutely continuous.
This prevents target-relevant covariate regions from having zero source support and ensures the existence of the population density ratio
$\LR(x)=p_T^X(x)/p_S^X(x)$.
However, while overlap makes transfer possible, the raw density ratio is highly susceptible to estimation error.
We therefore impose a controlled discrepancy through the DRO bridge to make this transfer useful and robust.
The finite-sample bounds additionally require uniform bridge-weight regularity; for the canonical bridge maps, this condition is verified under bounded likelihood-ratio overlap and hence does not cover an unbounded likelihood ratio.

The radii control how far the true target law can lie from each anchor and must be large enough for it to be feasible (Assumption~\ref{ass:target-feasibility}). They also control how much source borrowing is admissible. Under target feasibility, the source radius never truncates either symmetric structural path from $P_S$ to $P_T^\circ$, whereas the target radius determines the lower endpoint of the admissible range of bridge coordinates (Appendix~\ref{sec:radius-indexed-selection}). In the loss-aware problem, the source constraint retains a role through the residual robust budget around each bridge center, which can lower the robust value even when the full structural path is admissible (Section~\ref{sec:symmetric-single-ball-geometry}). We describe the endpoint behavior of the loss-aware responses separately for each geometry in Section~\ref{sec:bridge-geometries}.

In this work we focus on the Kullback--Leibler (KL) divergence and squared Hellinger distance (HD) to construct the two ambiguity sets.
Our choice is motivated by both their classical theoretical foundations and their distinct operational behaviors.
The minimum-divergence robustness literature motivates examining an HD-based ambiguity geometry under misspecification and outliers \citep{Beran1977,Basu1998,lindsay1994}.
In our construction, the pair of divergences determines how large likelihood ratios enter the weights. Specifically, the likelihood ratio $\LR$ enters polynomially as $\LR^\alpha$ in the KL/KL pairing, linearly in HD/HD, and linearly up to a squared-logarithmic attenuation in KL/HD (Section~\ref{sec:bridge-geometries}). The tempering of high ratios is therefore induced by the geometry itself and not imposed by truncation.
At the joint-law level, this distinction also matters under support mismatch: if, on a set of target covariates of positive probability, $P_T^{Y\mid X}$ assigns positive mass to a measurable label set to which $P_S^{Y\mid X}$ assigns zero mass, then $D_{\mathrm{KL}}(P_T\Vert P_S)=\infty$, whereas the extended squared Hellinger divergence remains finite.
Consequently, HD/HD can remain feasible in settings excluded by a source-side KL ball.
Using these divergence pairings, we uncover specific, analytically tractable bridge geometries.
The loss-agnostic structural bridges associated with all three pairings can be parameterized by the single mixing parameter $\alpha$, streamlining both theoretical analysis and algorithmic implementation.
The asymmetric KL/HD pairing, which uses KL for the source and HD for the target, yields a distinct Lambert-$W$ bridge geometry.
In the loss-aware formulation of the asymmetric bridge, the source KL divergence drives an exponential tilting toward adversarial risk, while the target HD enters through a square-root attenuation of the likelihood ratio.
This asymmetric configuration is motivated by settings in which the source domain is trustworthy, but the target covariate sample is susceptible to corruption or noise.

Importance-weighting methods often assume covariate shift \citep{Shimodaira2000} or label shift \citep{Lipton2018}, which posit invariance of the label conditional or the covariate conditional, respectively.
More general target, conditional, and generalized target-shift models have also been studied \citep{Zhang2013}.
Our framework does not impose a conditional-invariance equality.
It requires overlap and feasibility of the true target law within the doubly-anchored ambiguity set, with departures from the source-completed target center controlled through the divergence radii.
Because target labels are unavailable, target feasibility is a modeling condition rather than a property verifiable from the observed target sample alone.

Within this doubly anchored framework, our contributions are threefold.
First, we derive a general dual representation and closed-form KKT geometries for KL/KL, HD/HD, and KL/HD pairings.
For the symmetric pairings, exact identities along the classical exponential and normalized square-root barycentric paths represent the two-ball intersection as a continuum of effective single-ball problems whose centers are the structural bridges and whose radii are weighted admissibility slacks.
Under target feasibility, the source radius permits the entire symmetric structural path and still enters the residual robust budget.
In particular, the flattening tuning parameter of \citet{Shimodaira2000} and the square-root interpolation of \citet{Kimura2022} are the specific coordinates of our KL/KL and HD/HD bridges, respectively, and a positive sum of optimal multipliers identifies a minimizing coordinate through the relative target-side multiplier.
For the asymmetric KL/HD pairing, the different source and target penalties instead produce a geometry in which exponential loss tilting and Hellinger attenuation are coupled through the Lambert-$W$ function.

Second, under uniform bridge-weight regularity, we establish population locality and finite-sample target-risk bounds for hypotheses minimizing the empirical structural-bridge risk.
The empirical-Bernstein refinement of the ordinary empirical-process bound adapts to the effective sample size of the realized oracle weights. The error from estimating the likelihood ratio enters separately.
The symmetric radius profiles then characterize bridge admissibility and provide sufficient conditions that exclude the two endpoints from minimizers of a deterministic risk-bound envelope.
For fixed covariate laws under target feasibility, the admissible structural interval is $[\alpha_-(\rho_T),1]$. Its width is of order $\sqrt{\rho_T}$ as $\rho_T\downarrow0$, with a coefficient determined by the covariate shift.

Third, we show that uniform likelihood-ratio error propagates uniformly through the three canonical bridge maps and give density-estimation conditions that verify the required plug-in control.
For the resulting bridge-weighted Nadaraya--Watson estimator, we obtain uniform consistency and pointwise central limit theorems for the oracle and plug-in weights. Section~\ref{sec:kernel-regression} states the regression results, and Appendices~\ref{app:asymptotics-loss-agnostic-bridge-weights} and~\ref{app:asymptotics-nw} give the assumptions and proofs. A theorem-linked bounded-overlap simulation displays the radius profiles, admissible bridge range, effective sample size, and deterministic risk-bound envelope. A secondary domain-shifted Fashion-MNIST study \citep{Xiao2017} illustrates the same geometries in a higher-dimensional application.

The results below have complementary scopes.
The duality and KKT representations concern the anchor-dominated loss-aware problem, which coincides with the full problem for source-side KL, while the population locality result concerns the full ambiguity set.
The finite-sample guarantees apply to the loss-agnostic structural bridge under bounded overlap and uniform weight regularity; the empirical loss-aware iteration uses the corresponding dual response maps as a computational refinement.
Under target feasibility, the source constraint admits every center on either symmetric path and can still tighten the loss-aware robust value through the effective single-ball radii.

\subsection{Example: a no-covariate toy model}\label{sec:toy-example}
To build intuition for how the doubly-anchored ambiguity set enables knowledge transfer before introducing covariate geometries, we first consider a simpler regime in which labels are observed in both domains and there are no covariates.
This example is not the target-unlabeled setting studied in the remainder of the paper and is an analogy rather than a specialization of the doubly-anchored optimization problem; it isolates the source-borrowing tradeoff in its simplest form.
For $D\in\{S,T\}$, let $\mu_D=\E_{P_D}[Y]$ and $\sigma_D^2=\Var_{P_D}(Y)$.
We consider the problem of estimating the target mean $\mu_T$; related source--target and multi-source weighting arguments appear in Theorem~3 of \citet{ben-david2010} and in \citet{Mansour2009a}.
Suppose $(Y_1^S,\dots,Y_n^S)$ and $(Y_1^T,\dots,Y_m^T)$ are mutually independent i.i.d.\ samples from the source and target domains, respectively.
From these samples we compute the empirical means $\bar Y_S=n^{-1}\sum_{i=1}^nY_i^S$ and $\bar Y_T=m^{-1}\sum_{i=1}^mY_i^T$, with respective variances $v_S=\sigma_S^2/n$ and $v_T=\sigma_T^2/m$.
Among estimators $\widehat\mu_\alpha=(1-\alpha)\bar Y_S+\alpha\bar Y_T$, $\alpha\in[0,1]$, suppose that the transfer uncertainty between the true domain means is bounded by $|\mu_T-\mu_S|\leq\Delta$.
For fixed $v_S$ and $v_T$, the worst-case mean squared error over the admissible mean gap $d=\mu_T-\mu_S$ is
\[
    \sup_{|d|\leq\Delta}
    \left\{(1-\alpha)^2(v_S+d^2)+\alpha^2v_T\right\}
    =
    (1-\alpha)^2(v_S+\Delta^2)+\alpha^2v_T,
\]
and, provided $v_S+v_T+\Delta^2>0$, is minimized by the oracle target weight $\alpha^\star=(v_S+\Delta^2)/(v_T+v_S+\Delta^2)$.
The adapted predictor is then the convex combination
$\hat\mu_{\mathrm{br}}=(1-\alpha^\star)\bar Y_S+\alpha^\star\bar Y_T$.
This exposes the core mechanism of adaptation: drawing upon source data reduces estimation variance, while the source--target mean gap, bounded by $\Delta$, introduces bias.

To connect this transfer uncertainty to a divergence radius, consider the case where both domains are Gaussian with shared variance $\sigma^2$.
Then
\[
    D_{\mathrm{KL}}(P_T\Vert P_S)
    =\frac{(\mu_T-\mu_S)^2}{2\sigma^2}.
\]
Thus $D_{\mathrm{KL}}(P_T\Vert P_S)\leq\rho$ implies $(\mu_T-\mu_S)^2\leq2\sigma^2\rho$.
Taking the tight radius-induced uncertainty bound $\Delta^2=2\sigma^2\rho$ in the preceding worst-case calculation yields
\[
    \alpha^\star=\frac{v_S+2\sigma^2\rho}{v_T+v_S+2\sigma^2\rho}.
\]
This formulation shows how the divergence radius enters the bias--variance geometry.
In this no-covariate mean problem, the distributional uncertainty is summarized by the scalar bridge weight $\alpha^\star$, which is determined by the two sample-mean variances and the radius-induced mean-gap bound.
If the divergence radius $\rho=0$, the domains are assumed identical and the estimator pools the two samples according to their relative precision, placing greater weight on the source whenever the source estimate is more precise.
Conversely, as $\rho\to\infty$, the source domain is uninformative for estimating $\mu_T$, and hence $\alpha^\star\to1$, so that the estimator collapses to the target-only baseline.

Analyzing the asymptotic scaling of this optimal weight reveals that if the mean-gap bound $\Delta^2$ remains fixed and strictly positive as the sample sizes grow, the target data eventually dominate and the source weight vanishes.
Hence, source borrowing can help in finite samples but disappears asymptotically at first order.
Related finite-sample source-borrowing regimes are studied for the confidence-thresholding estimator in \citet{Cai2024}, while interactions between ambiguity-radius and sample-size scaling in DRO are studied by \citet{blanchet2023}.
Non-degenerate source--target weighting occurs when $v_S+\Delta^2\asymp v_T$; when $v_S=O(v_T)$, this includes local-gap regimes satisfying $\Delta^2\asymp\sigma_T^2/m$.
This fundamental dynamic, balancing source sample precision against the divergence radius of the ambiguity set, motivates the localized likelihood-ratio weights used when we reintroduce covariates in non-parametric regression.

\subsection{Setup and assumptions}\label{sec:setup-and-assumptions}

To formalize our framework, let $\xspace$ be a Borel subset of $\R^d$, let $\yspace$ be a standard Borel space, and equip the data space $\fs=\xspace\times\yspace$ with the product $\sigma$-field.
Let $\hs$ be a class of measurable maps $h\colon\xspace\to\yspace$, and let $\ell\colon\yspace\times\yspace\to\R_+$ be a measurable non-negative loss.
Let $P_S$ and $P_T$ denote the true source and target laws on $\fs$. For $z=(x,y)\in\fs$ and $h\in\hs$, write
\[
    L_h(z)=\ell\{h(x),y\}.
\]
We observe mutually independent samples
$\{Z_i^S=(X_i^S,Y_i^S):1\leq i\leq n\}$ drawn i.i.d.\ from $P_S$ and
$\{X_j^T:1\leq j\leq m\}$ drawn i.i.d.\ from $P_T^X$, so that the source
sample is labeled and the target sample contains only covariates.
Write
\[
    P_D(\mathrm{d} x,\mathrm{d} y)=P_D^X(\mathrm{d} x)P_D^{Y\mid X}(\mathrm{d} y\mid x),
    \qquad D\in\{S,T\}.
\]
Because the target conditional law is not directly observed, we define the source-completed target reference law
\begin{equation*}
    P_T^\circ(\mathrm{d} x, \mathrm{d} y)
    =P_T^X(\mathrm{d} x)P_S^{Y\mid X}(\mathrm{d} y\mid x).
\end{equation*}
The law $P_T^\circ$ uses the target covariate law and the available source conditional law, and will be used as the center of the target-side ambiguity set defined below.
It does not impose the covariate-shift condition $P_T^{Y\mid X}=P_S^{Y\mid X}$; departures of the true target conditional from this center are permitted through the divergence radius.

Let $\hat P_S$ and $\hat P_T^\circ$ denote regularized empirical anchors constructed from the two samples:
\[
    \hat P_S(\mathrm{d} x,\mathrm{d} y)
    =\hat P_S^X(\mathrm{d} x)\hat P_S^{Y\mid X}(\mathrm{d} y\mid x),
    \qquad
    \hat P_T^\circ(\mathrm{d} x,\mathrm{d} y)
    =\hat P_T^X(\mathrm{d} x)\hat P_S^{Y\mid X}(\mathrm{d} y\mid x).
\]
Throughout, an empirical anchor means a regularized distributional estimate admitting a density on the working region, rather than the raw atomic empirical measure; kernel density estimation provides one construction used later in the paper.
Although $\hat P_T^\circ$ is written as a joint reference law, its density ratio relative to $\hat P_S$ depends only on the source and target covariate marginals.

Whenever density notation is used for a pair of anchors, their densities are taken with respect to a common dominating measure.
In particular, let $\mu$ be a common dominating measure for $P_S$ and $P_T^\circ$, and let $\mu_X$ be a common dominating measure for $P_S^X$ and $P_T^X$; write $p_S,p_T^\circ$ and $p_S^X,p_T^X$ for the corresponding densities.

Let $\mathcal P(\fs)$ denote the probability measures on $\fs$.
For a proper lower-semicontinuous convex divergence generator
$\phi\colon[0,\infty)\to[0,\infty]$ satisfying $\phi(1)=0$,
define its recession slope by
\[
    \phi^\infty(1)
    =
    \lim_{t\to\infty}\frac{\phi(t)}{t}
    \in[0,\infty].
\]
If
\[
    Q=Q_P^{a}+Q_P^{s}
\]
is the Lebesgue decomposition of $Q$ relative to $P$, the extended
$\phi$-divergence is
\begin{equation}\label{eqn:extended-phi-divergence}
    D_\phi(Q\Vert P)
    =
    \int_\fs
    \phi\left(\frac{\mathrm{d} Q_P^{a}}{\mathrm{d} P}\right)\d P
    +
    \phi^\infty(1)\,Q_P^{s}(\fs),
\end{equation}
with the convention $0\cdot\infty=0$.
Thus, the divergence is defined even when $Q\not\ll P$.

For two anchor laws $A_S$ and $A_T$, divergence generators $\phi_S$ and $\phi_T$, and finite radii $\rho_S,\rho_T\in[0,\infty)$, define the full doubly-anchored
ambiguity set
\begin{equation*}
\as(A_S,A_T)
=
\left\{
Q\in\mathcal P(\fs):
D_{\phi_S}(Q\Vert A_S)\leq\rho_S \;\wedge\;
D_{\phi_T}(Q\Vert A_T)\leq\rho_T
\right\},
\end{equation*}
and its anchor-dominated subclass
\begin{equation}\label{eqn:dominated-two-anchor-ambiguity-set}
\as_{\mathrm{ac}}(A_S,A_T)
=
\left\{
Q\in\as(A_S,A_T):
Q\ll A_S
\right\}.
\end{equation}
When $A_S$ and $A_T$ are mutually absolutely continuous, the condition
$Q\ll A_S$ in~\eqref{eqn:dominated-two-anchor-ambiguity-set} is
equivalent to $Q\ll A_T$.
Here and below, an anchor is a reference probability law serving as the
center of a divergence constraint; the term describes its role in the
ambiguity-set geometry rather than asserting that the reference law equals
the unknown target law. Accordingly, in the mutually absolutely continuous
setting studied here, ``anchor-dominated'' means that a candidate law is
absolutely continuous with respect to either---and hence both---anchor laws,
not that one anchor dominates the other.

The population and empirical ambiguity sets are obtained by taking
\(
    (A_S,A_T)
    =
    (P_S,P_T^\circ)
\) and
\(
    (A_S,A_T)
    =
    (\hat P_S,\hat P_T^\circ),
\)
respectively. For KL and squared HD, we use
\[
    \phi_{\mathrm{KL}}(t)=t\log t-t+1,
    \qquad
    \phi_{\mathrm{HD}}(t)=(\sqrt t-1)^2,
\]
with the convention $0\log0=0$. Their recession slopes are
\[
    \phi_{\mathrm{KL}}^\infty(1)=\infty,
    \qquad
    \phi_{\mathrm{HD}}^\infty(1)=1.
\]
Consequently, finite source-side KL divergence forces
$Q\ll A_S$, whereas finite squared Hellinger divergence permits a
singular component at finite divergence cost.

To evaluate the worst-case robust risk, we impose the following structural assumptions.

\begin{assumption}[Bounded loss]\label{ass:bounded-loss}
The loss function is bounded such that $0\leq L_h(z)\leq M$ for some
$0<M<\infty$, for all hypotheses $h\in\hs$ and all data points $z\in\fs$.
\end{assumption}

\begin{assumption}[Anchor equivalence and overlap]
\label{ass:common-domination}
\leavevmode\par\noindent
The population anchor laws $P_S$ and $P_T^\circ$ are mutually
absolutely continuous on $\fs$.
In the target-unlabeled construction
\[
    P_T^\circ(\mathrm{d} x,\mathrm{d} y)
    =
    P_T^X(\mathrm{d} x)P_S^{Y\mid X}(\mathrm{d} y\mid x),
\]
this follows from mutual absolute continuity of $P_S^X$ and $P_T^X$
on $\xspace$, because the two joint anchors share the
conditional kernel $P_S^{Y\mid X}$.

Consequently,
\[
    \LR(z)
    =
    \frac{\di P_T^\circ}{\di P_S}(x,y)
    =
    \frac{\di P_T^X}{\di P_S^X}(x)
    =
    \frac{p_T^X(x)}{p_S^X(x)}
\]
is well defined $P_S$-almost surely.
\end{assumption}

\begin{assumption}[Feasibility]\label{ass:target-feasibility}
The true target law belongs to the full population doubly-anchored
ambiguity set:
\[
    P_T\in\as(P_S,P_T^\circ),
\]
or equivalently, $D_{\phi_S}(P_T\Vert P_S)\leq\rho_S$ and $D_{\phi_T}(P_T\Vert P_T^\circ)\leq\rho_T$, where the divergences are understood in the extended sense of
\eqref{eqn:extended-phi-divergence}.
\end{assumption}

Empirical radius calibration for the regularized anchors incorporates their sampling and regularization errors.

\section{Related Literature}

Distributionally Robust Domain Adaptation (DRDA) has emerged as a principled approach to counter the vulnerability of classical DA to unexpected target shifts.
For example, \citet{Awad2023} use Maximum Mean Discrepancy (MMD) \citep{gretton2012} to construct a single ambiguity ball centered at a reweighted empirical source law, with a radius chosen to cover the relevant source-reweighted and target laws with high probability.
Other recent works use Optimal Transport (OT) to align source and target distributions or representations \citep[e.g.,][]{courty2017,Balaji2020}.
In multi-source Unsupervised Domain Adaptation, other frameworks construct uncertainty classes from mixtures of source conditional laws or their estimates, together with target-covariate information \citep{Wang2026,Kim2026,Guo2025}.

These approaches use different discrepancy geometries and impose correspondingly different restrictions on the relation between the source and target laws.
Classical domain-adaptation bounds are formulated through domain discrepancies or importance weights \citep[e.g.,][]{ben-david2010,Mansour2009,Cortes2010,Cortes2019,Kimura2022}, while single-ball robust-learning formulations use one ambiguity-set metric \citep[e.g.,][]{Lee2018,Sinha2018,Awad2023}.
In the reweighting scheme of \citet{Awad2023}, the density-ratio weights are constrained by an imposed bound $B$, and both the prescribed radius and the resulting risk bound scale with $B$.
For covariate-shift regression, \citet{Ma2023} instead truncate the density ratio to ensure boundedness.
To reduce the sensitivity of standard DRO to outliers, \citet{Zhai2021} introduce the ``DORO'' risk, which filters out a small fraction of high-loss observations during training.

Related constructions based on intersecting two uncertainty regions have appeared previously.
In supervised linear-regression domain adaptation, \citet{Taskesen2021} define a moment-based ambiguity class through the intersection of two regions centered at the empirical source and target mean--covariance pairs, using either the same KL-type or the same Wasserstein-type divergence on both constraints.
In contextual optimization under covariate shift, \citet{Wang2024} instead intersect two Wasserstein balls centered at nonparametric and parametric estimators.
Relative to these constructions, our framework intersects two $\phi$-divergence balls on the space of full joint laws.
The first is centered at the source law $P_S$, and the second at the target-reference law $P_T^\circ$ which combines the target covariate marginal with the source conditional law.
The two constraints may use either a common divergence generator, as in the KL/KL and HD/HD pairings, or different generators, as in the KL/HD pairing.
To our knowledge, this full joint-law construction centered at $P_S$ and $P_T^\circ$ has not previously been studied in domain adaptation.
For the KL/HD pairing, in the attained regime, the pointwise KKT equation yields a Lambert-$W$ representation coupling source-side exponential tilting with the target-side Hellinger attenuation.
Unlike density-ratio truncation \citep{Ma2023} or the uniformly bounded target-induced weighting of \citet{Yamamoto2026}, the KL/HD density map remains unbounded but attenuates large likelihood ratios by a logarithmic factor.

Our doubly-anchored bridge geometries also connect with likelihood-ratio transformations used in covariate-shift adaptation and DRDA.
The raw likelihood ratio, corresponding to the target endpoint of the loss-agnostic symmetric bridges introduced below, is widely used to reweight estimators \citep[e.g.,][]{Shimodaira2000,huang2006,Sugiyama2008,Nguyen2010}.
The parameter $\alpha$ in the KL/KL bridge corresponds to the flattening parameter in covariate-shift weighted maximum likelihood
\citep{Shimodaira2000,Sugiyama2007}, but here it follows from the relative dual influence of the two ambiguity constraints.
The bounded interpolated ratio in RuLSIF \citep{yamada2013}, the bounded weights in \citet{Yamamoto2026}, and the truncated or clipped ratios in \citet{Ma2023,Fang2020} temper likelihood-ratio instability through different transformations.
By contrast, \citet{Cortes2010} analyze unbounded importance weights through second-moment control and derive alternative reweightings from a bias--variance bound.
The loss-agnostic HD/HD bridge weight is a specific case of the weight used in the geometrically generalized IWERM framework by \citet{Kimura2022}.
Our HD/HD and KL/HD maps are generally not bounded, but their forms are induced directly by the selected pair of divergence geometries, and Table~\ref{tab:weighting-comparisons} maps them to existing schemes.

\begin{table}[t]
\centering
\caption{Likelihood-ratio weighting schemes on the doubly-anchored bridges.
Here $\LR=p_T^X/p_S^X$, $\beta_\alpha=\alpha\sqrt{\LR}/(1-\alpha)$, and $\alpha$ is the bridge coordinate in Section~\ref{sec:loss-agnostic-density-bridges}.}
\label{tab:weighting-comparisons}
\begin{tabularx}{\linewidth}{
  >{\hsize=1.45\hsize\raggedright\arraybackslash}X
  >{\hsize=0.55\hsize\raggedright\arraybackslash}X
  l
}
\toprule
Scheme & Weight on a source point & Bridge coordinate \\
\midrule
Source-only ERM and source-anchored DRO \citep{Ben-Tal1987,HuHong2012} & $1$ (loss tilt $e^{L_h/\nu}$ for KL) & any pairing, $\alpha=0$ \\
\addlinespace[2pt]
Importance weighting \citep{Shimodaira2000,huang2006,Sugiyama2008} & $\LR$ & KL/KL, HD/HD, $\alpha=1$ \\
\addlinespace[2pt]
Flattened/tempered weighting \citep{Shimodaira2000,Sugiyama2007} & $\LR^{\lambda}$ & $\alpha=\lambda$, KL/KL \\
\addlinespace[2pt]
Geometrically generalized weighting \citep{Kimura2022} & $\{(1-\alpha)+\alpha\sqrt\LR\}^2$ & HD/HD, $\alpha$ \\
\addlinespace[2pt]
Lambert-$W$ attenuation (this paper) & $[\beta_\alpha/W_0(\beta_\alpha)]^2$, order $\LR/(\log\LR)^2$ & KL/HD, $\alpha$ \\
\bottomrule
\end{tabularx}
\end{table}

\section{Optimal adversarial bridges}\label{sec:optimal-adversarial-bridges}

Based on the doubly-anchored ambiguity set
$\as(P_S,P_T^\circ)$, our goal is to minimize the most adverse
target-compatible risk. The corresponding robust minimax value is
\begin{equation*}
    \inf_{h\in\hs}
    \sup_{Q\in\as(P_S,P_T^\circ)}
    \E_{Z\sim Q}[L_h(Z)].
\end{equation*}
Whenever this infimum is attained, we denote a minimizer by $h^\star$.
In the target-unlabeled setting of this paper, the target reference law is
$P_T^\circ=P_T^X\otimes P_S^{Y\mid X}$, and hence its likelihood ratio
relative to the source law depends only on the covariate:
\begin{equation*}
    \LR(z)
    :=\frac{\d P_T^\circ}{\d P_S}(z)
    =\frac{\d P_T^X}{\d P_S^X}(x)
    =\frac{p_T^X(x)}{p_S^X(x)}.
\end{equation*}

To solve the infinite-dimensional inner problem and subsequently derive
finite-sample generalization bounds, we proceed in two steps. First, we
establish a general dual representation below and then, in
Section~\ref{sec:bridge-geometries}, derive the optimal loss-aware adversarial
density for each of the three divergence pairings. Section~\ref{sec:symmetric-single-ball-geometry} identifies the symmetric two-ball intersections through their barycentric centers. The resulting formulas show how the target constraint modifies the source-side adversarial tilt. Second, in
Section~\ref{sec:loss-agnostic-density-bridges}, we retain the density geometry
induced by the same divergence pairings while removing the localized loss
term. The resulting structural bridge laws serve as hypothesis-independent intermediate
reference measures for the risk bounds and non-parametric estimators developed
in the remainder of the paper.

\subsection{General dual representation}\label{sec:general-dual-representation}

Fix $h\in\hs$ and define the full population robust risk,
\begin{equation*}
    \Psi_{S,T}(h)
    =
    \sup_{Q\in\as(P_S,P_T^\circ)}
    \E_Q[L_h(Z)].
\end{equation*}
We also define the anchor-dominated value
\begin{equation*}
    \Psi_{S,T}^{\mathrm{ac}}(h)
    =
    \sup_{Q\in\as_{\mathrm{ac}}(P_S,P_T^\circ)}
    \E_Q[L_h(Z)].
\end{equation*}

The density-ratio dual is first derived for
$\Psi_{S,T}^{\mathrm{ac}}(h)$.
Under Assumption~\ref{ass:common-domination}, the two anchors
$P_S$ and $P_T^\circ$ are mutually absolutely continuous.
For every
$Q\in\as_{\mathrm{ac}}(P_S,P_T^\circ)$, put
\begin{equation} \label{eqn:r-and-lr-under-anchor-equivalence}
    r=\frac{\d Q}{\d P_S},
    \,
    \LR=\frac{\d P_T^\circ}{\d P_S}
    \;\Rightarrow
    \frac{\d Q}{\d P_T^\circ}
    =
    \frac{r}{\LR},
\end{equation}
and the two divergence constraints can be expressed under the source law:
\begin{align}
    D_{\phi_S}(Q\Vert P_S)
    &=
    \E_{P_S}[\phi_S(r)],
    \label{eqn:source-divergence-under-source}
    \\
    D_{\phi_T}(Q\Vert P_T^\circ)
    &=
    \E_{P_S}\left[
        \LR\,\phi_T\left(\frac{r}{\LR}\right)
    \right].
    \label{eqn:target-divergence-under-source}
\end{align}
Consequently,
\begin{equation}\label{eqn:primal-density-ratio-problem}
\begin{aligned}
    \Psi_{S,T}^{\mathrm{ac}}(h)
    =
    \sup_{r\geq0}\quad
    &\E_{P_S}[rL_h]\\
    \text{subject to}\quad
    &\E_{P_S}[r]=1,\\
    &\E_{P_S}[\phi_S(r)]\leq\rho_S,\\
    &\E_{P_S}\left[
       \LR\,\phi_T\left(\frac{r}{\LR}\right)
     \right]\leq\rho_T.
\end{aligned}
\end{equation}

If the source-side divergence is KL, then
$\phi_{\mathrm{KL}}^\infty(1)=\infty$ and every law in the full
ambiguity set is source-dominated. Hence
\begin{equation}\label{eqn:full-ac-values-coincide-kl-source}
    \Psi_{S,T}(h)
    =
    \Psi_{S,T}^{\mathrm{ac}}(h),
\end{equation}
for both the KL/KL and KL/HD pairings.
For HD/HD, the full ambiguity set may contain singular laws at finite divergence cost.
The relation between the full extended-divergence value and the dominated problem, together with the associated attainment question, is treated separately below.
Lemma~\ref{lem:full-vs-ac-risk-hdhd} in Appendix~\ref{app:full-vs-ac-riks} further quantifies the gap between $\Psi_{S,T}$ and $\Psi_{S,T}^{\mathrm{ac}}$ for the HD/HD pairing.
The density-ratio problem in \eqref{eqn:primal-density-ratio-problem} is posed on the positive cone
\[
    \mathscr R
    :=
    L_+^1(P_S)
    =
    \left\{
      r\in L^1(P_S):r\geq0\quad P_S\text{-almost surely}
    \right\}.
\]
This cone is closed and convex. We use the extended-value convention
$\phi_S(t)=\phi_T(t)=+\infty$ for $t<0$. On $L^1(P_S)$, define the
extended-real convex integral functionals
\[
    \mathcal G_S(r)
    =
    \E_{P_S}[\phi_S(r)],
    \qquad
    \mathcal G_T(r)
    =
    \E_{P_S}\left[
      \LR\,\phi_T\left(\frac{r}{\LR}\right)
    \right].
\]
Their common effective domain is
\[
    \mathscr C
    :=
    \operatorname{dom}\mathcal G_S
    \cap
    \operatorname{dom}\mathcal G_T
    \subseteq \mathscr R,
\]
and the affine normalization map
\[
    \mathcal A(r)=\E_{P_S}[r]
\]
is continuous on $L^1(P_S)$.

The following result gives the dual representation underlying all three bridge geometries studied below.

\begin{theorem}[General dual representation on the dominated class]
\label{thm:general-dual-representation}
\leavevmode\par\noindent
Suppose Assumptions~\ref{ass:bounded-loss} and
\ref{ass:common-domination} hold. Assume that
$\phi_S,\phi_T\colon[0,\infty)\to[0,\infty]$ are proper
lower-semicontinuous convex functions satisfying
$\phi_S(1)=\phi_T(1)=0$, and that they are differentiable and strictly convex
on $(0,\infty)$. Suppose further that the anchor-dominated density
problem is strictly feasible:
there exists a measurable $r_0>0$, $P_S$-almost surely, such that
\begin{equation}\label{eqn:dual-slater-condition}
    \E_{P_S}[r_0]=1,
    \qquad
    \E_{P_S}[\phi_S(r_0)]<\rho_S,
    \qquad
    \E_{P_S}\left[
       \LR\,\phi_T\left(\frac{r_0}{\LR}\right)
    \right]<\rho_T.
\end{equation}
For $a\in\R$, $u>0$, and $\lambda_S,\lambda_T\geq0$, define the
extended-real pointwise conjugate
\begin{equation*}
\Gamma_{\lambda_S,\lambda_T}(a;u)
=
\sup_{t\geq0}
\left\{
    at-\lambda_S\phi_S(t)
    -\lambda_Tu\,\phi_T\left(\frac{t}{u}\right)
\right\}
\in\R\cup\{+\infty\}.
\end{equation*}
Here and below, a term with zero multiplier is taken to be zero, including
$0\cdot(+\infty)=0$.
For a fixed hypothesis $h\in\hs$, define the ordinary constrained-domain
dual function
\begin{equation*}
\begin{aligned}
d_h^{\mathscr C}(\lambda_S,\lambda_T,\eta)
={}&
\lambda_S\rho_S+\lambda_T\rho_T+\eta\\
&+
\sup_{r\in\mathscr C}
\left[
  \E_{P_S}\{r(L_h-\eta)\}
  -\lambda_S\mathcal G_S(r)
  -\lambda_T\mathcal G_T(r)
\right].
\end{aligned}
\end{equation*}
Also let
\begin{equation*}
\begin{aligned}
\mathfrak D_h
=
\Bigg\{
(\lambda_S,\lambda_T,\eta)
\in\R_+^2\times\R:\;
\E_{P_S}\left[
  \Gamma_{\lambda_S,\lambda_T}
  \{L_h(Z)-\eta;\LR(Z)\}
\right]
<\infty
\Bigg\}.
\end{aligned}
\end{equation*}
The effective-domain restriction is necessary because the pointwise
conjugate need not be finite for every multiplier triple. For example,
\[
    \Gamma_{0,0}(a;u)=+\infty
    \qquad\text{when }a>0.
\]
For the HD/HD pairing, if $\lambda_S+\lambda_T>0$, the conjugate is finite
precisely when
\[
    a<\lambda_S+\lambda_T.
\]
At equality, the remaining positive square-root term still makes the
pointwise supremum infinite.
Then, the robust risk admits the two equivalent dual representations
\begin{equation}\label{eqn:general-two-anchor-dual}
\begin{aligned}
\Psi_{S,T}^{\mathrm{ac}}(h)
    & =
    \min_{\lambda_S,\lambda_T\geq0,\,\eta\in\R}
    d_h^{\mathscr C}(\lambda_S,\lambda_T,\eta)\\
    & =
    \min_{(\lambda_S,\lambda_T,\eta)\in\mathfrak D_h}
    \Bigg\{
        \lambda_S\rho_S+\lambda_T\rho_T+\eta
        +\E_{P_S}\left[
          \Gamma_{\lambda_S,\lambda_T}
          \{L_h(Z)-\eta;\LR(Z)\}
        \right]
    \Bigg\},
\end{aligned}
\end{equation}
and they attain their minimum at the same multiplier triples.

Suppose additionally that the primal supremum is attained at $r_h^\star$, and
let $(\lambda_S^\star,\lambda_T^\star,\eta^\star)$ be any minimizer of
$d_h^{\mathscr C}$.
Then the complementary-slackness relations are
\begin{align}
    \lambda_S^\star
    \left\{
      \E_{P_S}[\phi_S(r_h^\star)]-\rho_S
    \right\}
    &=0,
    \label{eqn:kkt-source-slackness}
    \\
    \lambda_T^\star
    \left\{
      \E_{P_S}\left[
        \LR\,\phi_T\left(\frac{r_h^\star}{\LR}\right)
      \right]-\rho_T
    \right\}
    &=0.
    \label{eqn:kkt-target-slackness}
\end{align}
Moreover, if $\lambda_S^\star+\lambda_T^\star>0$ and
$r_h^\star>0$, $P_S$-almost surely, then it is the unique
pointwise maximizer and satisfies
\begin{equation}\label{eqn:general-pointwise-kkt}
    L_h(z)-\eta^\star
    =
    \lambda_S^\star\phi_S'\{r_h^\star(z)\}
    +
    \lambda_T^\star
    \phi_T'\left\{
       \frac{r_h^\star(z)}{\LR(z)}
    \right\}
\end{equation}
for $P_S$-almost every $z$.
\end{theorem}

\begin{proof}
We first reduce the anchor-dominated optimization problem to an
optimization over source-density ratios.
By definition of
$\as_{\mathrm{ac}}(P_S,P_T^\circ)$, every feasible law satisfies
$Q\ll P_S$.
Assumption~\ref{ass:common-domination} also gives $P_T^\circ\sim P_S$, and $r,\LR$ are as in~\eqref{eqn:r-and-lr-under-anchor-equivalence}.
The probability constraint is $\E_{P_S}[r]=1$, and \eqref{eqn:source-divergence-under-source} follows from the fact that the singular term in~\eqref{eqn:extended-phi-divergence} vanishes for $Q\ll P_S$.
Changing measure from $P_T^\circ$ to $P_S$ gives
\[
    D_{\phi_T}(Q\Vert P_T^\circ)
    =
    \int_\fs
    \phi_T\left(\frac{r}{\LR}\right)\d P_T^\circ
    =
    \int_\fs
    \LR\,\phi_T\left(\frac{r}{\LR}\right)\d P_S,
\]
which proves~\eqref{eqn:target-divergence-under-source} and the
density-ratio representation~\eqref{eqn:primal-density-ratio-problem}.
For multipliers $\lambda_S,\lambda_T\geq0$ and $\eta\in\R$, the Lagrangian of
\eqref{eqn:primal-density-ratio-problem} is
\begin{equation}\label{eqn:general-two-anchor-lagrangian}
\begin{aligned}
    \mathcal L_h(r,\lambda_S,\lambda_T,\eta)
    ={}&
    \lambda_S\rho_S+\lambda_T\rho_T+\eta\\
    &+\E_{P_S}\Bigg[
        r(L_h-\eta)
        -\lambda_S\phi_S(r)
        -\lambda_T\LR\,
          \phi_T\left(\frac{r}{\LR}\right)
      \Bigg].
\end{aligned}
\end{equation}
Here the final constant $+\eta$ comes from the normalization term
$-\eta\{\E_{P_S}[r]-1\}$. For every feasible $r$ and every
$\lambda_S,\lambda_T\geq0$, each term
$\lambda_j\{\rho_j-\mathcal G_j(r)\}$ is nonnegative.
It follows that the infimum of the dual objective is an upper bound on the
primal value.

The cone $\mathscr R=L_+^1(P_S)$ is a closed convex subset of
$L^1(P_S)$, and the normalization map
$\mathcal A(r)=\E_{P_S}[r]$ is continuous and affine.
The functionals $\mathcal G_S$ and $\mathcal G_T$ are proper, convex,
and lower semicontinuous on $L^1(P_S)$. Properness follows from the
strictly feasible density $r_0$. For lower semicontinuity, consider
$r_n\to r$ in $L^1(P_S)$, choose a subsequence realizing the relevant
limit inferior, and then pass to a further subsequence converging to $r$
almost surely. Lower semicontinuity of the nonnegative convex integrands
and Fatou's lemma give
\[
    \mathcal G_S(r)
    \leq
    \liminf_n\mathcal G_S(r_n),
    \qquad
    \mathcal G_T(r)
    \leq
    \liminf_n\mathcal G_T(r_n).
\]

To include the affine normalization in the perturbation argument, define
\begin{equation*}
\begin{aligned}
V_h(u_S,u_T,v)
=
\sup\Big\{&\E_{P_S}[rL_h]:r\in\mathscr C,\
\mathcal G_S(r)\leq\rho_S+u_S,\
\mathcal G_T(r)\leq\rho_T+u_T,\\
&\mathcal A(r)=1+v\Big\},
\end{aligned}
\end{equation*}
with the supremum of the empty set equal to $-\infty$.
Put $\sigma_j=\rho_j-\mathcal G_j(r_0)>0$, $j\in\{S,T\}$.
Since $r_0$ and $\LR$ are positive and finite almost surely, there is a
measurable set $B$ of positive $P_S$-measure on which both functions are
bounded above and away from zero. For sufficiently small $|v|$, put
\[
    r_v=r_0+\frac{v}{P_S(B)}\mathbf 1_B.
\]
Then $r_v>0$, $\mathcal A(r_v)=1+v$, and
$\mathcal G_j(r_v)\to\mathcal G_j(r_0)$ as $v\to0$.
Indeed, the relevant integrands are unchanged off $B$ and, on $B$, their
arguments remain in compact subsets of $(0,\infty)$, where the generators
are continuous. The strict slacks $\sigma_S,\sigma_T$ therefore imply that
$V_h$ is finite on a neighborhood of the origin. On that neighborhood it is
concave and satisfies
\[
    0\leq V_h(u_S,u_T,v)\leq M(1+v).
\]
Thus $V_h$ is locally bounded above, and the perturbation-duality theorem
\citep[Theorem~17(a)]{Rockafellar1974} applies.

For completeness, let
$(\lambda_S^\star,\lambda_T^\star,\eta^\star)$ be a supergradient of
$V_h$ at the origin. Monotonicity of $V_h$ in its first two coordinates
gives $\lambda_S^\star,\lambda_T^\star\geq0$. Evaluating the supergradient
inequality at
\[
\left(
\mathcal G_S(r)-\rho_S,
\mathcal G_T(r)-\rho_T,
\mathcal A(r)-1
\right),
\qquad r\in\mathscr C,
\]
and then taking the supremum over $r$ gives
\[
    d_h^{\mathscr C}
    (\lambda_S^\star,\lambda_T^\star,\eta^\star)
    \leq V_h(0,0,0).
\]
Weak duality gives the reverse inequality. Hence there is no duality gap and
the constrained-domain dual is attained.
Related $\phi$-divergence DRO duality arguments are developed by
\citet{Ben-Tal1987} and \citet{Shapiro2017}. Hence
\[
    \Psi_{S,T}^{\mathrm{ac}}(h)
    =
    \min_{\lambda_S,\lambda_T\geq0,\,\eta\in\R}
    d_h^{\mathscr C}(\lambda_S,\lambda_T,\eta).
\]

For fixed $(\lambda_S,\lambda_T,\eta)$, define
\[
\mathfrak l_{\lambda_S,\lambda_T,\eta}(z,t)
=
\begin{cases}
t\{L_h(z)-\eta\}
-\lambda_S\phi_S(t)
-\lambda_T\LR(z)
 \phi_T\!\left(\dfrac{t}{\LR(z)}\right),
& t\geq0,\\[6pt]
-\infty,
& t<0.
\end{cases}
\]
We now reduce the supremum over $\mathscr C$ to the pointwise problem by
a direct localization argument. Let
\[
b(z)
=
\mathfrak l_{\lambda_S,\lambda_T,\eta}\{z,r_0(z)\}.
\]
Boundedness of $L_h$ and~\eqref{eqn:dual-slater-condition} give
$b\in L^1(P_S)$. The function
$t\mapsto\mathfrak l_{\lambda_S,\lambda_T,\eta}(z,t)$ is concave and
continuous on $(0,\infty)$. Its value at zero, whenever finite, is
approached from the right by concavity and upper semicontinuity. Hence, if
$\{q_j:j\geq1\}$ enumerates $\mathbb Q\cap(0,\infty)$, then
\[
\Gamma_{\lambda_S,\lambda_T}
\{L_h(z)-\eta;\LR(z)\}
=
\sup_{j\geq1}
\mathfrak l_{\lambda_S,\lambda_T,\eta}(z,q_j).
\]
In particular, the pointwise supremum is measurable, dominates $b$, and
has a well-defined expectation in $\mathbb R\cup\{+\infty\}$.

For $N\geq1$, put
\[
g_N(z)
=
\max_{1\leq j\leq N}
\left[
\mathfrak l_{\lambda_S,\lambda_T,\eta}(z,q_j)-b(z)
\right]_+.
\]
On $\{g_N>0\}$, let $j_N(z)$ be the smallest index attaining this finite
maximum and put $s_N(z)=q_{j_N(z)}$; on $\{g_N=0\}$, put
$s_N(z)=r_0(z)$. These functions are measurable. For $K\geq1$, let
\[
B_{N,K}
=
\left\{
\max_{1\leq j\leq N}
\LR\phi_T(q_j/\LR)
\leq K
\right\},
\qquad
r_{N,K}
=
s_N\mathbf 1_{B_{N,K}}+r_0\mathbf 1_{B_{N,K}^c}.
\]
Because each $q_j>0$, differentiability of the generators on
$(0,\infty)$ gives $\phi_S(q_j)<\infty$ and
$\LR\phi_T(q_j/\LR)<\infty$ almost surely. Thus
$B_{N,K}\uparrow\fs$, up to a null set, as $K\to\infty$, and
\[
\mathcal G_S(r_{N,K})
\leq
\mathcal G_S(r_0)+\max_{1\leq j\leq N}\phi_S(q_j)<\infty,
\qquad
\mathcal G_T(r_{N,K})
\leq
\mathcal G_T(r_0)+K<\infty.
\]
Therefore $r_{N,K}\in L^1(P_S)$ and, by the displayed bounds,
$r_{N,K}\in\mathscr C$. Moreover,
\[
\E_{P_S}\left[
\mathfrak l_{\lambda_S,\lambda_T,\eta}\{Z,r_{N,K}(Z)\}
\right]
=
\E_{P_S}[b(Z)]
+\E_{P_S}\left[g_N(Z)\mathbf 1_{B_{N,K}}(Z)\right].
\]
First let $K\to\infty$ and then $N\to\infty$. Since
\[
g_N\uparrow
\Gamma_{\lambda_S,\lambda_T}
\{L_h-\eta;\LR\}-b,
\]
the monotone-convergence theorem, together with
$\mathfrak l_{\lambda_S,\lambda_T,\eta}\{z,r(z)\}
\leq
\Gamma_{\lambda_S,\lambda_T}\{L_h(z)-\eta;\LR(z)\}$ for every
$r\in\mathscr C$, gives
\begin{equation*}
\begin{aligned}
\sup_{r\in\mathscr C}
\E_{P_S}\left[
\mathfrak l_{\lambda_S,\lambda_T,\eta}\{Z,r(Z)\}
\right]
&=
\E_{P_S}\left[
\Gamma_{\lambda_S,\lambda_T}
\{L_h(Z)-\eta;\LR(Z)\}
\right].
\end{aligned}
\end{equation*}
For multiplier triples outside $\mathfrak D_h$, the right-hand side is
$+\infty$ and cannot minimize the dual objective. The minimum may therefore
be restricted to the effective domain $\mathfrak D_h$; every minimizer of
$d_h^{\mathscr C}$ belongs to this domain and also minimizes the pointwise
representation.
Substituting this identity into the dual objective proves
\eqref{eqn:general-two-anchor-dual}.

It remains to establish the KKT conclusions. Let $r_h^\star$ attain the
primal supremum and let
$(\lambda_S^\star,\lambda_T^\star,\eta^\star)$ minimize
$d_h^{\mathscr C}$. Feasibility and strong duality give
\[
\begin{aligned}
\Psi_{S,T}^{\mathrm{ac}}(h)
&=d_h^{\mathscr C}(\lambda_S^\star,\lambda_T^\star,\eta^\star)\\
&\geq
\mathcal L_h(r_h^\star,\lambda_S^\star,\lambda_T^\star,\eta^\star)\\
&=
\Psi_{S,T}^{\mathrm{ac}}(h)
+\lambda_S^\star\{\rho_S-\mathcal G_S(r_h^\star)\}
+\lambda_T^\star\{\rho_T-\mathcal G_T(r_h^\star)\}\\
&\geq\Psi_{S,T}^{\mathrm{ac}}(h).
\end{aligned}
\]
Equality throughout proves~\eqref{eqn:kkt-source-slackness} and
\eqref{eqn:kkt-target-slackness}, and also shows that $r_h^\star$
maximizes the Lagrangian over $\mathscr C$.

For the pointwise conclusion, use the preceding identity at the optimal
multiplier triple. The equality between the inner supremum and the value at
$r_h^\star$ gives
\[
\E_{P_S}\Bigg[
\Gamma_{\lambda_S^\star,\lambda_T^\star}
\{L_h(Z)-\eta^\star;\LR(Z)\}
-
\mathfrak l_{\lambda_S^\star,\lambda_T^\star,\eta^\star}
\{Z,r_h^\star(Z)\}
\Bigg]
=0.
\]
The integrand is nonnegative, so $r_h^\star$ is a pointwise maximizer
almost surely. If $\lambda_S^\star+\lambda_T^\star>0$, strict convexity
of the active penalty makes the pointwise Lagrangian integrand strictly
concave. At the positive optimizer, differentiating with respect to $r$
gives
\[
    0
    =
    L_h-\eta^\star
    -\lambda_S^\star\phi_S'(r_h^\star)
    -\lambda_T^\star
      \phi_T'\left(\frac{r_h^\star}{\LR}\right),
\]
which is~\eqref{eqn:general-pointwise-kkt}. The same strict convexity yields
pointwise uniqueness.
\end{proof}

\begin{remark}[Role of the target information]
Theorem~\ref{thm:general-dual-representation} separates the general convex
duality argument from the geometry-specific calculations below. The target
reference law enters the pointwise dual problem through the likelihood ratio
$\LR$, while $\rho_T$ enters the outer dual objective and thereby affects the
optimal target-side multiplier. In the target-unlabeled setting,
$\LR(z)=p_T^X(x)/p_S^X(x)$, so the target sample contributes only through its
observable covariate marginal, while source labels supply the supervised loss
information.
\end{remark}

\begin{remark}[Boundary radii and inactive constraints]
Theorem~\ref{thm:general-dual-representation} is stated under strict
feasibility so that strong duality and dual attainment. The pointwise KKT equations hold for an attained primal optimizer under the conditions of the theorem.
An inactive constraint has zero optimal multiplier and yields the corresponding
single-anchor pointwise equation. A zero radius is different: for nonnegative
divergences it is incompatible with strict feasibility and lies outside the
theorem. Such a case must be treated directly or through a justified limit of
positive-radius problems.
\end{remark}

\subsection{Symmetric Single-Ball Geometry}\label{sec:symmetric-single-ball-geometry}
The exponential KL path and the normalized square-root Hellinger path are
classical barycentric constructions.  The result below uses these paths to
identify the entire doubly anchored intersection, its effective slack radii,
and the bridge coordinate selected by dual profiling.

\begin{theorem}[Single-Ball Representation of the Symmetric Intersections]
\label{thm:symmetric-single-ball-representation}
\leavevmode\par\noindent
For $\alpha\in[0,1]$, the following hold.

\textbf{(i) KL/KL.}  Let
\[
\Lambda(\alpha)=\log\E_{P_S}[\varrho^\alpha],
\qquad
\frac{\d P_\alpha^{\mathrm{KL}}}{\d P_S}
=\varrho^\alpha e^{-\Lambda(\alpha)}.
\]
Then, as an identity in $[0,\infty]$,
\begin{equation}\label{eqn:single-ball-kl-identity}
(1-\alpha)D_{\mathrm{KL}}(Q\Vert P_S)
+\alpha D_{\mathrm{KL}}(Q\Vert P_T^\circ)
=
D_{\mathrm{KL}}(Q\Vert P_\alpha^{\mathrm{KL}})-\Lambda(\alpha).
\end{equation}
Consequently, with
\[
\rho_{\mathrm{eff}}^{\mathrm{KL}}(\alpha)
=(1-\alpha)\rho_S+\alpha\rho_T+\Lambda(\alpha),
\]
\begin{equation}\label{eqn:single-ball-kl-intersection}
\as(P_S,P_T^\circ)
=
\bigcap_{\alpha\in[0,1]}
\left\{Q:D_{\mathrm{KL}}(Q\Vert P_\alpha^{\mathrm{KL}})
\le \rho_{\mathrm{eff}}^{\mathrm{KL}}(\alpha)\right\}.
\end{equation}

\textbf{(ii) HD/HD.}  Let $\delta_H=D_{\phi_{\mathrm{HD}}}(P_S\Vert P_T^\circ)$ and
\(
\ZH(\alpha)=1-\alpha(1-\alpha)\delta_H.
\)
For densities $p_S,p_T^\circ$ with respect to a common dominating measure, define
\[
\sqrt{p_\alpha^{\mathrm{HD}}}
=
\frac{(1-\alpha)\sqrt{p_S}+\alpha\sqrt{p_T^\circ}}{\sqrt{\ZH(\alpha)}}.
\]
Then, with the extended Hellinger divergence,
\begin{equation}\label{eqn:single-ball-hd-identity}
\begin{aligned}
&(1-\alpha)D_{\phi_{\mathrm{HD}}}(Q\Vert P_S)
+\alpha D_{\phi_{\mathrm{HD}}}(Q\Vert P_T^\circ)\\
&\qquad=
2\{1-\sqrt{\ZH(\alpha)}\}
+\sqrt{\ZH(\alpha)}\,
D_{\phi_{\mathrm{HD}}}(Q\Vert P_\alpha^{\mathrm{HD}}).
\end{aligned}
\end{equation}
Hence, with
\[
\rho_{\mathrm{eff}}^{\mathrm{HD}}(\alpha)
=
\frac{(1-\alpha)\rho_S+\alpha\rho_T-2\{1-\sqrt{\ZH(\alpha)}\}}
{\sqrt{\ZH(\alpha)}},
\]
\begin{equation}\label{eqn:single-ball-hd-intersection}
\as(P_S,P_T^\circ)
=
\bigcap_{\alpha\in[0,1]}
\left\{Q:D_{\phi_{\mathrm{HD}}}(Q\Vert P_\alpha^{\mathrm{HD}})
\le \rho_{\mathrm{eff}}^{\mathrm{HD}}(\alpha)\right\}.
\end{equation}
\end{theorem}

\begin{proof}
For KL/KL, whenever the divergences are finite,
\[
\log\frac{\d Q}{\d P_\alpha^{\mathrm{KL}}}
=
\log\frac{\d Q}{\d P_S}-\alpha\log\varrho+\Lambda(\alpha),
\]
and integration gives~\eqref{eqn:single-ball-kl-identity}. Lemma~\ref{lem:supp-finite-measure-barycentric}
covers all $Q\ll P_S$ in the extended sense. If $Q\not\ll P_S$, every KL divergence to the equivalent centers is infinite. For HD/HD, take $\mu$ to dominate $Q,P_S,P_T^\circ$. Then
$D_{\phi_{\mathrm{HD}}}(Q\Vert P)=2-2\int\sqrt{qp}\,\dd\mu$
gives~\eqref{eqn:single-ball-hd-identity} directly. In each case, membership
in the original two balls implies every displayed single-ball inequality,
while $\alpha=0$ and $\alpha=1$ recover the original constraints.
\end{proof}

The effective radii make the role of the two original radii transparent.  If
$\delta_{B,S}(\alpha)$ and $\delta_{B,T}(\alpha)$ denote the source- and target-side divergences of the corresponding bridge center, then
\begin{align}
\rho_{\mathrm{eff}}^{\mathrm{KL}}(\alpha)
&=(1-\alpha)\{\rho_S-\delta_{\mathrm{KL},S}(\alpha)\}
+\alpha\{\rho_T-\delta_{\mathrm{KL},T}(\alpha)\},
\label{eqn:single-ball-kl-slack}\\
\sqrt{\ZH(\alpha)}\,\rho_{\mathrm{eff}}^{\mathrm{HD}}(\alpha)
&=(1-\alpha)\{\rho_S-\delta_{\mathrm{HD},S}(\alpha)\}
+\alpha\{\rho_T-\delta_{\mathrm{HD},T}(\alpha)\}.
\label{eqn:single-ball-hd-slack}
\end{align}
Thus the symmetric structural bridges are precisely the centers of the effective single-ball problems, and the effective radius is their weighted admissibility slack.  Under target feasibility, the source-side slack is nonnegative along the full symmetric bridge path: $\rho_S$ therefore need not truncate that path, while it continues to affect the residual adversarial radius around each bridge center.

The same representation has a dual interpretation. Define the effective
single-ball values
\begin{align*}
\Psi_{\alpha}^{\mathrm{KL}}(h)
&=
\sup\left\{\E_Q[L_h]:
D_{\mathrm{KL}}(Q\Vert P_\alpha^{\mathrm{KL}})
\leq \rho_{\mathrm{eff}}^{\mathrm{KL}}(\alpha)\right\},\\
\Psi_{\alpha}^{\mathrm{HD},\mathrm{ac}}(h)
&=
\sup\left\{\E_Q[L_h]:Q\ll P_S,
D_{\phi_{\mathrm{HD}}}(Q\Vert P_\alpha^{\mathrm{HD}})
\leq \rho_{\mathrm{eff}}^{\mathrm{HD}}(\alpha)\right\}.
\end{align*}
Under the conditions of Theorem~\ref{thm:general-dual-representation},
Section~\ref{supp:single-ball-dual} proves the profiling identities
\[
\Psi_{S,T}^{\mathrm{ac}}(h)
=
\min_{\alpha\in[0,1]}\Psi_\alpha^B(h),
\]
for both symmetric geometries.  If
$\lambda_S^\star+\lambda_T^\star>0$, a minimizing coordinate is
$\alpha^\star=\lambda_T^\star/(\lambda_S^\star+\lambda_T^\star)$.
The optimal multiplier ratio selects an effective ball for the loss-aware problem. Appendix~\ref{sec:radius-indexed-selection} characterizes the coordinates for which the structural center itself satisfies both original constraints.

\subsection{Dual-induced bridge geometries}\label{sec:bridge-geometries}

We now solve the common pointwise KKT equation
\eqref{eqn:general-pointwise-kkt} for the three divergence pairings used in
this paper. Let
\[
    \phi_{\mathrm{KL}}(t)=t\log t-t+1,
    \qquad
    \phi_{\mathrm{HD}}(t)=(\sqrt t-1)^2,
    \qquad t\geq0.
\]
The parameter $\alpha$ below is the normalized target-side dual weight. For
the two symmetric pairings,
$\alpha=\lambda_T/(\lambda_S+\lambda_T)$, whereas for the asymmetric KL/HD
pairing,
$\alpha=\lambda_T/(2\lambda_S+\lambda_T)$. Thus the precise normalization
is geometry-specific, even though in every case $\alpha$ records the relative
source--target dual balance.

\begin{proposition}[Dual-induced adversarial bridges]
\label{prop:loss-adversarial-bridges}
Fix $h\in\hs$ under the conditions of
Theorem~\ref{thm:general-dual-representation}, suppose that the primal
supremum is attained at $r_h^\star$, and let
$(\lambda_S^\star,\lambda_T^\star,\eta^\star)$ be any minimizer of
$d_h^{\mathscr C}$.
Set $q_h^\star=r_h^\star p_S$. This is an optimal anchor-dominated
adversarial density. For the KL/KL and KL/HD pairings, it also solves the
full extended-divergence problem by
\eqref{eqn:full-ac-values-coincide-kl-source}.
For the multiplier regimes described below, it has the following form.
\begin{enumerate}
\item For the KL/KL pairing, let
\[
    \nu=\lambda_S^\star+\lambda_T^\star>0,
    \qquad
    \alpha=\frac{\lambda_T^\star}{\lambda_S^\star+\lambda_T^\star}.
\]
Then
\begin{equation}\label{eqn:adv-bridge-klkl}
    q_{h,\mathrm{KL}}^\star(z)
    =
    \frac{
      \exp\{L_h(z)/\nu\}
      p_S(z)^{1-\alpha}
      p_T^\circ(z)^\alpha
    }{
      \displaystyle
      \int_{\fs}
      \exp\{L_h(u)/\nu\}
      p_S(u)^{1-\alpha}
      p_T^\circ(u)^\alpha\,\d\mu(u)
    }.
\end{equation}

\item For the HD/HD pairing, suppose
\[
    \lambda_+
    =\lambda_S^\star+\lambda_T^\star>0,
    \qquad
    \alpha
    =\frac{\lambda_T^\star}{\lambda_+},
    \qquad
    \nu
    =
    \lambda_++\eta^\star.
\]
Assume that this optimizer lies in the uniformly separated regime
\[
    \nu>
    \operatorname*{ess\,sup}_{P_S}L_h.
\]
Then
\begin{equation}\label{eqn:adv-bridge-hdhd}
    q_{h,\mathrm{HD}}^\star(z)
    =
    \frac{
      \displaystyle
      \frac{
      \left\{(1-\alpha)\sqrt{p_S(z)}
             +\alpha\sqrt{p_T^\circ(z)}\right\}^{2}
      }{\{\nu-L_h(z)\}^{2}}
    }{
      \displaystyle
      \int_{\fs}
      \frac{
      \left\{(1-\alpha)\sqrt{p_S(u)}
             +\alpha\sqrt{p_T^\circ(u)}\right\}^{2}
      }{\{\nu-L_h(u)\}^{2}}\,\d\mu(u)
    }.
\end{equation}

\item For the KL/HD pairing, suppose $\lambda_S^\star>0$ and put
\begin{equation}\label{eqn:klhd-dual-parameters}
    \nu=2\lambda_S^\star,
    \qquad
    \alpha
    =\frac{\lambda_T^\star}{2\lambda_S^\star+\lambda_T^\star}
    \in[0,1),
    \qquad
    \beta_\alpha(z)
    =\frac{\alpha}{1-\alpha}\sqrt{\LR(z)}.
\end{equation}
Then
\begin{equation}\label{eqn:adv-bridge-klhd}
    q_{h,\mathrm{KH}}^\star(z)
    =
    p_S(z)
    \left[
      \frac{
        \beta_\alpha(z)
      }{
        W_0\left(
          \beta_\alpha(z)
          \exp\{-L_h(z)/\nu-\gamma\}
        \right)
      }
    \right]^2,
\end{equation}
where
$\gamma=-(\eta^\star+\lambda_T^\star)/\nu$ is the scalar dual offset;
equivalently, it is the unique value for which
$\int q_{h,\mathrm{KH}}^\star\d\mu=1$. At $\alpha=0$, the ratio in
\eqref{eqn:adv-bridge-klhd} is understood through its continuous limit.
The function $W_0$ denotes the principal branch of the Lambert-$W$ function.
\end{enumerate}
\end{proposition}

\begin{proof}
Before applying the pointwise KKT equation
\eqref{eqn:general-pointwise-kkt}, we verify its interior condition.
The zero-gap argument in the proof of
Theorem~\ref{thm:general-dual-representation} first shows that
$r_h^\star$ is a pointwise maximizer almost surely. In each multiplier
regime stated in the proposition, the scalar Lagrangian integrand is strictly
concave on $(0,\infty)$ and its right derivative tends to $+\infty$ as
$t\downarrow0$. The divergent positive terms in the three cases are,
respectively, $-\nu\log t$,
$(\lambda_S^\star+\lambda_T^\star\sqrt\LR)/\sqrt t$, and
$-\lambda_S^\star\log t+\lambda_T^\star\sqrt{\LR(z)/t}$.
Hence its maximizer is positive, and the pointwise KKT
equation applies. We now use it separately for the three divergence pairings.

For the KL/KL bridge,
$\phi_{\mathrm{KL}}'(t)=\log t$, and hence
\[
\begin{aligned}
    L_h(z)-\eta^\star
    =
    \lambda_S^\star\log r(z)
    +\lambda_T^\star
      \log\left\{\frac{r(z)}{\LR(z)}\right\}
    =
    (\lambda_S^\star+\lambda_T^\star)\log r(z)
    -\lambda_T^\star\log\LR(z).
\end{aligned}
\]
With $\nu=\lambda_S^\star+\lambda_T^\star$ and
$\alpha=\lambda_T^\star/\nu$, this gives
\[
    r(z)
    =
    \exp\{-\eta^\star/\nu\}
    \exp\{L_h(z)/\nu\}\LR(z)^\alpha.
\]
The normalization constraint $\E_{P_S}[r]=1$ determines the multiplicative
constant. Multiplying by $p_S$ and using
$p_S\LR^\alpha=p_S^{1-\alpha}(p_T^\circ)^\alpha$ proves
\eqref{eqn:adv-bridge-klkl}.

For the HD/HD bridge, put
\[
    B(z)=\lambda_S^\star+\lambda_T^\star\sqrt{\LR(z)}.
\]
Then
$\phi_{\mathrm{HD}}'(t)=1-t^{-1/2}$. Therefore,
\begin{align*}
    L_h(z)-\eta^\star
    =
    \lambda_S^\star
      \left\{1-\frac{1}{\sqrt{r(z)}}\right\}
    +
    \lambda_T^\star
      \left\{1-
      \sqrt{\frac{\LR(z)}{r(z)}}\right\}
    =
    \lambda_+
    -
    \frac{B(z)}{\sqrt{r(z)}}.
\end{align*}
Thus, with
$\nu=\lambda_++\eta^\star$,
\[
    \sqrt{r(z)}
    =
    \frac{B(z)}{\nu-L_h(z)}.
\]
Indeed, writing the scalar pointwise objective in terms of $s=\sqrt t$
gives
\[
    -\{\nu-L_h(z)\}s^2+2B(z)s
    -\lambda_S^\star-\lambda_T^\star\LR(z).
\]
Consequently, the pointwise conjugate is finite precisely when
$L_h(z)<\nu$, and then equals
\[
    \frac{B(z)^2}{\nu-L_h(z)}
    -\lambda_S^\star-\lambda_T^\star\LR(z).
\]
Thus the exact dual-domain conditions are
\[
    L_h<\nu\quad P_S\text{-almost surely},
    \qquad
    \E_{P_S}\left[\frac{B^2}{\nu-L_h}\right]<\infty.
\]
The uniform separation assumed in the proposition is sufficient for both
conditions because $\E_{P_S}[B^2]<\infty$. Since
$B=\lambda_+\{(1-\alpha)+\alpha\sqrt\LR\}$, normalization gives
\[
    1
    =
    \lambda_+^2\E_{P_S}\left[
      \frac{\{(1-\alpha)+\alpha\sqrt\LR\}^2}
           {(\nu-L_h)^2}
    \right].
\]
Multiplying by $p_S$ now gives
\eqref{eqn:adv-bridge-hdhd}.

Finally, for the KL/HD bridge, the pointwise KKT equation is
\begin{equation}\label{eqn:proof-klhd-bridge-foc}
    L_h(z)-\eta^\star
    =
    \lambda_S^\star\log r(z)
    +
    \lambda_T^\star
    \left\{1-
      \sqrt{\frac{\LR(z)}{r(z)}}
    \right\}.
\end{equation}
Put $s(z)=\sqrt{r(z)}$ and $\nu=2\lambda_S^\star$. Rearranging
\eqref{eqn:proof-klhd-bridge-foc} yields
\begin{equation}\label{eqn:klhd-lambert-equation}
    -\log s(z)
    +
    \frac{\lambda_T^\star}{2\lambda_S^\star}
    \frac{\sqrt{\LR(z)}}{s(z)}
    =
    -\left\{\frac{L_h(z)}{\nu}+\gamma\right\},
\end{equation}
where
\[
    \gamma
    =
    -\frac{\eta^\star+\lambda_T^\star}{2\lambda_S^\star}
\]
is a scalar. Under the parameterization in
\eqref{eqn:klhd-dual-parameters},
$\lambda_T^\star/(2\lambda_S^\star)=\alpha/(1-\alpha)$, and hence the
second term on the left-hand side of~\eqref{eqn:klhd-lambert-equation} is
$\beta_\alpha(z)/s(z)$. If
$A_h(z)=L_h(z)/\nu+\gamma$, then
\[
    -\log s(z)+\frac{\beta_\alpha(z)}{s(z)}=-A_h(z).
\]
Setting $v(z)=\beta_\alpha(z)/s(z)$ gives
$v(z)e^{v(z)}=\beta_\alpha(z)e^{-A_h(z)}$. Therefore,
\[
    v(z)
    =
    W_0\!\left(
      \beta_\alpha(z)e^{-A_h(z)}
    \right),
    \qquad
    s(z)
    =
    \frac{\beta_\alpha(z)}{
      W_0\{\beta_\alpha(z)e^{-A_h(z)}\}},
\]
which proves~\eqref{eqn:adv-bridge-klhd} after multiplying by $p_S$.
For fixed interior $(\alpha,\nu)$, put
\[
    W_\gamma(z)
    =
    W_0\!\left(
      \beta_\alpha(z)
      e^{-L_h(z)/\nu-\gamma}
    \right).
\]
Since
\[
    r_\gamma(z)
    =
    \exp\left\{
      \frac{2L_h(z)}{\nu}
      +2\gamma
      +2W_\gamma(z)
    \right\},
    \qquad
    \frac{\partial W_\gamma(z)}{\partial\gamma}
    =
    -\frac{W_\gamma(z)}{1+W_\gamma(z)},
\]
we have
\[
    \frac{\partial r_\gamma(z)}{\partial\gamma}
    =
    \frac{2r_\gamma(z)}{1+W_\gamma(z)}
    >0.
\]
Thus $\E_{P_S}[r_\gamma]$ is strictly increasing wherever finite.
The assumed attained optimizer supplies a normalizing value, and strict
monotonicity makes that value unique. When $\alpha=0$, the limit
$\beta/W_0(\beta e^{-A})\to e^A$ recovers the source-only KL tilt.
\end{proof}

\begin{remark}[Attainment and interior regimes]
For the KL/KL and KL/HD pairings, the feasible set of~\eqref{eqn:primal-density-ratio-problem} is weakly compact in $L^1(P_S)$.
It is uniformly integrable because $\phi_{\mathrm{KL}}(t)/t\to\infty$ \citep[de la Vall\'ee-Poussin's criterion;][Theorem~4.5.9]{Bogachev2007}, and hence relatively weakly compact \citep[Dunford--Pettis theorem;][Theorem~4.7.18]{Bogachev2007}.
It is convex and norm closed, because Fatou's lemma along almost surely convergent subsequences makes $\mathcal G_S$ and $\mathcal G_T$ lower semicontinuous on $L^1(P_S)$; by Mazur's theorem it is therefore weakly closed.
Since $L_h\in L^\infty(P_S)$, the map $r\mapsto\E_{P_S}[rL_h]$ is weakly continuous, so the dominated primal supremum is attained.
The KL/HD formula above applies to the source-KL-active regime $\lambda_S^\star>0$ and represents a genuinely two-anchor balance when also $\lambda_T^\star>0$. When $\lambda_S^\star=0$, the pointwise optimality equation has the target-anchor Hellinger form; the source constraint remains part of the primal feasible set.
For HD/HD, squared Hellinger divergence has finite recession slope, so the ambiguity set need not be uniformly integrable or weakly compact, and primal attainment is not automatic.
The proposition therefore states the formula for an attained optimizer in the uniformly separated regime $\nu>\operatorname*{ess\,sup}_{P_S}L_h$.
Boundary solutions must instead be checked using the pointwise and integrability conditions in the proof.
\end{remark}

\begin{remark}[Interpretation of the mixing parameter $\alpha$]
The parameter $\alpha$ acts as a continuous source--target dual-balance parameter, although its normalization is geometry-specific.
At $\alpha=0$, the target-side multiplier vanishes and the pointwise density has the
source-anchor form.
For the symmetric pairings, $\alpha=1$ gives the corresponding target-anchor form. The other constraint remains part of the primal feasible set at either endpoint.
For KL/HD, the Lambert-$W$ formula is derived under $\lambda_S>0$, and its $\alpha=1$ endpoint is treated through the target-anchor Hellinger pointwise equation rather than by substitution into \eqref{eqn:adv-bridge-klhd}.
Interior coordinates borrow source information through a tempered likelihood ratio. Appendix~\ref{sec:radius-indexed-selection} identifies the feasible symmetric structural centers and quantifies the contribution of their weights to the effective sample size.
\end{remark}

\begin{remark}[Interpretation of the geometries]
In the KL/KL setting, the ambiguity geometry forms a log-linear density bridge, while the supervised loss produces the classical exponential tilt toward high-loss regions \citep{Ben-Tal1987,HuHong2012}.
With HD at both anchors, the densities enter through their square roots,
while the Hellinger geometry produces the loss tilt $(\nu-L_h)^{-2}$,
rather than the exponential tilt arising under KL.
After removing the loss tilt and normalizing, the KL/KL structural bridge is
the familiar exponential, or $e$-geodesic, density path used in path sampling
\citep{Amari2016,GelmanMeng1998}. The normalized HD/HD path is the radial
normalization of the chord between the two square-root densities; it traces
the corresponding Fisher--Rao arc, although $\alpha$ is not its
constant-speed parameter \citep{Amari2016}.

The KL/HD Lambert-$W$ bridge is geometrically different because the two penalty geometries are coupled in the same pointwise equation.
The likelihood ratio enters through $\beta_\alpha$, while the loss $L_h/\nu$ and the mass-normalizing dual offset $\gamma$ enter the Lambert-$W$ argument.
Since $W_0(x)$ grows logarithmically \citep{Corless1996}, the KL/HD density map remains unbounded but attenuates large likelihood ratios by a logarithmic factor.
More precisely, for a fixed interior structural parameter its loss-agnostic form is of order $u/(\log u)^2$ as $u\to\infty$, whereas the HD/HD map grows asymptotically linearly.
The KL/KL map, on the other hand, grows polynomially, as $u^\alpha$, at every interior coordinate, so the three geometries order the tempering of extreme ratios as KL/KL, followed by KL/HD and then HD/HD.
The KL/HD bridge therefore provides a geometry-induced alternative to explicit truncation \citep{Ma2023}.
\end{remark}

\subsection{Loss-Agnostic Structural Bridge Weights}
\label{sec:loss-agnostic-density-bridges}

The structural families below have one coordinate, written $\bm\theta=(\alpha)$ in the statistical results. Removing the localized loss from the responses in Proposition~\ref{prop:loss-adversarial-bridges} gives three hypothesis-independent reference laws. For $B\in\{\mathrm{KL},\mathrm{HD},\mathrm{KH}\}$, denoting the KL/KL, HD/HD and KL/HD pairings, respectively, write
\begin{equation}\label{eqn:normalized-structural-bridge-law}
 w_\alpha^B=g_\alpha^B(\LR),\qquad
 \frac{\d P_\alpha^B}{\d P_S}=\frac{w_\alpha^B}{\E_{P_S}w_\alpha^B},\qquad
 R_\alpha^B(h)=\frac{\E_{P_S}[w_\alpha^B L_h]}{\E_{P_S}w_\alpha^B}.
\end{equation}
The three weight maps are
\begin{align}
 g_\alpha^{\mathrm{KL}}(u)&=u^\alpha, &&0\leq\alpha\leq1, \label{eqn:bridge-weight-kl}\\
 g_\alpha^{\mathrm{HD}}(u)&=(1-\alpha+\alpha\sqrt u)^2, &&0\leq\alpha\leq1, \label{eqn:bridge-weight-hd}\\
 g_\alpha^{\mathrm{KH}}(u)&=\exp\{2W_0(a_\alpha\sqrt u)\},\quad
 a_\alpha=\frac{\alpha}{1-\alpha}, &&0\leq\alpha<1. \label{eqn:bridge-weight-klhd}
\end{align}
We set $g_0^{\mathrm{KL}}(0)=1$. The symmetric laws are exactly the barycentric centers of Theorem~\ref{thm:symmetric-single-ball-representation}. For KL/HD the scalar offset $\gamma$ in Proposition~\ref{prop:loss-adversarial-bridges} can be absorbed into the structural coordinate. Define
\begin{equation}\label{eqn:kh-effective-alpha}
  \bar\alpha=\frac{\alpha e^{-\gamma}}{1-\alpha+\alpha e^{-\gamma}}.
\end{equation}
Then $a_{\bar\alpha}=a_\alpha e^{-\gamma}$ and
\begin{equation}\label{eqn:kh-structural-gauge-equivalence}
 \left[\frac{a_\alpha\sqrt u}{W_0(a_\alpha e^{-\gamma}\sqrt u)}\right]^2
 =e^{2\gamma}g_{\bar\alpha}^{\mathrm{KH}}(u).
\end{equation}
The identity extends by continuity at zero. The factor $e^{2\gamma}$ cancels under normalization, so we use $\gamma=0$ and write $\alpha$ for the effective structural coordinate. The same substitution applies to the normalized loss-aware response at fixed $h,\nu$, with $e^{-L_h/\nu}$ retained inside $W_0$. In that case the original multiplier ratio also depends on the normalizing offset.

All three paths start at $P_S$, and the symmetric paths end at $P_T^\circ$. For KL/HD, let $k_\alpha(u)=g_\alpha^{\mathrm{KH}}(u)/g_\alpha^{\mathrm{KH}}(1)$. For $u>0$ and $0<\alpha<1$,
\[
 k_\alpha(u)=u\left\{\frac{W_0(a_\alpha)}{W_0(a_\alpha\sqrt u)}\right\}^{\!2}
 \longrightarrow u\qquad(\alpha\uparrow1),
\]
since $W_0(t)\sim\log t$. The limit also holds at $u=0$. Monotonicity of $k_\alpha$ in $u$ and continuity of the limit give uniform convergence on compact intervals. Monotonicity of $W_0$ in this identity for $u\ge1$, and monotonicity of $g_\alpha^{\mathrm{KH}}$ for $u\le1$, give $k_\alpha(u)\le\max(1,u)$. Since $\E_{P_S}\LR=1$, the normalizers are finite and dominated convergence gives $P_\alpha^{\mathrm{KH}}\to P_T^\circ$ in total variation. Figure~\ref{fig:bridge-weights-bygeom} shows the scale-normalized maps. Section~\ref{sec:nonparametric-density-estimation} bounds the plug-in error. Appendix~\ref{app:asymptotics-loss-agnostic-bridge-weights} gives uniform consistency and pointwise limit theory for the estimated bridge weights.

\begin{figure}[t]
    \centering
    \includegraphics[width=0.9\linewidth]{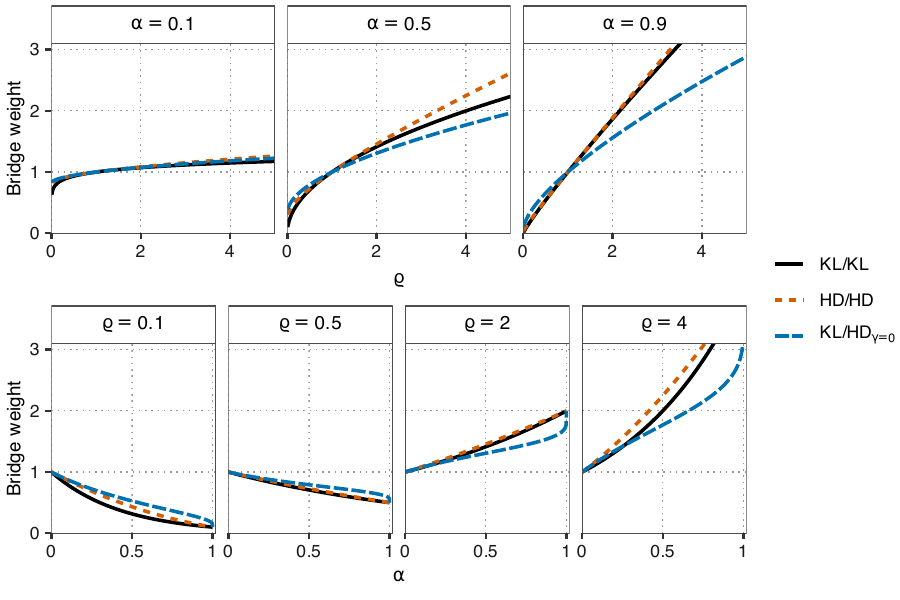}
    \caption{Scale-normalized bridge maps $g_{\alpha}^B(\LR)/g_{\alpha}^B(1)$ as a function of the likelihood
ratio $\LR$ (top row) or as a function of the bridge parameter $\alpha$ (bottom row).
    The endpoint $\alpha=0$ corresponds to source-only weighting and $\alpha=1$ is raw importance weighting.
    For KL/HD, the scale-normalized map is extended to $\alpha=1$ by its pointwise likelihood-ratio limit.
    Interior KL/KL curves are the flattened weights of \citet{Shimodaira2000} and interior HD/HD curves are the interpolated weights of \citet{Kimura2022}.
    }
    \label{fig:bridge-weights-bygeom}
\end{figure}

\section{Bridge risk bounds}\label{sec:bridge-risk-bounds}

A primary goal in domain adaptation is to control the prediction risk of a
hypothesis under the true target law,
\begin{equation*}
    R_T(h)=\E_{P_T}[L_h(Z)].
\end{equation*}
When target labels are unavailable, this risk cannot be estimated directly.
Section~\ref{sec:general-dual-representation} defines the loss-aware robust risk
\[
    \Psi_{S,T}(h)
    =
    \sup_{Q\in\as(P_S,P_T^\circ)}\E_Q[L_h],
\]
whereas Section~\ref{sec:loss-agnostic-density-bridges} defines the structural bridge risk
\[
    R_{\bm\theta}^{B}(h)
    =
    \frac{\E_{P_S}[w_{\bm\theta}^{B}L_h]}
         {\E_{P_S}[w_{\bm\theta}^{B}]}.
\]
The first is a one-sided worst-case envelope whose inner optimization depends
on the loss. In the second, the bridge weights are independent of $h$ for each
$\bm\theta$. Our finite-sample analysis therefore studies the empirical
structural risks uniformly over $h\in\hs$ and over the class of bridge weights
indexed by $\bm\theta$. We first establish the population target locality of
both risks and then derive uniform finite-sample bounds for the empirical
structural bridge.

\subsection{Population robust-risk envelope}

Recall from Section~\ref{sec:setup-and-assumptions} that, in the target-unlabeled setting,
\[
    P_T^\circ(\mathrm{d} x,\mathrm{d} y)
    =P_T^X(\mathrm{d} x)P_S^{Y\mid X}(\mathrm{d} y\mid x).
\]
The population discrepancy between the true target law and this target reference is
\begin{equation}\label{eqn:target-completion-discrepancy}
    \delta_T^\circ
    :=D_{\phi_T}(P_T\Vert P_T^\circ).
\end{equation}
By construction, $P_T^\circ$ and $P_T$ have the same covariate marginal
$P_T^X$. Disintegration of the extended $\phi_T$-divergence therefore gives
\begin{equation*}
    \delta_T^\circ
    =
    \int_{\xspace}
    D_{\phi_T}\left(
       P_T^{Y\mid X=x}\Big\Vert P_S^{Y\mid X=x}
    \right)\,\d P_T^X(x).
\end{equation*}
Here and below, the conditional divergences are understood in the extended
sense of~\eqref{eqn:extended-phi-divergence}.
Thus $\delta_T^\circ$ measures the discrepancy between the unobserved target
conditional law and the source conditional law used to complete the target
anchor. It vanishes under covariate shift, but Assumption~\ref{ass:target-feasibility}
allows it to be positive, subject to $\delta_T^\circ\leq\rho_T$.

We first record the standard relation between distributional distance and
risk.

\begin{lemma}[Risk difference from distribution distance]
\label{lem:risk-difference-tv}
Under Assumption~\ref{ass:bounded-loss}, for any probability laws $P$ and
$Q$ on $\fs$ and every $h\in\hs$,
\begin{equation*}
    \left|\E_P[L_h]-\E_Q[L_h]\right|
    \leq
    M\operatorname{TV}(P,Q).
\end{equation*}
\end{lemma}

\begin{proof}
Under the convention
$\operatorname{TV}(P,Q)=\sup_A|P(A)-Q(A)|$, the variational
characterization of total variation \citep[Section~2]{Gibbs2002} gives
\[
    \operatorname{TV}(P,Q)
    =
    \sup_{0\leq f\leq 1}
    \left|\E_P[f]-\E_Q[f]\right|
    .
\]
Since $0\leq L_h/M\leq1$, the result follows by taking $f=L_h/M$.
\end{proof}

\begin{theorem}[Population robust-risk envelope]
\label{thm:population-robust-risk-envelope}
Suppose Assumptions~\ref{ass:bounded-loss}--\ref{ass:target-feasibility}
hold, and define the target-locality term
\begin{equation}\label{eqn:population-target-locality}
    \Delta_T^\circ(\rho_T)
    =
    \sup_{\substack{Q:\,
      D_{\phi_T}(Q\Vert P_T^\circ)\leq\rho_T}}
    \sup_{h\in\hs}
    \left|\E_Q[L_h]-R_T(h)\right|.
\end{equation}
Then, for every $h\in\hs$,
\begin{equation*}
    R_T(h)
    \leq
    \Psi_{S,T}(h)
    \leq
    R_T(h)+\Delta_T^\circ(\rho_T).
\end{equation*}
\end{theorem}

\begin{proof}
By Assumption~\ref{ass:target-feasibility},
\[
    D_{\phi_S}(P_T\Vert P_S)\leq\rho_S
    \quad\text{and}\quad
    D_{\phi_T}(P_T\Vert P_T^\circ)\leq\rho_T.
\]
Hence $P_T\in\as(P_S,P_T^\circ)$, and therefore
\[
    R_T(h)
    =\E_{P_T}[L_h]
    \leq
    \sup_{Q\in\as(P_S,P_T^\circ)}\E_Q[L_h]
    =\Psi_{S,T}(h),
\]
which proves the lower bound. Every
$Q\in\as(P_S,P_T^\circ)$ also satisfies
$D_{\phi_T}(Q\Vert P_T^\circ)\leq\rho_T$. By the definition of
$\Delta_T^\circ(\rho_T)$,
\[
    \E_Q[L_h]
    \leq
    R_T(h)+\Delta_T^\circ(\rho_T).
\]
Taking the supremum over the doubly anchored ambiguity set proves the upper
bound.
\end{proof}

The target-locality term can be bounded in terms of $\rho_T$ and
$\delta_T^\circ$ for the target-side divergences used in this paper. Write $\Psi_{\mathrm{KL}}$,
$\Psi_{\mathrm{HD}}$, and $\Psi_{\mathrm{KH}}$ for the robust risk
under the KL/KL, HD/HD, and KL/HD pairings, respectively.

\begin{corollary}[Divergence-specific population bounds]
\label{cor:population-bridge-specializations}
Under the conditions of Theorem~\ref{thm:population-robust-risk-envelope},
the following bounds hold.

\begin{enumerate}
\item If the target-side divergence is KL, then
\begin{equation}\label{eqn:population-kl-locality}
    \Delta_T^\circ(\rho_T)
    \leq
    M\left\{
      \sqrt{\frac{\rho_T}{2}}
      +
      \sqrt{\frac{\delta_T^\circ}{2}}
    \right\}
    \leq
    M\sqrt{2\rho_T}.
\end{equation}
Consequently,
\[
    R_T(h)
    \leq
    \Psi_{\mathrm{KL}}(h)
    \leq
    R_T(h)
    +M\left\{
      \sqrt{\frac{\rho_T}{2}}
      +
      \sqrt{\frac{\delta_T^\circ}{2}}
    \right\}.
\]

\item If the target-side divergence is the squared Hellinger distance, then
\begin{equation}\label{eqn:population-hd-locality}
    \Delta_T^\circ(\rho_T)
    \leq
    M\left\{
      \sqrt{\rho_T}
      +
      \sqrt{\delta_T^\circ}
    \right\}
    \leq
    2M\sqrt{\rho_T}.
\end{equation}
Consequently, for $B\in\{\mathrm{HD},\mathrm{KH}\}$,
\begin{equation}\label{eqn:population-hd-risk-envelope}
    R_T(h)
    \leq
    \Psi_B(h)
    \leq
    R_T(h)
    +M\left\{
      \sqrt{\rho_T}
      +
      \sqrt{\delta_T^\circ}
    \right\}.
\end{equation}
\end{enumerate}
\end{corollary}

\begin{proof}
Fix $h\in\hs$ and a probability law $Q$ satisfying
$D_{\mathrm{KL}}(Q\Vert P_T^\circ)\leq\rho_T$.
Lemma~\ref{lem:risk-difference-tv} first gives
\[
    \left|\E_Q[L_h]-R_T(h)\right|
    \leq M\operatorname{TV}(Q,P_T).
\]
The triangle inequality for total variation, followed by Pinsker's inequality
applied to $(Q,P_T^\circ)$ and $(P_T,P_T^\circ)$, then gives
\begin{align*}
    \operatorname{TV}(Q,P_T)
    &\leq
    \operatorname{TV}(Q,P_T^\circ)
    +\operatorname{TV}(P_T^\circ,P_T)\\
    &\leq
    \sqrt{\frac{D_{\mathrm{KL}}(Q\Vert P_T^\circ)}{2}}
    +\sqrt{\frac{D_{\mathrm{KL}}(P_T\Vert P_T^\circ)}{2}}\\
    &\leq
    \sqrt{\frac{\rho_T}{2}}
    +\sqrt{\frac{\delta_T^\circ}{2}}.
\end{align*}
The target-side ball controls the first divergence in the last display, while
the definition of $\delta_T^\circ$ identifies the second. Taking the suprema
in~\eqref{eqn:population-target-locality} proves the first inequality
in~\eqref{eqn:population-kl-locality}. The second follows from
$\delta_T^\circ\leq\rho_T$, which is part of
Assumption~\ref{ass:target-feasibility}. Substitution into
Theorem~\ref{thm:population-robust-risk-envelope} gives the stated KL robust-risk
bound.

For the Hellinger target geometry, write
$H^2(P,Q)=D_{\phi_{\mathrm{HD}}}(P\Vert Q)$. Under the normalization used in
this paper, $\operatorname{TV}(P,Q)\leq H(P,Q)$, and $H$ is a metric. Thus,
for every $Q$ satisfying $D_{\phi_{\mathrm{HD}}}(Q\Vert P_T^\circ)\leq\rho_T$,
\begin{align*}
    \left|\E_Q[L_h]-R_T(h)\right|
    &\leq M\operatorname{TV}(Q,P_T)\\
    &\leq M H(Q,P_T)\\
    &\leq M\{H(Q,P_T^\circ)+H(P_T^\circ,P_T)\}\\
    &=M\left\{
      \sqrt{D_{\phi_{\mathrm{HD}}}(Q\Vert P_T^\circ)}
      +\sqrt{D_{\phi_{\mathrm{HD}}}(P_T\Vert P_T^\circ)}
    \right\}\\
    &\leq M\left\{\sqrt{\rho_T}+\sqrt{\delta_T^\circ}\right\}.
\end{align*}
Here the first inequality uses Lemma~\ref{lem:risk-difference-tv}, the second
uses $\operatorname{TV}\leq H$, the third uses the triangle inequality for
$H$, and the last uses the target-side ball together with the definition of
$\delta_T^\circ$. Taking the suprema proves the first inequality
in~\eqref{eqn:population-hd-locality}; the second again follows from target
feasibility. Substitution into
Theorem~\ref{thm:population-robust-risk-envelope} gives
\eqref{eqn:population-hd-risk-envelope} for both pairings having target-side
Hellinger geometry.
\end{proof}

For the HD/HD pairing, Assumption~\ref{ass:target-feasibility} ensures that the
full ambiguity set is nonempty. Writing $\rho_-=\min\{\rho_S,\rho_T\}$ and
$\tilde\rho=1-(1-\rho_-/2)^2$, Corollary~\ref{cor:population-bridge-specializations}
and Lemma~\ref{lem:full-vs-ac-risk-hdhd} give, whenever $\rho_-<2$,
\[
    R_T(h)-M\tilde\rho
    \leq\Psi_{S,T}^{\mathrm{ac}}(h)
    \leq R_T(h)+M\{\sqrt{\rho_T}+\sqrt{\delta_T^\circ}\}.
\]
Under the conditions of Proposition~\ref{prop:loss-adversarial-bridges},
including primal attainment and the stated HD/HD multiplier regime, the HD/HD
dual-induced density attains $\Psi_{S,T}^{\mathrm{ac}}(h)$.

\begin{remark}[Structure of the population envelope]
The source-side divergence constraint can make the doubly-anchored ambiguity
set strictly smaller than the target-side ball; when it is active and the
supremum is attained, it also shapes the loss-aware adversarial optimizer. The displayed upper bounds
on population target locality depend on the target-side radius together with
the discrepancy $\delta_T^\circ$ between the true target conditional law and
the source conditional law used to complete the target reference. Under
covariate shift, $P_T=P_T^\circ$ and $\delta_T^\circ=0$, so the conditional-shift
term vanishes and the upper-envelope increments
in~\eqref{eqn:population-kl-locality} and~\eqref{eqn:population-hd-locality}
reduce to $M\sqrt{\rho_T/2}$ and $M\sqrt{\rho_T}$, respectively. More generally,
along any sequence satisfying target feasibility, $\delta_T^\circ\leq\rho_T$;
hence the displayed envelope widths vanish as $\rho_T\to0$. The dual and KKT
characterization in Theorem~\ref{thm:general-dual-representation} additionally
assumes that the two radii admit a common anchor-dominated density satisfying
both divergence inequalities strictly, as in
\eqref{eqn:dual-slater-condition}. This joint strict-feasibility condition is
additional to target feasibility.
\end{remark}

\subsection{Target locality of the structural bridges}

The empirical-process analysis below concerns the structural bridge risk
$R_{\bm\theta}^B$, rather than the full loss-aware robust envelope. Since the
structural bridge weight depends on $z=(x,y)$ only through the covariate
likelihood ratio $\LR(x)$, write
$w_{\bm\theta}^B(x)=g_{\bm\theta}^B\{\LR(x)\}$. Let
$P_{\bm\theta}^{B,X}$ denote the covariate marginal of the bridge law. From
\eqref{eqn:normalized-structural-bridge-law},
\[
    \frac{\mathrm{d}P_{\bm\theta}^{B,X}}{\mathrm{d}P_S^X}(x)
    =
    \frac{w_{\bm\theta}^B(x)}{\E_{P_S}[w_{\bm\theta}^B]}.
\]
Consequently,
\begin{equation*}
    P_{\bm\theta}^B(\mathrm{d} x,\mathrm{d} y)
    =
    P_{\bm\theta}^{B,X}(\mathrm{d} x)
    P_S^{Y\mid X}(\mathrm{d} y\mid x).
\end{equation*}
Thus, under both $P_{\bm\theta}^B$ and $P_T^\circ(\mathrm{d} x,\mathrm{d} y)$ the conditional distribution of $Y$ given $X=x$ is
$P_S^{Y\mid X=x}$; the two laws differ only in their covariate marginals.
Assumption~\ref{ass:common-domination} implies $P_S^X\sim P_T^X$. Since
$P_{\bm\theta}^{B,X}\ll P_S^X$, it follows that
$P_{\bm\theta}^{B,X}\ll P_T^X$, and hence
\[
    \frac{\mathrm{d} P_{\bm\theta}^B}{\mathrm{d} P_T^\circ}(x,y)
    =
    \frac{\mathrm{d} P_{\bm\theta}^{B,X}}{\mathrm{d} P_T^X}(x)
    \qquad
    P_T^\circ\text{-almost surely}.
\]
Substituting this derivative into the extended $\phi$-divergence and
integrating first with respect to the common conditional distribution gives
\begin{equation}\label{eqn:structural-marginal-divergence}
    D_\phi(P_{\bm\theta}^B\Vert P_T^\circ)
    =
    D_\phi(P_{\bm\theta}^{B,X}\Vert P_T^X).
\end{equation}
The same calculation with $P_S$ in place of $P_T^\circ$ gives
\begin{equation}\label{eqn:structural-marginal-source-divergence}
D_{\phi_S}(P_{\bm\theta}^B\Vert P_S)=D_{\phi_S}(P_{\bm\theta}^{B,X}\Vert P_S^X).
\end{equation}
More generally, under Assumption~\ref{ass:common-domination}, let $Q^X$ be a
covariate law satisfying $Q^X\ll P_S^X$ and define
\[
    Q(\mathrm{d}x,\mathrm{d}y)
    =Q^X(\mathrm{d}x)P_S^{Y\mid X}(\mathrm{d}y\mid x).
\]
Since $P_S^X\sim P_T^X$, we also have $Q^X\ll P_T^X$. Writing
\[
    r_S^X=\frac{\mathrm{d}Q^X}{\mathrm{d}P_S^X},
    \qquad
    r_T^X=\frac{\mathrm{d}Q^X}{\mathrm{d}P_T^X},
\]
the common conditional distribution gives
\[
    \frac{\mathrm{d}Q}{\mathrm{d}P_S}(x,y)=r_S^X(x),
    \qquad
    \frac{\mathrm{d}Q}{\mathrm{d}P_T^\circ}(x,y)=r_T^X(x).
\]
Therefore,
\begin{equation}\label{eqn:completed-marginal-divergences}
\begin{aligned}
    D_{\phi_S}(Q\Vert P_S)
    &=\int_{\xspace}\phi_S\{r_S^X(x)\}\,\mathrm{d}P_S^X(x)
      =D_{\phi_S}(Q^X\Vert P_S^X),\\
    D_{\phi_T}(Q\Vert P_T^\circ)
    &=\int_{\xspace}\phi_T\{r_T^X(x)\}\,\mathrm{d}P_T^X(x)
      =D_{\phi_T}(Q^X\Vert P_T^X).
\end{aligned}
\end{equation}
Because $Q$ is dominated by both anchors, no singular recession term enters
these identities.

\begin{proposition}[Target locality of the structural bridge]
\label{prop:structural-bridge-target-locality}
Under Assumptions~\ref{ass:bounded-loss}--\ref{ass:common-domination}, let
$B\in\{\mathrm{KL},\mathrm{HD},\mathrm{KH}\}$, and let $\bm\theta$ lie in the
parameter domain of the corresponding bridge family. Define
\begin{equation*}
    \delta_{B,T}(\bm\theta)
    =
    D_{\phi_T}(P_{\bm\theta}^{B,X}\Vert P_T^X)
\end{equation*}
and
\begin{equation*}
    \Delta_B(\bm\theta)
    =
    \sup_{h\in\hs}
    \left|R_{\bm\theta}^B(h)-R_T(h)\right|.
\end{equation*}
Then, for every such $B$ and $\bm\theta$,
\begin{equation*}
    \Delta_B(\bm\theta)
    \leq
    M\operatorname{TV}(P_{\bm\theta}^B,P_T).
\end{equation*}
If the target-side divergence is KL, then
\begin{equation}\label{eqn:structural-locality-kl}
    \Delta_B(\bm\theta)
    \leq
    M\left\{
      \sqrt{\frac{\delta_{B,T}(\bm\theta)}{2}}
      +
      \sqrt{\frac{\delta_T^\circ}{2}}
    \right\}.
\end{equation}
If the target-side divergence is squared Hellinger distance, then
\begin{equation}\label{eqn:structural-locality-hd}
    \Delta_B(\bm\theta)
    \leq
    M\left\{
      \sqrt{\delta_{B,T}(\bm\theta)}
      +
      \sqrt{\delta_T^\circ}
    \right\}.
\end{equation}
\end{proposition}

\begin{proof}
By~\eqref{eqn:normalized-structural-bridge-law} and
Lemma~\ref{lem:risk-difference-tv}, for every $h\in\hs$,
\begin{align*}
    \left|R_{\bm\theta}^B(h)-R_T(h)\right|
    &=
    \left|
      \E_{P_{\bm\theta}^B}[L_h]
      -\E_{P_T}[L_h]
    \right|\\
    &\leq
    M\operatorname{TV}(P_{\bm\theta}^B,P_T).
\end{align*}
Taking the supremum over $h$ proves the first bound.

Suppose first that the target-side divergence is KL. The triangle inequality
for total variation, followed by Pinsker's inequality, gives
\begin{align*}
    \operatorname{TV}(P_{\bm\theta}^B,P_T)
    &\leq
    \operatorname{TV}(P_{\bm\theta}^B,P_T^\circ)
    +\operatorname{TV}(P_T^\circ,P_T)\\
    &\leq
    \sqrt{\frac{D_{\mathrm{KL}}(P_{\bm\theta}^B\Vert P_T^\circ)}{2}}
    +\sqrt{\frac{D_{\mathrm{KL}}(P_T\Vert P_T^\circ)}{2}}\\
    &=
    \sqrt{\frac{\delta_{B,T}(\bm\theta)}{2}}
    +\sqrt{\frac{\delta_T^\circ}{2}}.
\end{align*}
In the second line, symmetry of total variation is used before applying
Pinsker's inequality to $P_T$ and $P_T^\circ$. The last equality follows from
\eqref{eqn:structural-marginal-divergence} and
\eqref{eqn:target-completion-discrepancy}. Combining this display with the
first bound proves~\eqref{eqn:structural-locality-kl}.

For the Hellinger target geometry, let
$H^2(P,Q)=D_{\phi_{\mathrm{HD}}}(P\Vert Q)$.
Because $\operatorname{TV}(P,Q)\leq H(P,Q)$ and $H$ is a metric,
\begin{align*}
    \Delta_B(\bm\theta)
    &\leq M H(P_{\bm\theta}^B,P_T)\\
    &\leq
    M\{H(P_{\bm\theta}^B,P_T^\circ)+H(P_T^\circ,P_T)\}\\
    &=
    M\left\{
      \sqrt{\delta_{B,T}(\bm\theta)}
      +\sqrt{\delta_T^\circ}
    \right\}.
\end{align*}
The first inequality uses $\operatorname{TV}\leq H$ together with the first
bound, the second uses the triangle inequality for $H$, and the equality uses
\eqref{eqn:structural-marginal-divergence},
\eqref{eqn:target-completion-discrepancy}, and symmetry of $H$. This
proves~\eqref{eqn:structural-locality-hd}.
\end{proof}

\begin{remark}[Marginal and conditional components]
The quantity $\delta_{B,T}(\bm\theta)$ depends only on $P_S^X$ and
$P_T^X$ and measures how closely the structural bridge matches the target
covariate law. It therefore admits an empirical analogue based on the source
and unlabeled target covariates. By contrast, $\delta_T^\circ$ depends on the
unobserved target conditional law and is not identified from the observed data
without additional assumptions. Under
Assumption~\ref{ass:target-feasibility}, $\delta_T^\circ\leq\rho_T$. Hence, for
any bridge parameter satisfying $\delta_{B,T}(\bm\theta)\leq\rho_T$, the bounds
in~\eqref{eqn:structural-locality-kl} and
\eqref{eqn:structural-locality-hd} yield
$\Delta_B(\bm\theta)\leq M\sqrt{2\rho_T}$ for target-side KL and
$\Delta_B(\bm\theta)\leq2M\sqrt{\rho_T}$ for target-side squared Hellinger
distance.
\end{remark}

\subsection{Empirical structural bridge risk}

Because the bridge weights depend only on the covariate, when they are
evaluated at a joint observation $z=(x,y)$ we use the convention
$w_{\bm\theta}^B(z)=w_{\bm\theta}^B(x)$, and likewise for
$\widehat w_{\bm\theta}^B$.
Let
\[
    P_n f
    =
    \frac1n\sum_{i=1}^n f(Z_i^S)
\]
denote the empirical source measure. For a bridge family $B$ and parameter
$\bm\theta$, let $\widehat w_{\bm\theta}^B$ be a nonnegative measurable
estimator of $w_{\bm\theta}^B$. In the density plug-in construction considered
below, it is obtained by substituting an estimated covariate likelihood ratio
into the same bridge map $g_{\bm\theta}^B$. Define the oracle empirical risk
\begin{equation}\label{eqn:oracle-empirical-bridge-risk}
    R_{n,\bm\theta}^B(h)
    =
    \frac{P_n(w_{\bm\theta}^B L_h)}{P_nw_{\bm\theta}^B}
\end{equation}
and, whenever $P_n\widehat w_{\bm\theta}^B>0$, the plug-in empirical risk
\begin{equation}\label{eqn:plugin-empirical-bridge-risk}
    \widehat R_{\bm\theta}^B(h)
    =
    \frac{P_n(\widehat w_{\bm\theta}^B L_h)}
         {P_n\widehat w_{\bm\theta}^B}.
\end{equation}
If $P_n\widehat w_{\bm\theta}^B=0$, set
$\widehat R_{\bm\theta}^B(h)=0$. This convention is immaterial on the
high-probability events used below, where the estimated normalizer is bounded
away from zero.

The oracle risk isolates source-sample variation under the population bridge
weights. The plug-in risk also contains the error from estimating those
weights; for the density plug-in construction, this includes estimation of
the source and target covariate densities. Although each normalized risk is
invariant under positive rescaling, the uniform weight classes and
perturbation bounds below refer to a fixed scale. We therefore use the
representatives $g_{\bm\theta}^B$ fixed in
Section~\ref{sec:loss-agnostic-density-bridges} and place their estimators on
the same scale.

Let $\Theta_B$ be a bridge-parameter set contained in the corresponding range
specified in Section~\ref{sec:loss-agnostic-density-bridges}. In particular,
the compactness imposed below means that
$\Theta_{\mathrm{KH}}\subseteq[0,\alpha_{\mathrm{KH},\max}]$ for some
$\alpha_{\mathrm{KH},\max}<1$ for the raw
KL/HD bridge map. On this fixed scale, define
\[
    \mathcal L_{\hs}
    =
    \{L_h:h\in\hs\},
\]
the population bridge-weight class
\begin{equation}\label{eqn:bridge-weight-class}
    \mathcal W_B
    =
    \{w_{\bm\theta}^B:\bm\theta\in\Theta_B\},
\end{equation}
and the weighted loss class
\begin{equation}\label{eqn:weighted-loss-class}
    \mathcal W_B\mathcal L_{\hs}
    =
    \{w_{\bm\theta}^B L_h:
       \bm\theta\in\Theta_B,\ h\in\hs\}.
\end{equation}
For a measurable function class $\mathcal F$, let
\begin{equation*}
    \mathfrak R_n(\mathcal F)
    =
    \E_{Z,\sigma}\left[
      \sup_{f\in\mathcal F}
      \left|
        \frac1n\sum_{i=1}^n\sigma_i f(Z_i^S)
      \right|
    \right],
\end{equation*}
where the expectation is over the source sample and independent Rademacher
variables $\sigma_1,\ldots,\sigma_n$.

We impose the following regularity condition on the structural bridge family.

\begin{assumption}[Uniform bridge-weight regularity]
\label{ass:bridge-weight-regularity}
For a fixed bridge family $B$, the parameter set $\Theta_B$ is compact.
There exists a measurable covariate region
$\mathcal X_B\subseteq\xspace$ satisfying
\[
    P_S^X(\mathcal X_B)=1
\]
and deterministic functions
$c_B(\cdot),C_B(\cdot)\colon\Theta_B\to(0,\infty)$ such that
\begin{equation}\label{eqn:uniform-weight-bounds}
    c_B(\bm\theta)
    \leq
    w_{\bm\theta}^B(x)
    \leq
    C_B(\bm\theta)
\end{equation}
for every $x\in\mathcal X_B$ and every
$\bm\theta\in\Theta_B$, where
\[
    c_B:=\inf_{\bm\theta\in\Theta_B}c_B(\bm\theta)>0,
    \qquad
    C_B:=\sup_{\bm\theta\in\Theta_B}C_B(\bm\theta)<\infty.
\]

For a measurable function $f$ on $\mathcal X_B$, write
\[
    \|f\|_{\infty,\mathcal X_B}
    =
    \sup_{x\in\mathcal X_B}|f(x)|.
\]
For every $\delta_0\in(0,1)$, there are deterministic functions
$\varepsilon_{n,m}^B(\cdot;\delta_0)\colon\Theta_B\to[0,\infty)$ such that
\begin{equation}\label{eqn:high-probability-weight-error}
    \Pr\left\{
      \|\widehat w_{\bm\theta}^B-w_{\bm\theta}^B
      \|_{\infty,\mathcal X_B}
      \leq
      \varepsilon_{n,m}^B(\bm\theta;\delta_0)
      \quad\text{for every }\bm\theta\in\Theta_B
    \right\}
    \geq
    1-\delta_0,
\end{equation}
and
\[
    \varepsilon_{n,m}^B(\delta_0)
    =
    \sup_{\bm\theta\in\Theta_B}
    \varepsilon_{n,m}^B(\bm\theta;\delta_0).
\]
The classes in
\eqref{eqn:bridge-weight-class}--\eqref{eqn:weighted-loss-class}
are assumed pointwise measurable in the standard empirical-process sense:
each has a countable subclass that is pointwise dense in the class. The
estimated-weight family is also assumed measurable so that the event in
\eqref{eqn:high-probability-weight-error} and the suprema below are measurable.
\end{assumption}

At bridge-weight confidence level $\delta_0$, the finite-sample bounds below
require
\begin{equation}\label{eqn:small-relative-bridge-weight-error}
    \sup_{\bm\theta\in\Theta_B}
    \frac{\varepsilon_{n,m}^B(\bm\theta;\delta_0)}
         {c_B(\bm\theta)}
    \leq\frac12.
\end{equation}
The results use $\delta_0=\delta/2$. If
$\sup_{\bm\theta}\varepsilon_{n,m}^B(\bm\theta;\delta_0)/c_B(\bm\theta)\to0$
as $n,m\to\infty$ for each fixed $\delta_0$, then for every fixed
$\delta\in(0,1)$ condition~\eqref{eqn:small-relative-bridge-weight-error}
holds with $\delta_0=\delta/2$ once $n$ and $m$ are sufficiently large.

The full-$P_S^X$-mass condition is needed because the empirical risks
\eqref{eqn:oracle-empirical-bridge-risk} and
\eqref{eqn:plugin-empirical-bridge-risk} use all source observations.
It ensures that every source covariate entering $P_n$ lies in
$\mathcal X_B$ almost surely, so that the lower and upper bounds in
\eqref{eqn:uniform-weight-bounds} apply to every summand.

If $0<r_-\leq\LR(x)\leq r_+<\infty$ for every $x\in\mathcal X_B$ and the
bridge maps are positive and jointly continuous on the corresponding compact
parameter--likelihood-ratio set, then
\eqref{eqn:uniform-weight-bounds} follows from compactness. For the untrimmed
bridge maps, this bounded-likelihood-ratio verification is the regime used in
the present finite-sample analysis; the resulting bounds do not cover an
unbounded likelihood ratio.

The final subsection gives explicit high-probability choices of
$\varepsilon_{n,m}^B(\bm\theta;\delta_0)$ under non-parametric density estimation.
For that calculation to verify Assumption~\ref{ass:bridge-weight-regularity}
for the untrimmed empirical risks, the compact density-estimation working region must carry
full $P_S^X$-mass and the deterministic weight bounds above must hold there.
If the density or overlap conditions hold only on a proper subregion, the
resulting control is local and does not verify the assumption for the original
untrimmed risk. A region-restricted or clipped construction defines a modified
estimator and requires a separate analysis.

\subsection{A uniform empirical-process bound}

\begin{theorem}[Uniform structural-bridge empirical-process bound]
\label{thm:uniform-structural-bridge-bound}
\leavevmode\par\noindent
Under Assumptions~\ref{ass:bounded-loss} and
\ref{ass:bridge-weight-regularity},
fix $\delta\in(0,1)$. Suppose $n$ and $m$ are such that
\eqref{eqn:small-relative-bridge-weight-error} holds with $\delta_0=\delta/2$.
Then, with probability at least $1-\delta$,
\begin{align}
&\sup_{\bm\theta\in\Theta_B}
 \sup_{h\in\hs}
 \left|
   \widehat R_{\bm\theta}^B(h)
   -R_{\bm\theta}^B(h)
 \right|
\notag\\
&\quad\leq
\frac{2}{c_B}
\left[
  \mathfrak R_n(\mathcal W_B\mathcal L_{\hs})
  +M\mathfrak R_n(\mathcal W_B)
\right]
+\frac{MC_B}{c_B}\sqrt{\frac{\log(2/\delta)}{2n}}
+
\frac{2M}{c_B}\varepsilon_{n,m}^B(\delta/2).
\label{eqn:uniform-structural-bridge-bound}
\end{align}
For later use we denote the right-hand side of \eqref{eqn:uniform-structural-bridge-bound} by $\overline{\mathfrak C}_{n,m}^B(\delta)$.
\end{theorem}

\paragraph{Idea of the proof.}
The proof uses two self-normalized ratio identities. On the uniform
bridge-weight event, the deterministic lower weight bound and the
small-relative-error condition keep the estimated normalizers uniformly away
from zero, yielding a plug-in-to-oracle perturbation bound. A second identity reduces the oracle-to-population error to a centered class. Its Rademacher complexity is bounded in terms of
$\mathcal W_B\mathcal L_{\hs}$ and $\mathcal W_B$, and one bounded-difference estimate controls its fluctuation.

\begin{proof}
\paragraph{Step 1: plug-in bridge-weight perturbation.}
Because $P_S^X(\mathcal X_B)=1$, every source covariate belongs to
$\mathcal X_B$ almost surely. Hence, on the event in
\eqref{eqn:high-probability-weight-error} with $\delta_0=\delta/2$,
\eqref{eqn:uniform-weight-bounds} and
\eqref{eqn:small-relative-bridge-weight-error} give
\[
    P_n\widehat w_{\bm\theta}^B
    \geq
    P_n w_{\bm\theta}^B
    -
    \|\widehat w_{\bm\theta}^B-w_{\bm\theta}^B
      \|_{\infty,\mathcal X_B}
    \geq
    c_B(\bm\theta)
    -\varepsilon_{n,m}^B(\bm\theta;\delta/2)
    \geq
    \frac{c_B(\bm\theta)}{2}
    \geq
    \frac{c_B}{2}.
\]
Thus the estimated normalizer is positive on this event. For every $h$ and
$\bm\theta$, the exact identity
\begin{equation*}
\widehat R_{\bm\theta}^B(h)-R_{n,\bm\theta}^B(h)
=
\frac{
  P_n\left[
    (\widehat w_{\bm\theta}^B-w_{\bm\theta}^B)
    \{L_h-R_{n,\bm\theta}^B(h)\}
  \right]
}{P_n\widehat w_{\bm\theta}^B},
\end{equation*}
holds because
\[
    P_n\left[
      w_{\bm\theta}^B
      \{L_h-R_{n,\bm\theta}^B(h)\}
    \right]=0.
\]
Since $0\leq L_h\leq M$,
$0\leq R_{n,\bm\theta}^B(h)\leq M$, and therefore
$|L_h-R_{n,\bm\theta}^B(h)|\leq M$. Hence, uniformly over $h$ and
$\bm\theta$ on the same event,
\begin{equation}\label{eqn:plugin-risk-perturbation-bound}
    \left|
      \widehat R_{\bm\theta}^B(h)
      -R_{n,\bm\theta}^B(h)
    \right|
    \leq
    \frac{
      M\varepsilon_{n,m}^B(\bm\theta;\delta/2)
    }{
      P_nw_{\bm\theta}^B
      -\varepsilon_{n,m}^B(\bm\theta;\delta/2)
    }
    \leq
    \frac{2M}{c_B(\bm\theta)}
    \varepsilon_{n,m}^B(\bm\theta;\delta/2)
    \leq
    \frac{2M}{c_B}\varepsilon_{n,m}^B(\delta/2).
\end{equation}
This bound is pathwise on the bridge-weight event and does not require the
estimated source weights to be independent of the source empirical risk.

\paragraph{Step 2: oracle self-normalized deviation.}
The oracle risk is a ratio, so its sampling error contains both a
weighted-loss numerator process and a weight-normalization process.
Cross-multiplying the empirical and population ratios gives the exact identity
\begin{align}
R_{n,\bm\theta}^B(h)-R_{\bm\theta}^B(h)
={}&
\frac{
 (P_n-P_S)(w_{\bm\theta}^B L_h)
 -R_{\bm\theta}^B(h)(P_n-P_S)w_{\bm\theta}^B
}{P_nw_{\bm\theta}^B}.
\label{eqn:oracle-risk-exact-identity}
\end{align}
Since $P_S^X(\mathcal X_B)=1$, every source covariate entering $P_n$
belongs to $\mathcal X_B$ almost surely. Hence
\eqref{eqn:uniform-weight-bounds} gives
$P_n w_{\bm\theta}^B\geq c_B$, and boundedness of the loss gives
$0\leq R_{\bm\theta}^B(h)\leq M$. Thus
\begin{align}
&\sup_{\bm\theta\in\Theta_B}
 \sup_{h\in\hs}
 \left|
   R_{n,\bm\theta}^B(h)-R_{\bm\theta}^B(h)
 \right|
\notag\\
&\quad\leq
\frac1{c_B}
\left[
  \sup_{f\in\mathcal W_B\mathcal L_{\hs}}
  |(P_n-P_S)f|
  +
  M\sup_{w\in\mathcal W_B}|(P_n-P_S)w|
\right].
\label{eqn:oracle-risk-process-decomposition}
\end{align}
Thus the first process is indexed by
$\mathcal W_B\mathcal L_{\hs}$, while the correction caused by replacing
$P_Sw_{\bm\theta}^B$ with $P_nw_{\bm\theta}^B$ is indexed by $\mathcal W_B$.
The common lower weight bound gives
$P_nw_{\bm\theta}^B\geq c_B$ simultaneously over $\bm\theta$ almost surely,
so no separate denominator-tail or ratio-process inequality is needed.

\paragraph{Step 3: symmetrization and concentration.}
For a nonnegative class $\mathcal F$ with envelope $F$, put, for
$z_1,\ldots,z_n\in\fs$,
\[
    T_{\mathcal F}(z_1,\ldots,z_n)
    =
    \sup_{f\in\mathcal F}
    \left|
      \frac1n\sum_{i=1}^nf(z_i)-P_Sf
    \right|.
\]
Standard symmetrization gives
$\E[T_{\mathcal F}(Z_1^S,\ldots,Z_n^S)]
\leq2\mathfrak R_n(\mathcal F)$. If one observation is replaced, then
\[
\left|T_{\mathcal F}(z_1,\ldots,z_i,\ldots,z_n)
-T_{\mathcal F}(z_1,\ldots,z_i',\ldots,z_n)\right|
\leq
\frac1n\sup_{f\in\mathcal F}|f(z_i)-f(z_i')|
\leq\frac Fn,
\]
where the last inequality uses $0\leq f\leq F$. McDiarmid's inequality,
combined with the symmetrization bound, therefore gives
\begin{equation}\label{eqn:standard-rademacher-deviation}
    \sup_{f\in\mathcal F}|(P_n-P_S)f|
    \leq
    2\mathfrak R_n(\mathcal F)
    +F\sqrt{\frac{\log(1/\delta_0)}{2n}}
\end{equation}
with probability at least $1-\delta_0$; see also
\citet{Bartlett2002} and \citet{Gine2015}. The class
$\mathcal W_B\mathcal L_{\hs}$ has envelope $MC_B$, while
$\mathcal W_B$ has envelope $C_B$. Applying
\eqref{eqn:standard-rademacher-deviation} to the two classes with
$\delta_0=\delta/4$ gives, simultaneously with probability at least
$1-\delta/2$,
\begin{align*}
\sup_{f\in\mathcal W_B\mathcal L_{\hs}}|(P_n-P_S)f|
&\leq
2\mathfrak R_n(\mathcal W_B\mathcal L_{\hs})
+MC_B\sqrt{\frac{\log(4/\delta)}{2n}},
\\
\sup_{w\in\mathcal W_B}|(P_n-P_S)w|
&\leq
2\mathfrak R_n(\mathcal W_B)
+C_B\sqrt{\frac{\log(4/\delta)}{2n}}.
\end{align*}
Substitution into~\eqref{eqn:oracle-risk-process-decomposition} yields
\begin{align}
&\sup_{\bm\theta\in\Theta_B}
 \sup_{h\in\hs}
 \left|
   R_{n,\bm\theta}^B(h)-R_{\bm\theta}^B(h)
 \right|
\notag\\
&\quad\leq
\frac{2}{c_B}
\left[
  \mathfrak R_n(\mathcal W_B\mathcal L_{\hs})
  +M\mathfrak R_n(\mathcal W_B)
  +MC_B\sqrt{\frac{\log(4/\delta)}{2n}}
\right].
\label{eqn:oracle-risk-process-bound}
\end{align}
The two separate concentration estimates above give~\eqref{eqn:oracle-risk-process-bound}. For the sharper bound in the theorem, write
\[
 g_{h,\bm\theta}=w_{\bm\theta}^B\{L_h-R_{\bm\theta}^B(h)\},
 \qquad \mathcal G=\{g_{h,\bm\theta}:h\in\hs,\ \bm\theta\in\Theta_B\}.
\]
The numerator in~\eqref{eqn:oracle-risk-exact-identity} is $(P_n-P_S)g_{h,\bm\theta}$. For fixed $(h,\bm\theta)$, the range of $g_{h,\bm\theta}$ is contained in
$[-C_B R_{\bm\theta}^B(h),C_B\{M-R_{\bm\theta}^B(h)\}]$, whose length is $MC_B$. Replacing one observation in $\sup_{g\in\mathcal G}|(P_n-P_S)g|$ therefore changes it by at most $MC_B/n$. Symmetrization and the bounded-difference inequality at level $\delta/2$ give the upper bound
\[
 2\mathfrak R_n(\mathcal G)+MC_B\sqrt{\frac{\log(2/\delta)}{2n}}.
\]
Because $0\le R_{\bm\theta}^B(h)\le M$, the triangle inequality in the Rademacher sum gives
$\mathfrak R_n(\mathcal G)\le\mathfrak R_n(\mathcal W_B\mathcal L_{\hs})+M\mathfrak R_n(\mathcal W_B)$.
Dividing by $P_n w_{\bm\theta}^B\ge c_B$ and intersecting this event with the plug-in event at level $\delta/2$ proves~\eqref{eqn:uniform-structural-bridge-bound}.

\end{proof}

The terms involving $\mathfrak R_n(\mathcal W_B\mathcal L_{\hs})$ and
$\mathfrak R_n(\mathcal W_B)$ control the oracle numerator and denominator
processes, respectively, while the square-root term is their concentration
contribution. The final additive term controls replacement of the population
bridge weights by their estimates; under the density plug-in construction,
this error is induced by likelihood-ratio estimation. Since the event holds
simultaneously over $\bm\theta$ and $h$, it remains valid after substituting
any measurable, data-dependent pair
$(\widehat{\bm\theta},\widehat h)\in\Theta_B\times\hs$, including choices based
on the labeled source sample and unlabeled target covariates. No additional
rate for $\widehat{\bm\theta}$ is needed when the empirical and population
risks are evaluated at the same selected parameter. The oracle
empirical-process error is controlled uniformly over $\bm\theta$ through the
two Rademacher complexities and the global weight bounds $c_B$ and $C_B$,
while uniformity of the bridge-weight approximation is contained in
$\varepsilon_{n,m}^B(\delta/2)$. Comparing
$\widehat{\bm\theta}$ with a fixed population parameter uses a separate
estimation rate, as in the later Nadaraya--Watson oracle-equivalence analysis.

\subsection{Empirical target-risk generalization}

Let $(\widehat{\bm\theta},\widehat h)$ be a measurable, possibly
data-dependent pair taking values in $\Theta_B\times\hs$ and satisfying,
almost surely,
\begin{equation*}
    \widehat R_{\widehat{\bm\theta}}^B(\widehat h)
    \leq
    \inf_{h\in\hs}
    \widehat R_{\widehat{\bm\theta}}^B(h)
    +\xi_n,
\end{equation*}
where $\xi_n\geq0$ is a deterministic optimization tolerance.

\begin{theorem}[Empirical structural-bridge generalization]
\label{thm:empirical-structural-bridge-generalization}
Under Assumption~\ref{ass:common-domination} and the conditions of
Theorem~\ref{thm:uniform-structural-bridge-bound},
with probability at least $1-\delta$,
\begin{equation*}
    R_T(\widehat h)
    \leq
    \inf_{h\in\hs}R_T(h)
    +2\Delta_B(\widehat{\bm\theta})
    +2\overline{\mathfrak C}_{n,m}^B(\delta)
    +\xi_n.
\end{equation*}
\end{theorem}

\paragraph{Idea of the proof.}
On the uniform event of
Theorem~\ref{thm:uniform-structural-bridge-bound}, approximate empirical
optimality transfers to approximate population optimality for the structural
bridge risk, at cost $2\overline{\mathfrak C}_{n,m}^B(\delta)$. Target
locality then transfers this comparison to the target risk, at cost
$2\Delta_B(\widehat{\bm\theta})$.

\begin{proof}
Define
\[
\mathcal E_{n,m}^B(\delta)
=
\left\{
\sup_{\bm\theta\in\Theta_B}\sup_{h\in\hs}
\left|
\widehat R_{\bm\theta}^B(h)-R_{\bm\theta}^B(h)
\right|
\leq
\overline{\mathfrak C}_{n,m}^B(\delta)
\right\}.
\]
Theorem~\ref{thm:uniform-structural-bridge-bound} gives
$\Pr\{\mathcal E_{n,m}^B(\delta)\}\geq1-\delta$. Work on this event and write
$\overline{\mathfrak C}=\overline{\mathfrak C}_{n,m}^B(\delta)$ for brevity.
Then
\begin{align*}
R_{\widehat{\bm\theta}}^B(\widehat h)
&\leq
\widehat R_{\widehat{\bm\theta}}^B(\widehat h)
+\overline{\mathfrak C}\\
&\leq
\inf_{g\in\hs}\widehat R_{\widehat{\bm\theta}}^B(g)
+\xi_n+\overline{\mathfrak C}\\
&\leq
\inf_{g\in\hs}R_{\widehat{\bm\theta}}^B(g)
+\xi_n+2\overline{\mathfrak C}.
\end{align*}
Thus approximate empirical optimality has been transferred to the population
structural bridge risk. The definition of
$\Delta_B(\widehat{\bm\theta})$ now gives
\begin{align*}
R_T(\widehat h)
&\leq
R_{\widehat{\bm\theta}}^B(\widehat h)
+\Delta_B(\widehat{\bm\theta})\\
&\leq
\inf_{g\in\hs}R_{\widehat{\bm\theta}}^B(g)
+\xi_n+2\overline{\mathfrak C}
+\Delta_B(\widehat{\bm\theta})\\
&\leq
\inf_{g\in\hs}R_T(g)
+\xi_n+2\overline{\mathfrak C}
+2\Delta_B(\widehat{\bm\theta}).
\end{align*}
This proves the result. Neither infimum is required to be attained.
\end{proof}

On the same event, Proposition~\ref{prop:structural-bridge-target-locality}
gives the following geometry-specific bounds. For the KL/KL bridge
$B=\mathrm{KL}$, whose target-side divergence is KL,
\begin{align}
R_T(\widehat h)
\leq{}&
\inf_{h\in\hs}R_T(h)
+2M\left\{
  \sqrt{\frac{\delta_{B,T}(\widehat{\bm\theta})}{2}}
  +\sqrt{\frac{\delta_T^\circ}{2}}
\right\}
\notag\\
&+2\overline{\mathfrak C}_{n,m}^B(\delta)+\xi_n.
\label{eqn:empirical-target-risk-kl}
\end{align}
For $B\in\{\mathrm{HD},\mathrm{KH}\}$, whose target-side divergence is
squared Hellinger distance,
\begin{align}
R_T(\widehat h)
\leq{}&
\inf_{h\in\hs}R_T(h)
+2M\left\{
  \sqrt{\delta_{B,T}(\widehat{\bm\theta})}
  +\sqrt{\delta_T^\circ}
\right\}
\notag\\
&+2\overline{\mathfrak C}_{n,m}^B(\delta)+\xi_n.
\label{eqn:empirical-target-risk-hd}
\end{align}
If the fixed population parameter set $\Theta_B$ satisfies
$\sup_{\bm\theta\in\Theta_B}\delta_{B,T}(\bm\theta)\leq\rho_T$, then
$\delta_{B,T}(\widehat{\bm\theta})$ in the preceding displays may be replaced
by $\rho_T$. If Assumption~\ref{ass:target-feasibility} also holds, then
$\delta_T^\circ\leq\rho_T$, and the entire locality contribution is bounded
by $2M\sqrt{2\rho_T}$ for a KL target-side divergence and by
$4M\sqrt{\rho_T}$ for a squared-Hellinger target-side divergence.

\subsection{VC-type and effective-sample-size bounds}

Theorem~\ref{thm:uniform-structural-bridge-bound} gives a uniform bound
through expected Rademacher complexities and global weight envelopes. We now
derive a variance-sensitive refinement. A uniform empirical-Bernstein
inequality converts the observed second moment of each population
bridge-weight vector into its effective sample size, while finite-sample
covering numbers make the event simultaneous in $h$ and $\bm\theta$. The
bridge-weight plug-in error remains a separate additive term.

Recall that a function class
$\mathcal F$ with envelope $F$ is VC-type with characteristics $(A,v)$ if
\[
    \sup_Q
    N\left(
      \epsilon\|F\|_{Q,2},
      \mathcal F,
      L_2(Q)
    \right)
    \leq
    \left(\frac{A}{\epsilon}\right)^v,
    \qquad 0<\epsilon\leq1,
\]
where the supremum is over finitely supported probability measures $Q$.
Here $N(\epsilon,\mathcal F,L_2(Q))$ denotes the minimal number of
$L_2(Q)$-balls of radius $\epsilon$ required to cover $\mathcal F$; see, for
example, \citet{Gine2015}.

Because $P_S^X(\mathcal X_B)=1$, all function classes and covering numbers in
this subsection are understood on the full-$P_S$-measure set
$\{(x,y):x\in\mathcal X_B\}$. Equivalently, the functions may be redefined on
its null complement.
Because the scale of each bridge weight is immaterial, put
\[
    u_{\bm\theta}^B
    =
    \frac{w_{\bm\theta}^B}{C_B(\bm\theta)},
    \qquad
    \ell_h=\frac{L_h}{M}.
\]
This auxiliary normalization places the relevant functions in $[0,1]$ and
does not change $R_{n,\bm\theta}^B$, $R_{\bm\theta}^B$, or the effective
sample size. The plug-in perturbation remains on the fixed bridge-weight
scale used in Assumption~\ref{ass:bridge-weight-regularity}.

Define the two $[0,1]$-valued classes
\begin{equation}\label{eqn:normalized-bridge-classes}
    \mathcal F_{0,B}
    =
    \{u_{\bm\theta}^B:\bm\theta\in\Theta_B\},
    \qquad
    \mathcal F_{1,B}
    =
    \{u_{\bm\theta}^B\ell_h:
      \bm\theta\in\Theta_B,\ h\in\hs\}.
\end{equation}
For a bounded class $\mathcal F$, let
$\mathcal N_\infty(\epsilon,\mathcal F,N)$ be the supremum, over all
restrictions of $\mathcal F$ to $N$ points, of the corresponding
$L_\infty$ covering number at radius $\epsilon$. Set
\begin{equation}\label{eqn:bridge-growth-and-kappa}
    \mathcal N_{B,n}
    =
    \max_{j\in\{0,1\}}
    \mathcal N_\infty(1/n,\mathcal F_{j,B},2n),
    \qquad
    \kappa_{B,n}(\delta)
    =
    \log\left\{\frac{80\mathcal N_{B,n}}{\delta}\right\}.
\end{equation}
The radius $1/n$ on $2n$ points is the finite-cover input used by the
uniform empirical-Bernstein inequality below.

For $\bm\theta\in\Theta_B$, define the empirical effective sample size (ESS)
\begin{equation}\label{eqn:simultaneous-effective-sample-size}
    n_{\mathrm{eff},n}^{B}(\bm\theta)
    =
    \frac{
      \left\{\sum_{i=1}^{n}w_{\bm\theta}^{B}(Z_i^S)\right\}^{2}
    }{
      \sum_{i=1}^{n}\{w_{\bm\theta}^{B}(Z_i^S)\}^{2}
    }.
\end{equation}
If
\[
    \pi_{i,\bm\theta}^B
    =
    \frac{w_{\bm\theta}^B(Z_i^S)}
         {\sum_{j=1}^n w_{\bm\theta}^B(Z_j^S)},
\]
then
$n_{\mathrm{eff},n}^B(\bm\theta)
=1/\sum_{i=1}^n(\pi_{i,\bm\theta}^B)^2$, so
$1\leq n_{\mathrm{eff},n}^B(\bm\theta)\leq n$, with the upper endpoint
attained when the bridge weights are equal. Thus,
$n/n_{\mathrm{eff},n}^B(\bm\theta)
=n\sum_{i=1}^n(\pi_{i,\bm\theta}^B)^2$ is the self-normalized empirical
version of the order-two likelihood-ratio moment used in generalization bounds
for importance weighting; see \citet{Cortes2010}. The ESS is unchanged by
the normalization in~\eqref{eqn:normalized-bridge-classes}. Because it is
formed from the population bridge weights $w_{\bm\theta}^B$, it is an oracle
empirical ESS when the likelihood ratio is estimated. We do not replace it
here by an ESS computed from $\widehat w_{\bm\theta}^B$; such a replacement
would require a separate perturbation comparison.

Suppose that $\mathcal L_{\hs}$ is VC-type with envelope $M$ and entropy
exponent $V$, and that $\mathcal F_{0,B}$ is VC-type with envelope $1$ and
entropy exponent $v_B$. After adjusting the entropy characteristic, the
product class $\mathcal F_{1,B}$ is then VC-type with entropy exponent
$V+v_B$; see, for example,
\citet{Gine2015}. For the bridge maps in
Section~\ref{sec:loss-agnostic-density-bridges}, the compact
parameter--likelihood-ratio ranges, the envelope choices in
Lemma~\ref{lem:explicit-bridge-constants}, and the joint regularity in
Lemma~\ref{lem:bridge-maps-lipschitz} show that
$\alpha\mapsto u_\alpha^B$ is uniformly Lipschitz in the supremum norm. A
grid in $\alpha$ therefore shows that one may take $v_B=1$, after adjusting
the entropy characteristic. For KL/HD, this verification uses
$\alpha\leq\alpha_{\mathrm{KH},\max}<1$.

\begin{proposition}[Simultaneous ESS bound]
\label{prop:vc-structural-bridge-complexity}
Fix $\delta\in(0,1)$. Suppose Assumptions~\ref{ass:bounded-loss} and
\ref{ass:bridge-weight-regularity} hold, $n\geq16$, the normalized classes
in~\eqref{eqn:normalized-bridge-classes} are pointwise measurable and satisfy
$\mathcal N_{B,n}<\infty$, and $n,m$ are such that
\eqref{eqn:small-relative-bridge-weight-error} holds with $\delta_0=\delta/2$.
Then, with probability at least $1-\delta$, the following holds
simultaneously for every
$\bm\theta\in\Theta_B$:
\begin{align}
&\sup_{h\in\hs}
 \left|
   \widehat R_{\bm\theta}^B(h)-R_{\bm\theta}^B(h)
 \right|
\notag\\
&\quad\leq
2M\sqrt{
  \frac{18n\kappa_{B,n}(\delta)}
       {(n-1)n_{\mathrm{eff},n}^{B}(\bm\theta)}
}
+
\frac{30M C_B(\bm\theta)}{P_nw_{\bm\theta}^B}
\frac{\kappa_{B,n}(\delta)}{n-1}
\notag\\
&\qquad\quad+
\frac{
  M\varepsilon_{n,m}^B(\bm\theta;\delta/2)
}{
  P_nw_{\bm\theta}^B
  -\varepsilon_{n,m}^B(\bm\theta;\delta/2)
}.
\label{eqn:simultaneous-effective-sample-size-bound}
\end{align}
Denote the right-hand side by
$\mathfrak C_{n,m}^B(\bm\theta;\delta)$. Define the weight-spread factor and
the corresponding deterministic rate term by
\begin{equation*}
    \mathfrak s_B(\bm\theta)
    =
    \frac{C_B(\bm\theta)}{c_B(\bm\theta)},
    \qquad
    \lambda_{B,n}(\delta)
    =
    2\sqrt{\frac{18\kappa_{B,n}(\delta)}{n-1}}
    +\frac{30\kappa_{B,n}(\delta)}{n-1},
\end{equation*}
respectively. Then
\begin{equation}\label{eqn:deterministic-structural-bridge-complexity}
    \mathfrak C_{n,m}^B(\bm\theta;\delta)
    \leq
    M \mathfrak s_B(\bm\theta)\lambda_{B,n}(\delta)
    +
    \frac{2M}{c_B(\bm\theta)}
    \varepsilon_{n,m}^B(\bm\theta;\delta/2).
\end{equation}

Under the preceding VC-type conditions, there are constants $A,C<\infty$,
determined only by the entropy characteristics, such that, with
\[
    K_{B,n}(\delta)
    =
    (V+v_B)\log(An)+\log(80/\delta),
\]
\begin{align}
\mathfrak C_{n,m}^B(\bm\theta;\delta)
\leq{}&
CM\left[
  \sqrt{
    \frac{K_{B,n}(\delta)}
         {n_{\mathrm{eff},n}^{B}(\bm\theta)}
  }
  +
  \frac{C_B(\bm\theta)}{P_nw_{\bm\theta}^B}
  \frac{K_{B,n}(\delta)}{n}
\right]
\notag\\
&+
\frac{2M}{c_B(\bm\theta)}
\varepsilon_{n,m}^B(\bm\theta;\delta/2)
\label{eqn:vc-structural-bridge-complexity}
\end{align}
simultaneously for every $\bm\theta\in\Theta_B$.
The exact bound, its deterministic consequence, and the VC-type specialization
remain valid on their common event after substituting any measurable,
data-dependent $\widehat{\bm\theta}\in\Theta_B$.
\end{proposition}

\paragraph{Idea of the proof.}
The oracle risk is treated through the numerator--denominator identity used
in Theorem~\ref{thm:uniform-structural-bridge-bound}, but the two empirical
processes are now controlled by a uniform empirical-Bernstein inequality
rather than by expected Rademacher complexities. After normalizing the
bridge weights and losses to take values in $[0,1]$, both empirical variances
are bounded by $nP_n[(u_{\bm\theta}^B)^2]/(n-1)$. Division by
$P_nu_{\bm\theta}^B$ then produces the effective-sample-size term. The
resulting event is simultaneous in $h$ and $\bm\theta$; intersecting it with
the uniform bridge-weight estimation event adds the plug-in contribution
separately.

\begin{proof}
For a class $\mathcal F$, write
$1-\mathcal F=\{1-f:f\in\mathcal F\}$. Put
\[
    \mathcal G_B
    =
    \mathcal F_{0,B}\cup\mathcal F_{1,B}
    \cup(1-\mathcal F_{0,B})\cup(1-\mathcal F_{1,B}).
\]
This class has envelope $1$ and
\[
    \mathcal N_\infty(1/n,\mathcal G_B,2n)
    \leq4\mathcal N_{B,n}.
\]
Apply the one-sided uniform empirical-Bernstein inequality of
\citet[Theorem~6]{Maurer2009} to $\mathcal G_B$ with confidence level
$\delta/2$. The complement classes provide the reverse deviations and have
the same sample variance. Thus, with probability at least $1-\delta/2$,
\begin{equation}\label{eqn:uniform-empirical-bernstein-input}
    |(P_n-P_S)f|
    \leq
    \sqrt{
      \frac{18\kappa_{B,n}(\delta)V_n(f)}{n}
    }
    +
    \frac{15\kappa_{B,n}(\delta)}{n-1}
\end{equation}
simultaneously for every
$f\in\mathcal F_{0,B}\cup\mathcal F_{1,B}$, where
\[
    V_n(f)
    =
    \frac{1}{n(n-1)}
    \sum_{1\leq i<j\leq n}
    \{f(Z_i^S)-f(Z_j^S)\}^2
    =
    \frac{n}{n-1}
    \{P_nf^2-(P_nf)^2\}
    \leq
    \frac{n}{n-1}P_nf^2.
\]

Write $r_{\bm\theta}^B(h)=R_{\bm\theta}^B(h)/M$. The exact ratio identity
\eqref{eqn:oracle-risk-exact-identity} becomes
\begin{align*}
R_{n,\bm\theta}^B(h)-R_{\bm\theta}^B(h)
={}&
\frac{M}{P_nu_{\bm\theta}^B}
\left[
 (P_n-P_S)(u_{\bm\theta}^B\ell_h)
 -r_{\bm\theta}^B(h)(P_n-P_S)u_{\bm\theta}^B
\right].
\end{align*}
Because $0\leq r_{\bm\theta}^B(h)\leq1$ and
\[
    P_n\{(u_{\bm\theta}^B\ell_h)^2\}
    \leq
    P_n[(u_{\bm\theta}^B)^2],
\]
both empirical variance terms are bounded by
$nP_n[(u_{\bm\theta}^B)^2]/(n-1)$. Thus the two square-root contributions have
the same upper bound, while the two linear contributions sum to
$30\kappa_{B,n}(\delta)/(n-1)$. Consequently,
inequality~\eqref{eqn:uniform-empirical-bernstein-input} gives, simultaneously
over $h$ and $\bm\theta$,
\begin{align}
\left|R_{n,\bm\theta}^B(h)-R_{\bm\theta}^B(h)\right|
\leq{}&
2M\sqrt{
  \frac{18n\kappa_{B,n}(\delta)}
       {(n-1)n_{\mathrm{eff},n}^{B}(\bm\theta)}
}
\notag\\
&+
\frac{30M}{P_nu_{\bm\theta}^B}
\frac{\kappa_{B,n}(\delta)}{n-1}.
\label{eqn:oracle-simultaneous-effective-sample-size-bound}
\end{align}
Here we used
\[
    n_{\mathrm{eff},n}^{B}(\bm\theta)
    =
    \frac{n(P_nu_{\bm\theta}^B)^2}
         {P_n[(u_{\bm\theta}^B)^2]},
    \qquad
    \frac{1}{P_nu_{\bm\theta}^B}
    =
    \frac{C_B(\bm\theta)}{P_nw_{\bm\theta}^B}.
\]

On the event in~\eqref{eqn:high-probability-weight-error} at level
$\delta/2$, the lower weight bound and
\eqref{eqn:small-relative-bridge-weight-error} give
\[
P_nw_{\bm\theta}^B
-\varepsilon_{n,m}^B(\bm\theta;\delta/2)
\geq \frac{c_B(\bm\theta)}2>0.
\]
The plug-in identity used in
\eqref{eqn:plugin-risk-perturbation-bound} gives, simultaneously over $h$ and
$\bm\theta$,
\[
    \left|
      \widehat R_{\bm\theta}^B(h)-R_{n,\bm\theta}^B(h)
    \right|
    \leq
    \frac{
      M\varepsilon_{n,m}^B(\bm\theta;\delta/2)
    }{
      P_nw_{\bm\theta}^B
      -\varepsilon_{n,m}^B(\bm\theta;\delta/2)
    }
    \leq
    \frac{2M}{c_B(\bm\theta)}
    \varepsilon_{n,m}^B(\bm\theta;\delta/2).
\]
By the union bound, the empirical-Bernstein and bridge-weight events hold
simultaneously with probability at least $1-\delta$. This proves
\eqref{eqn:simultaneous-effective-sample-size-bound}.

Since $1/\mathfrak s_B(\bm\theta)\leq u_{\bm\theta}^B\leq1$, the weight bounds give
\begin{equation}\label{eqn:neff-lower-bound}
    P_nu_{\bm\theta}^B
    \geq
    \frac{1}{\mathfrak s_B(\bm\theta)},
    \qquad
    n_{\mathrm{eff},n}^{B}(\bm\theta)
    \geq
    \frac{n}{\mathfrak s_B(\bm\theta)^2}.
\end{equation}
These inequalities prove
\eqref{eqn:deterministic-structural-bridge-complexity}.

Finally, if $Q$ is uniform on any $2n$ points, then
$\|f-g\|_{\infty,\operatorname{supp}(Q)}
\leq\sqrt{2n}\|f-g\|_{Q,2}$. Hence an $L_2(Q)$-cover at
radius $1/(n\sqrt{2n})$ supplies the $L_\infty$-cover used in
\eqref{eqn:bridge-growth-and-kappa}. The preceding VC-type assumptions and
the standard product-cover bound give, after adjusting a fixed entropy
characteristic $A_0$,
\[
    \log\mathcal N_{B,n}
    \leq
    (V+v_B)
    \left\{\frac32\log n+\log A_0\right\}.
\]
Therefore
$\kappa_{B,n}(\delta)\leq C K_{B,n}(\delta)$ after absorbing the factor
$3/2$ into $C$ and the fixed numerical factors into $A$.
Substitution in~\eqref{eqn:simultaneous-effective-sample-size-bound} proves
\eqref{eqn:vc-structural-bridge-complexity}.
\end{proof}

The effective sample size governs the leading square-root term, but it does
not by itself control the empirical mean $P_nu_{\bm\theta}^B$ relative to the
weight envelope. The lower-order Bernstein term in
\eqref{eqn:simultaneous-effective-sample-size-bound} must therefore be retained
unless additional leverage control is imposed. Indeed, apart from the factor
$n/(n-1)$, this term is proportional to
$C_B(\bm\theta)/\sum_iw_{\bm\theta}^B(Z_i^S)$, which bounds the largest
normalized sample weight and is an envelope-based leverage correction. The
density-estimation error remains a separate additive perturbation.

The probability event in
Proposition~\ref{prop:vc-structural-bridge-complexity} is simultaneous in
$h$ and $\bm\theta$. Consequently, it remains valid after substituting any
measurable, data-dependent $\widehat{\bm\theta}$ in the fixed deterministic
set $\Theta_B$, including one selected using the source labels. The
finite-cover quantity $\mathcal N_{B,n}$ is the empirical-process complexity
price for uniformity over the bridge family; under the VC-type specialization,
this includes the contribution involving $v_B$. Uniformity of the estimated
weights is carried separately by $\varepsilon_{n,m}^B$, yielding a common
post-selection event at the selected parameter. For the bridge maps in
Section~\ref{sec:loss-agnostic-density-bridges}, joint continuity on the
compact parameter--likelihood-ratio region gives the measurability needed for
these evaluations.

\begin{corollary}[ESS-based target-risk generalization]
\label{cor:parameter-indexed-structural-bridge-generalization}
Assume Assumption~\ref{ass:common-domination} and the conditions of
Proposition~\ref{prop:vc-structural-bridge-complexity}. Let
$\widehat{\bm\theta}$ and $\widehat h$ be as in
Theorem~\ref{thm:empirical-structural-bridge-generalization}. Then, with
probability at least $1-\delta$,
\begin{equation*}
    R_T(\widehat h)
    \leq
    \inf_{h\in\hs}R_T(h)
    +2\Delta_B(\widehat{\bm\theta})
    +2\mathfrak C_{n,m}^B(\widehat{\bm\theta};\delta)
    +\xi_n.
\end{equation*}
\end{corollary}

\begin{proof}
Work on the simultaneous event in
Proposition~\ref{prop:vc-structural-bridge-complexity} and put
$\mathfrak C=\mathfrak C_{n,m}^B(\widehat{\bm\theta};\delta)$ for brevity.
Then
\begin{align*}
R_{\widehat{\bm\theta}}^B(\widehat h)
&\leq
\widehat R_{\widehat{\bm\theta}}^B(\widehat h)+\mathfrak C\\
&\leq
\inf_{g\in\hs}\widehat R_{\widehat{\bm\theta}}^B(g)
+\xi_n+\mathfrak C\\
&\leq
\inf_{g\in\hs}R_{\widehat{\bm\theta}}^B(g)
+\xi_n+2\mathfrak C.
\end{align*}
Proposition~\ref{prop:structural-bridge-target-locality} now gives
\begin{align*}
R_T(\widehat h)
&\leq
R_{\widehat{\bm\theta}}^B(\widehat h)
+\Delta_B(\widehat{\bm\theta})\\
&\leq
\inf_{g\in\hs}R_{\widehat{\bm\theta}}^B(g)
+\xi_n+2\mathfrak C
+\Delta_B(\widehat{\bm\theta})\\
&\leq
\inf_{g\in\hs}R_T(g)
+\xi_n+2\mathfrak C
+2\Delta_B(\widehat{\bm\theta}).
\end{align*}
This proves the result.
\end{proof}

\begin{remark}[Why the structural bridge is useful]
The loss-aware adversarial optimization depends on $h$, so a direct
empirical-process analysis involves a class of joint loss-and-adversary
transformations. The structural bridge separates the low-dimensional source--target density
geometry from the hypothesis class, producing a hypothesis-independent bridge
class that can be analyzed separately and then combined with
$\mathcal L_{\hs}$ through the product-cover bound above.
\end{remark}

\subsection{Complexity under non-parametric density estimation}
\label{sec:nonparametric-density-estimation}

Finally, we make the bridge-weight term in
Theorem~\ref{thm:uniform-structural-bridge-bound} explicit.  We first state
the required density-estimation input without tying it to a particular
estimator, and then give a boundary-corrected construction for which the
input is nonvacuous.

Let $\mathcal X_0\subseteq\xspace$ be a compact covariate region. To verify
Assumption~\ref{ass:bridge-weight-regularity} for the untrimmed empirical risk,
assume
\begin{equation*}
    P_S^X(\mathcal X_0)=1,
\end{equation*}
and suppose that there are constants $c_S,c_T<\infty$ with $c_S>0$ such that
\begin{equation}\label{eqn:density-bounds-section4}
    p_S^X(x)\geq c_S>0,
    \qquad
    p_T^X(x)\leq c_T<\infty,
    \qquad x\in\mathcal X_0.
\end{equation}
The existence of these constants does not follow from compactness alone.  A
simple sufficient condition is that both densities are continuous on
$\mathcal X_0$ and $p_S^X$ is strictly positive there.  For the uniform
bridge-map calculation, also suppose that
\begin{equation}\label{eqn:strong-overlap-section4}
    0<r_-\leq\LR(x)\leq r_+<\infty,
    \qquad x\in\mathcal X_0.
\end{equation}
For the bridge maps considered here, one may then take
\[
    c_B(\bm\theta)
    =\inf_{u\in[r_-,r_+]}g_{\bm\theta}^B(u),
    \qquad
    C_B(\bm\theta)
    =\sup_{u\in[r_-,r_+]}g_{\bm\theta}^B(u).
\]
Positivity and joint continuity of the bridge maps on the compact
parameter--likelihood-ratio set make these valid finite bounds in
Assumption~\ref{ass:bridge-weight-regularity}; for KL/HD, as before, the
parameter set is bounded away from $\alpha=1$.

For a measurable function $f$ on $\mathcal X_0$, put
\[
    \|f\|_{\infty,\mathcal X_0}
    =
    \sup_{x\in\mathcal X_0}|f(x)|.
\]
Let $\widehat p_S^X$ and $\widehat p_T^X$ be nonnegative measurable density
estimators on $\mathcal X_0$, and define
\[
    \widehat\LR(x)
    =
    \begin{cases}
      \widehat p_T^X(x)/\widehat p_S^X(x),
        &\widehat p_S^X(x)>0,\\
      0,&\widehat p_S^X(x)=0.
    \end{cases}
\]
The second case is immaterial on the high-probability event used below.
Let $a_{S,n}(\delta_0)$ and $a_{T,m}(\delta_0)$ be deterministic quantities such
that
\begin{align}
\Pr\Bigl\{
 &\|\widehat p_S^X-p_S^X\|_{\infty,\mathcal X_0}
    \leq a_{S,n}(\delta_0),
 \quad
 \|\widehat p_T^X-p_T^X\|_{\infty,\mathcal X_0}
    \leq a_{T,m}(\delta_0)
\Bigr\}
\geq1-\delta_0,
\label{eqn:density-high-probability-event}
\end{align}
Write $\mathcal E_{\delta_0}$ for the event in
\eqref{eqn:density-high-probability-event}.
For the confidence level under consideration, assume
\[
    a_{S,n}(\delta_0)\leq c_S/2.
\]
This is a substantive finite-sample denominator-stability condition.  It is
eventually satisfied for each fixed $\delta_0$ whenever
$a_{S,n}(\delta_0)\to0$, but it is not automatic for fixed $n$.

\begin{proposition}[Bridge-weight error from density estimation]
\label{prop:bridge-weight-error-density-estimation}
On $\mathcal E_{\delta_0}$,
\begin{equation}\label{eqn:lr-high-probability-bound}
    \|\widehat\LR-\LR\|_{\infty,\mathcal X_0}
    \leq
    \frac{2}{c_S}a_{T,m}(\delta_0)
    +
    \frac{2c_T}{c_S^2}a_{S,n}(\delta_0).
\end{equation}
Put
\begin{equation*}
    e_{n,m}(\delta_0)
    =
    \frac{2}{c_S}a_{T,m}(\delta_0)
    +
    \frac{2c_T}{c_S^2}a_{S,n}(\delta_0).
\end{equation*}
Suppose $e_{n,m}(\delta_0)\leq r_-/2$ and put
\begin{equation}\label{eqn:buffered-lr-interval-section4}
    I_\star
    =
    \left[\frac{r_-}{2},r_++\frac{r_-}{2}\right].
\end{equation}
For each $\bm\theta\in\Theta_B$, let $L_{B,I_\star}(\bm\theta)$ be a valid
Lipschitz constant for $u\mapsto g_{\bm\theta}^B(u)$ on $I_\star$. Then,
simultaneously for every $\bm\theta\in\Theta_B$,
\begin{equation}\label{eqn:bridge-weight-high-probability-bound}
    \|\widehat w_{\bm\theta}^B-w_{\bm\theta}^B
    \|_{\infty,\mathcal X_0}
    \leq
    L_{B,I_\star}(\bm\theta)e_{n,m}(\delta_0).
\end{equation}
Consequently, when $\mathcal X_B=\mathcal X_0$, the choice
\[
    \varepsilon_{n,m}^B(\bm\theta;\delta_0)
    =
    L_{B,I_\star}(\bm\theta)e_{n,m}(\delta_0)
\]
verifies~\eqref{eqn:high-probability-weight-error}. If, in addition,
$\sup_{\bm\theta}L_{B,I_\star}(\bm\theta)e_{n,m}(\delta_0)/c_B(\bm\theta)\leq1/2$,
it
verifies~\eqref{eqn:small-relative-bridge-weight-error} at level $\delta_0$.
\end{proposition}

\begin{proof}
On $\mathcal E_{\delta_0}$,
\[
    \widehat p_S^X(x)
    \geq p_S^X(x)-a_{S,n}(\delta_0)
    \geq c_S/2
\]
uniformly over $x\in\mathcal X_0$. Thus the first case in the definition of
$\widehat\LR$ applies throughout $\mathcal X_0$. Adding and subtracting
$p_T^X/\widehat p_S^X$ gives the pointwise identity
\[
    \widehat\LR-\LR
    =
    \frac{\widehat p_T^X-p_T^X}{\widehat p_S^X}
    +
    \frac{p_T^X(p_S^X-\widehat p_S^X)}
         {\widehat p_S^X p_S^X}.
\]
The two terms on the right are bounded, respectively, by
\[
    \frac{2}{c_S}|\widehat p_T^X-p_T^X|
    \quad\text{and}\quad
    \frac{2c_T}{c_S^2}|\widehat p_S^X-p_S^X|.
\]
Taking the supremum over $\mathcal X_0$ proves
\eqref{eqn:lr-high-probability-bound}. If
$e_{n,m}(\delta_0)\leq r_-/2$, then
\eqref{eqn:strong-overlap-section4} and the same uniform bound imply that
both $\LR(x)$ and $\widehat\LR(x)$ belong to the deterministic interval
$I_\star$.  The mean-value inequality therefore gives
\[
    |g_{\bm\theta}^B\{\widehat\LR(x)\}
       -g_{\bm\theta}^B\{\LR(x)\}|
    \leq
    L_{B,I_\star}(\bm\theta)|\widehat\LR(x)-\LR(x)|
\]
for every $x$ and $\bm\theta$. Taking suprema proves
\eqref{eqn:bridge-weight-high-probability-bound}.  All conclusions use the
same density event, so no further union bound over $x$ or $\bm\theta$ is
needed.
\end{proof}

One concrete way to obtain the required full-support uniform density bounds
is a boundary-corrected projection estimator.  For this construction,
suppose that both marginal laws are supported on $\mathcal X_0$.  After a
fixed affine rescaling, take $\mathcal X_0=[0,1]^d$ and use the tensor-product
Cohen--Daubechies--Vial boundary scaling spaces.  Their projection kernels
play the role of support-adapted mollifiers: they smooth at resolution $h$
while retaining the boundary approximation order.  Since a raw orthogonal
projection estimator may be signed, take its positive part and renormalize it
to integrate to one.  This operation preserves its uniform rate; details are
given in
Corollary~\ref{cor:unif-consistency-bridge-weights-boundary-projection} in the
Appendix.  This construction avoids the boundary bias of an uncorrected
convolution KDE on a compact full-mass support; see
\citet{Cohen1993} and \citet[Sections~4.3.5--4.3.6 and~5.1]{Gine2015}.

For $N\geq1$, $0<h\leq1/2$, and $\delta_0\in(0,1)$, define the
boundary-projection rate function
\begin{equation*}
    \omega_{N,h}(\delta_0)
    =
    \sqrt{
      \frac{\log(1/h)+\log(4/\delta_0)}{Nh^d}
    }
    +
    \frac{\log(1/h)+\log(4/\delta_0)}{Nh^d}.
\end{equation*}
If the two marginal densities belong to an intrinsic
$\beta$-H\"older--Zygmund class on $\mathcal X_0$ and the boundary projection
basis has regularity and approximation order exceeding $\beta$, there are finite constants $A_S$ and
$A_T$, independent of $n,m,h_S,h_T$, and $\delta_0$, for which the preceding
joint density event holds with
\begin{align}
    a_{S,n}(\delta_0)
    &\leq
    A_S\{h_S^\beta+\omega_{n,h_S}(\delta_0)\},
    \label{eqn:source-boundary-projection-high-probability-rate}
    \\
    a_{T,m}(\delta_0)
    &\leq
    A_T\{h_T^\beta+\omega_{m,h_T}(\delta_0)\}.
    \label{eqn:target-boundary-projection-high-probability-rate}
\end{align}
The constants absorb the union bound over the two samples.  The
$L_\infty$ control is essential here: a fixed-$p$ $L_p$ bound alone would not
verify the full-support uniform weight condition in
Assumption~\ref{ass:bridge-weight-regularity}.

These requirements are not asymptotically empty. Put
$(N_S,N_T)=(n,m)$. For example, taking
\[
    h_D\asymp
    \left(\frac{\log N_D}{N_D}\right)^{1/(2\beta+d)},
\]
makes the density error for sample $D$ converge to zero at order
$(\log N_D/N_D)^{\beta/(2\beta+d)}$ for fixed $\delta_0$.  If
\[
    K_{B,\star}
    =
    \sup_{\bm\theta\in\Theta_B}
       \frac{L_{B,I_\star}(\bm\theta)}{c_B(\bm\theta)}
    <\infty,
\]
then $a_{S,n}(\delta_0)\leq c_S/2$,
$e_{n,m}(\delta_0)\leq r_-/2$, and
$K_{B,\star}e_{n,m}(\delta_0)\leq1/2$ all hold once $n$ and $m$ are sufficiently large.
Finiteness of $K_{B,\star}$ follows for the bridge families considered here from the
compact parameter set and the compact interval $I_\star$; for KL/HD this uses
$\alpha\leq\alpha_{\mathrm{KH},\max}<1$.  As a concrete nonconstant example on $[0,1]^d$,
\[
    p_S^X(x)=1,
    \qquad
    p_T^X(x)=1+\epsilon(2x_1-1),
    \qquad 0<\epsilon<1,
\]
satisfies all density and overlap requirements, with
$1-\epsilon\leq\LR\leq1+\epsilon$.

At $\delta_0=\delta/2$, suppose in addition that
\[
    a_{S,n}(\delta/2)\leq c_S/2,
    \qquad
    e_{n,m}(\delta/2)\leq r_-/2,
    \qquad
    K_{B,\star}e_{n,m}(\delta/2)\leq\frac12.
\]
Combining Proposition~\ref{prop:bridge-weight-error-density-estimation} with
Proposition~\ref{prop:vc-structural-bridge-complexity} then yields the explicit
rate
\begin{align}
\mathfrak C_{n,m}^B(\bm\theta;\delta)
\lesssim{}&
M\left[
  \sqrt{
    \frac{K_{B,n}(\delta)}
         {n_{\mathrm{eff},n}^B(\bm\theta)}
  }
  +
  \frac{C_B(\bm\theta)}{P_nw_{\bm\theta}^B}
  \frac{K_{B,n}(\delta)}{n}
\right]
\notag\\
&+
\frac{ML_{B,I_\star}(\bm\theta)}{c_B(\bm\theta)}
\Bigl[
  h_S^\beta+h_T^\beta
  +\omega_{n,h_S}(\delta/2)
  +\omega_{m,h_T}(\delta/2)
\Bigr],
\label{eqn:explicit-density-structural-complexity}
\end{align}
simultaneously for every $\bm\theta\in\Theta_B$, where the suppressed constant
depends on the density lower and upper bounds,
the projection and smoothness constants, and the entropy characteristics.
Thus
the finite-sample error separates the source empirical-process complexity
from the source and target density-estimation errors. More specialized or
direct likelihood-ratio estimators may instead be substituted in
\eqref{eqn:bridge-weight-high-probability-bound} when they provide sharper
uniform control for a particular bridge geometry; see, for example,
\citet{Nguyen2010}.  The role and robustness of importance weighting under
estimated ratios are studied more broadly by
\citet{Kato2023} and \citet{Gogolashvili2023}.

Appendix~\ref{sec:radius-indexed-selection} characterizes the admissible
parameters of the full symmetric bridge paths and gives sufficient target-radius
and endpoint-slope conditions under which every minimizer of the deterministic
finite-sample risk-bound envelope belongs to
$[\alpha_-(\rho_T),1)\subset(0,1)$. Under target feasibility, the source
radius does not truncate the path, while the target radius may exclude
$\alpha=0$ and the endpoint-slope condition may exclude $\alpha=1$. As
$\alpha\uparrow1$, the bridge law approaches $P_T^\circ$, with the remaining
target-reference discrepancy governed by $\delta_T^\circ$.

\section{Application to non-parametric kernel regression}\label{sec:kernel-regression}

To demonstrate how the doubly-anchored bridge actively shapes prediction without the rigid constraints of a parametric hypothesis class, we apply our framework to nonparametric kernel regression.
We choose kernel regression, specifically the Nadaraya--Watson (NW) estimator, as our illustrative example, because kernel regression is well studied in the empirical risk minimization and DA literature \citep[e.g.,][]{Ma2023}.
We can therefore clearly contrast how the doubly-anchored bridge changes the structure of the resulting estimator compared to classical DRDA.
It also shows how the domain likelihood ratio enters through a local,
covariate-dependent weight rather than only through finitely many model
coordinates.

We specialize in this section to a real-valued response $Y$.
We consider the source-supervised setting where we observe $n$ labeled source data points, $(\mat X_i^S, Y_i^S)_{i=1}^n$, and $m$ unlabeled target covariates, $(\mat X_j^T)_{j=1}^m$.
The true target regression function of interest is $m_T(\mat x) = \mathbb{E}_{P_T^{Y\mid\mat X}}[Y \mid \mat X= \mat x]$.
Without some restriction on the relation between the source and target conditional laws, $m_T$ is not identified from labeled source data and unlabeled target covariates alone.
In this setup, the source labels provide the foundational supervised learning signal, while the target covariates define the regions where the adapted predictor should concentrate its mass.

Within this section, $p_S$ and $p_T$ denote the marginal covariate densities
$p_S^X$ and $p_T^X$.  At interior locations, one possible construction of
the empirical anchors uses ordinary convolution kernel density estimators
with nonnegative smoothing kernels:
$$
\hat{p}_S(\mat x) = \frac{1}{n}\sum_{i=1}^n K_{h_S}(\mat x-\mat X_i^S), \qquad \hat{p}_T(\mat x) = \frac{1}{m}\sum_{j=1}^m K_{h_T}(\mat x-\mat X_j^T),
$$
where $K_{h_S}$ and $K_{h_T}$ are smoothing kernels with bandwidths $h_S$ and $h_T$.
The resulting empirical likelihood ratio is
\[
    \widehat\LR(x)
    =
    \begin{cases}
      \widehat p_T(x)/\widehat p_S(x),&\widehat p_S(x)>0,\\
      0,&\widehat p_S(x)=0.
    \end{cases}
\]
The fallback value is immaterial on the positive-denominator events used
below.  Appendix~\ref{app:asymptotics-loss-agnostic-bridge-weights} permits
signed higher-order kernels only for its separate pointwise limit
calculation and uses a corresponding positive-density fallback convention.
For uniform statements on a compact region carrying all source mass, we
instead use the support-adapted density-estimation input of
Section~\ref{sec:nonparametric-density-estimation}; the boundary-projection
construction there and in
Appendix~\ref{app:asymptotics-loss-agnostic-bridge-weights} is one sufficient
choice. The boundary correction enters only in verifying
the uniform density-estimation input and does not alter the bridge or
empirical-process arguments.  Any support-adapted density or direct
likelihood-ratio estimator satisfying the stated uniform bound may be used.

In the three structural bridges discussed in Section~\ref{sec:bridge-geometries}, the bridge parameter $\alpha$ is the geometry-specific source--target bridge-balance parameter.
Crucially, in a covariate-dependent setting, $\alpha$ does not simply form a linear combination of two distinct predictors.
Instead, it defines the unnormalized structural weight~\eqref{eqn:normalized-structural-bridge-law} and the normalized bridge law satisfies~\eqref{eqn:normalized-structural-bridge-law}.
These induced weights determine how source-labeled observations are
integrated into the adapted NW prediction rule.
For KL/HD, $\alpha\in[0,1)$ denotes the effective coordinate obtained after fixing the redundant structural offset $\gamma=0$ as in~\eqref{eqn:kh-effective-alpha}.
For simplicity, we use $\bm\theta=(\alpha)$ to denote the structural parameter throughout.

\subsection{Loss-adversarial tilting of the conditional risk}

Because the methodology is formulated for unlabeled target data, the target
sample can inform only the covariate marginal distribution
$P_T^{\mat X}$.
Consequently, when evaluating the covariate-level adversarial risk, the labels are supplied by the source conditional distribution $P_S^{Y\mid X}(\mathrm{d} y\mid \mat x) = P_S(\mathrm{d} y \mid \mat X = \mat x)$.

For any given predictor $h \in \hs$, the expected loss conditioned on the source covariate value $\mat X= \mat x$ is
$$
G_h(\mat x) = \mathbb{E}_{P_S^{Y\mid \mat X}}[\ell(h(\mat x), Y) \mid \mat X= \mat x].
$$
If the adversary shifts the covariate law to some distribution $Q^{\mat X}$, the resulting risk becomes the integral of this conditional risk over the new covariate marginal: $\int G_h(\mat x) \d Q^{\mat X}(\mat x)$.
Therefore, $G_h$ acts as the conditional risk seen by the covariate bridge; the bridge dictates how mass is shifted over $\mat x$, while the localized supervised loss at each point is governed by $G_h(\mat x)$.
More precisely, for a covariate law $Q^X\ll P_S^X$, define the completed law
\[
    Q(\mathrm{d} x,\mathrm{d} y)
    =Q^X(\mathrm{d} x)P_S^{Y\mid X}(\mathrm{d} y\mid x).
\]
As shown in~\eqref{eqn:completed-marginal-divergences}, the two density
ratios depend only on $x$, and hence
\[
    D_{\phi_S}(Q\Vert P_S)
    =D_{\phi_S}(Q^X\Vert P_S^X),
    \qquad
    D_{\phi_T}(Q\Vert P_T^\circ)
    =D_{\phi_T}(Q^X\Vert P_T^X).
\]
Thus the relevant covariate-level ambiguity set is
\begin{equation}\label{eqn:covariate-level-ambiguity-set}
\as^X(P_S^X,P_T^X)
=
\left\{
Q^X\in\mathcal P(\xspace):
Q^X\ll P_S^X,
D_{\phi_S}(Q^X\Vert P_S^X)\leq\rho_S,
D_{\phi_T}(Q^X\Vert P_T^X)\leq\rho_T
\right\}.
\end{equation}
For $h\in\hs$, define the corresponding covariate-level robust value by
\begin{equation}\label{eqn:covariate-level-robust-value}
\Psi_{S,T}^X(h)
=
\sup_{Q^X\in\as^X(P_S^X,P_T^X)}
\int_{\xspace}G_h(x)\,\d Q^X(x)
=
\sup_{Q^X\in\as^X(P_S^X,P_T^X)}
\E_Q[L_h].
\end{equation}
By~\eqref{eqn:completed-marginal-divergences},
$Q\in\as_{\mathrm{ac}}(P_S,P_T^\circ)$ whenever
$Q^X\in\as^X(P_S^X,P_T^X)$. Hence the covariate-level problem is the
restriction of the anchor-dominated joint problem to laws completed with the
source conditional kernel, and
\begin{equation}\label{eqn:covariate-and-joint-robust-values}
\Psi_{S,T}^X(h)
\leq
\Psi_{S,T}^{\mathrm{ac}}(h)
\leq
\Psi_{S,T}(h).
\end{equation}

Under Assumptions~\ref{ass:bounded-loss}--\ref{ass:target-feasibility},
the data-processing inequality for the projection $(x,y)\mapsto x$ gives
\[
D_{\phi_S}(P_T^X\Vert P_S^X)
\leq
D_{\phi_S}(P_T\Vert P_S)
\leq\rho_S,
\]
while $D_{\phi_T}(P_T^X\Vert P_T^X)=0\leq\rho_T$.
Assumption~\ref{ass:common-domination} also gives
$P_T^X\ll P_S^X$, so $P_T^X\in\as^X(P_S^X,P_T^X)$. Consequently,
\begin{equation}\label{eqn:completed-target-and-robust-values}
R_T^\circ(h)
=
\E_{P_T^X}[G_h(X)]
\leq
\Psi_{S,T}^X(h).
\end{equation}
Combining~\eqref{eqn:completed-target-and-robust-values} with
Lemma~\ref{lem:risk-difference-tv} yields
\begin{equation}\label{eqn:covariate-level-target-risk-bound}
R_T(h)
\leq
\Psi_{S,T}^X(h)
+
\begin{cases}
M\sqrt{\delta_T^\circ/2},
& \phi_T=\phi_{\mathrm{KL}},\\[2mm]
M\sqrt{\delta_T^\circ},
& \phi_T=\phi_{\mathrm{HD}}.
\end{cases}
\end{equation}
Indeed, Pinsker's inequality gives the first correction, while
$\operatorname{TV}\leq H$ gives the second.
Thus, without covariate shift, the covariate-level adversary controls
target-compatible reallocations of covariate mass around the source
conditional-risk baseline, while the final term in
\eqref{eqn:covariate-level-target-risk-bound} accounts for the unobserved
conditional shift. Under covariate shift, $\delta_T^\circ=0$ and the
correction vanishes. The covariate-level problem does not itself vary the
conditional law.

The corresponding structural bridge risk is
$$
R_{\bm\theta}^B(h)=\frac{\E_{P_S^X}[w_{\bm\theta}^B(X)G_h(X)]}{\E_{P_S^X}[w_{\bm\theta}^B(X)]},
$$
the covariate-level specialization of~\eqref{eqn:normalized-structural-bridge-law}.
We can therefore use two sets of bridge weights.
The density-bridge weights $w_{\bm\theta}^B(\mat x)$ defined in Section~\ref{sec:loss-agnostic-density-bridges} rely entirely on the source--target density geometry and are agnostic to the downstream loss function.
When the corresponding marginal dual optimum is attained in the interior
regime described in Proposition~\ref{prop:loss-adversarial-bridges}, the
loss-adversarial weights are obtained by applying that proposition on the
covariate space with source and target anchors $P_S^X$ and $P_T^X$ and
pointwise loss $G_h$.
Their marginal adversarial ratio is $r_h^{B,X,\star}(\mat x)=\mathrm{d} Q_h^{B,X,\star}/\mathrm{d} P_S^X(\mat x)$.
They actively tilt the estimator toward covariate regions having high source conditional risk.
Under the covariate shift assumption, \citet{Yamamoto2026} achieve a similar target-induced tilt by adding an additive component for the predictor shift to the objective function.

\subsection{Empirical plug-in estimation and asymptotic properties}

The ordinary source Nadaraya--Watson (NW) estimator uses a regression kernel $K_{b}$ to smooth local labels \citep{Nadaraya1964,Watson1964}.
By incorporating bridge weights $w_i$ evaluated at the source sample points $\bX_i^S$, the adapted estimator modifies the local influence of each observation:
$$
\hat{m}_w(\bx) = \frac{\sum_{i=1}^n w_i K_{b}(\bx-\bX_i^S)Y_i^S}{\sum_{i=1}^n w_i K_{b}(\bx-\bX_i^S)}.
$$
The ratio is assigned the value zero when its denominator vanishes; the
results below show that this exceptional event has probability tending to
zero.  The density bandwidths $h_S$ and $h_T$ are used to estimate the
underlying marginal densities, while the regression bandwidth $b$ smooths
the final prediction.
The bridge structurally alters prediction by changing these local NW weights.

The pure density-bridge formulation provides a computationally efficient, one-pass adaptation.
We first estimate the empirical likelihood ratio at each source point,
\[
    \widehat{\LR}_i
    =\frac{\widehat p_T(\bX_i^S)}{\widehat p_S(\bX_i^S)}.
\]
The empirical loss-agnostic density-bridge weights $\hat{w}_i$ are then computed directly from the chosen domain geometry per Section~\ref{sec:loss-agnostic-density-bridges}.
Because these weights depend strictly on the covariate densities and the geometric bridge parameters, they can be computed before fitting the NW regression.
Substituting $\hat{w}_i$ into the adapted estimator yields $\hat{m}_{\mathrm{br}}(x)$, achieving knowledge transfer without requiring an iterative fixed-point solution.

It is important to distinguish the finite-sample behavior of the bridge-weighted NW estimator from its asymptotic limit.
In finite samples, particularly when the smoothing bandwidth $b$ is large relative to the data sparsity (e.g., global or near-global averaging), the density-bridge weights $\hat{w}_i$ actively dictate which source observations dominate the prediction.
However, as $n \to \infty$ and $b \to 0$, the local averaging becomes
pointwise.
Therefore, the smooth bridge weight $w_0(x)$ cancels out of the numerator and denominator at first order.
The estimator thus converges to the source conditional mean $m_S(x) = \mathbb{E}_{P_S}[Y \mid X=x]$, and exact target recovery requires $m_S=m_T$, as under covariate shift.
We present these asymptotic theorems not to claim a general solution to arbitrary shift, but to show that the doubly-anchored bridge is a finite-sample structural regularizer that is first-order neutral for the source regression limit and the asymptotic variance.
Technical details and proofs are given in Appendix~\ref{app:asymptotics-nw}.

For the statements below, fix a bridge family $B$ and a population parameter
$\bm\theta_0$, and write $w_0=w_{\bm\theta_0}^B$.
The estimator $\widehat m_0$ uses the population weight $w_0$,
$m_b$ is its population-smoothed counterpart, and
$\widehat m_{\widehat w,\widehat{\bm\theta}}^B$ uses the estimated likelihood
ratio and bridge parameter. These objects are defined precisely in
Appendix~\ref{app:asymptotics-nw}. The pointwise and compact-set results are
local to an interior neighborhood of the evaluation point or set. Ordinary
convolution density estimators may therefore be used under the corresponding
local conditions; no full-support boundary claim is made in these results.

\begin{theorem}[Bridge-weighted Nadaraya--Watson consistency and rate]
\label{thm:nw-consistency-rate}
\leavevmode\par\noindent
Suppose Assumptions~\ref{ass:nw-kernel}--\ref{ass:nw-bandwidth} hold at a
fixed point $x_0\in\operatorname{int}(\cX)$. Then
\[
\widehat m_0(x_0)-m_S(x_0)
=\Op\left(b^s+(nb^d)^{-1/2}\right),
\]
and, if Assumption~\ref{ass:nw-weight-error} also holds on a neighborhood
of $x_0$,
\[
\widehat m_{\widehat w,\widehat{\bm\theta}}^{B}(x_0)-m_S(x_0)
=\Op\left(b^s+(nb^d)^{-1/2}+\bar\delta_{n,m}\right).
\]

For a compact interior evaluation set
$\mathcal K\Subset\operatorname{int}(\cX)$, suppose
Assumptions~\ref{ass:nw-kernel}--\ref{ass:nw-bandwidth} and
\ref{ass:nw-uniform-ep} hold on a neighborhood of $\mathcal K$. Then
\[
\sup_{x\in\mathcal K}|\widehat m_0(x)-m_S(x)|
=\Op(b^s+\zeta_{n,b}),
\]
and, if Assumption~\ref{ass:nw-weight-error} also holds on that neighborhood,
\[
\sup_{x\in\mathcal K}|\widehat m_{\widehat w,\widehat{\bm\theta}}^{B}(x)-m_S(x)|
=\Op(b^s+\zeta_{n,b}+\bar\delta_{n,m}).
\]
With a locally bounded response envelope and the usual compact-support
kernel empirical-process conditions, one may take
\[
\zeta_{n,b}\asymp \sqrt{\frac{\log n}{nb^d}},
\]
up to constants. Under only the stated moment condition, the abstract
$\zeta_{n,b}$ allows for a moment-dependent tail term.
\end{theorem}

\begin{theorem}[Oracle bridge-weighted Nadaraya--Watson pointwise CLT]\label{thm:nw-oracle-clt}
\leavevmode\par\noindent
Suppose Assumptions~\ref{ass:nw-kernel} and
\ref{ass:nw-regularity}--\ref{ass:nw-bandwidth} hold at a fixed
point \(x_0\in\operatorname{int}(\cX)\).
Assume additionally that \(\sigma_S^2(x_0) = \Var_{P_S}[Y \mid X = x_0] > 0\).
Then
\[
\sqrt{nb^d}\{\widehat m_0(x_0)-m_b(x_0)\}
\dto
N\left(0,\frac{R(K)\sigma_S^2(x_0)}{\pS(x_0)}\right).
\]
If, in addition,
\(
\sqrt{nb^d}\,b^s\to0,
\)
then
\[
\sqrt{nb^d}\{\widehat m_0(x_0)-m_S(x_0)\}
\dto
N\left(0,\frac{R(K)\sigma_S^2(x_0)}{\pS(x_0)}\right).
\]
\end{theorem}

\begin{theorem}[Plug-in bridge-weighted Nadaraya--Watson pointwise CLT]\label{thm:nw-plugin-clt}
\leavevmode\par\noindent
Suppose the assumptions of Theorem~\ref{thm:nw-oracle-clt} hold and
Assumption~\ref{ass:nw-weight-error} holds in a neighborhood of
\(x_0\in\operatorname{int}(\cX)\).
If
\[
\sqrt{nb^d}\,\bar\delta_{n,m}\to0,
\]
then
\[
\sqrt{nb^d}\{\widehat m_{\widehat w,\widehat{\bm\theta}}^{B}(x_0)-m_b(x_0)\}
\dto
N\left(0,\frac{R(K)\sigma_S^2(x_0)}{\pS(x_0)}\right).
\]
Under the density-estimation and parameter-rate specification in
Assumption~\ref{ass:nw-weight-error}, a sufficient condition is
\[
\sqrt{nb^d}\{a_{S,n}+a_{T,m}+r_{\bm\theta,n,m}\}\to0.
\]
If additionally
\(
\sqrt{nb^d}\,b^s\to0,
\)
then
\[
\sqrt{nb^d}\{\widehat m_{\widehat w,\widehat{\bm\theta}}^{B}(x_0)-m_S(x_0)\}
\dto
N\left(0,\frac{R(K)\sigma_S^2(x_0)}{\pS(x_0)}\right).
\]
\end{theorem}
The oracle-equivalence condition is non-vacuous under standard
bandwidth regimes. For example, suppose $m\asymp n$,
$h_S\asymp h_T\asymp n^{-a}$, $b\asymp n^{-c}$, the marginal
densities are $\beta$-H\"older smooth, and the bridge parameter is fixed, so
$r_{\bm\theta,n,m}=0$. Ignoring logarithmic factors, choose
\[
    \frac{1}{2\beta+d}<c<\frac1d,
    \qquad
    \frac{1-cd}{2\beta}<a<c.
\]
Then $nb^d\to\infty$ and
\[
    \sqrt{nb^d}\,h_S^\beta\to0,
    \qquad
    \sqrt{nb^d}\,(nh_S^d)^{-1/2}
    =\left(\frac{b}{h_S}\right)^{d/2}\to0,
\]
with the same conclusions for the target density estimator.
Strict inequalities also absorb the logarithmic factors in the uniform
density-estimation rates supplied in
Section~\ref{sec:nonparametric-density-estimation}.
The same bandwidths apply to an estimated bridge parameter whenever
$\sqrt{nb^d}\,r_{\bm\theta,n,m}\to0$. A root-$n$ parameter rate is one
sufficient case if it is established separately, since then this term is
of order $b^{d/2}$; it is not a consequence of the radius-profile diagnostic
below. If centering at $m_S(x_0)$ is desired, one may additionally
take $c>(2s+d)^{-1}$ so that
$\sqrt{nb^d}\,b^s\to0$.

The asymptotic properties established in Theorems~\ref{thm:nw-consistency-rate}--\ref{thm:nw-plugin-clt} have several important implications for our doubly-anchored framework.
First, Theorem~\ref{thm:nw-consistency-rate} gives a decomposition of the
prediction error into the classical non-parametric smoothing bias, the local
stochastic error, and the structural bridge-weight estimation error
$\bar\delta_{n,m}$.

Second, and most crucially, Theorem~\ref{thm:nw-plugin-clt} establishes an oracle equivalence.
If the empirical density estimates and the bridge parameters converge
sufficiently fast that $\sqrt{nb^d}\,\bar\delta_{n,m}\to0$, the empirical
plug-in estimator has the same asymptotic normal distribution as the oracle
estimator equipped with the population weights.
This implies that the statistical cost of learning the geometric domain shift from data does not inflate the first-order asymptotic variance of the bridge-weighted source regression estimator.

Finally, the limiting variance $R(K)\sigma_S^2(x_0)/\pS(x_0)$ derived in
Theorem~\ref{thm:nw-oracle-clt} reveals a fundamental characteristic of the
non-parametric bridge. The doubly-anchored weights can alter finite-bandwidth
local averaging and finite-sample influence, but a positive smooth weight
cancels from the first-order pointwise variance. The limiting precision is
therefore controlled by the source domain's local design density and noise.
Together with consistent local estimators of $p_S(x_0)$ and
$\sigma_S^2(x_0)$, the CLTs yield asymptotically valid pointwise confidence
intervals for $m_S(x_0)$ under undersmoothing or a suitable bias correction;
when $m_S(x_0)=m_T(x_0)$, including under covariate shift, the same intervals
apply to the target regression function.

\subsubsection{A loss-aware empirical refinement}
\label{sec:loss-adversarial-refinement}

To incorporate the loss-aware tilt suggested by the population dual, we
consider a marginal loss-aware refinement.
For a candidate predictor $h \in \hs$, the population conditional risk $G_h$
can be replaced by the local NW-smoothed estimate
\[
\widehat G_h(\bx)
=
\begin{cases}
\displaystyle
\frac{\sum_{i=1}^n K_b(\bx-\bX_i^S)
      \ell\{h(\bX_i^S),Y_i^S\}}
     {\sum_{i=1}^n K_b(\bx-\bX_i^S)},
&\displaystyle \sum_{i=1}^nK_b(\bx-\bX_i^S)>0,\\[12pt]
0,&\text{otherwise}.
\end{cases}
\]
The fallback value is immaterial on regions where the local denominator is
positive.
Using this empirical risk, the plug-in marginal loss-aware weights
$\hat{r}_{h,i}^{B,X}$ are obtained from
Proposition~\ref{prop:loss-adversarial-bridges}, with the pointwise loss
replaced by $\widehat G_h$.
They adaptively tilt the baseline density bridge toward target-compatible covariate regions having high estimated source conditional risk.

Unlike the density bridge, these loss-aware weights depend on the current predictor $h$ through the loss.
For any fixed $h$, their empirical normalization gives the associated
loss-aware reweighted criterion
\[
\widehat{\Psi}_{\mathrm{rw}}(h)
=
\frac{\sum_{i=1}^n \hat{r}_{h,i}^{B,X}L_h(Z_i^S)}
     {\sum_{i=1}^n \hat{r}_{h,i}^{B,X}}.
\]
Starting from the loss-agnostic bridge weights, the implementation alternates
between fitting a weighted NW predictor and updating the loss-aware weights
through the plug-in response map from
Proposition~\ref{prop:loss-adversarial-bridges}.
Appendix~\ref{supp:implementation} gives the resulting damped
fixed-point iteration. The grid-indexed procedure is a computational
loss-aware refinement guided by the dual-induced response maps. The
statistical guarantees above concern the loss-agnostic structural bridge
estimator.

\subsection{Choosing bridge parameters}

If we observed target labels, the bridge parameter could be tuned using target validation risk.
However, because we only observe $\{(X_i^S, Y_i^S)\}_{i=1}^n$ and $\{X_j^T\}_{j=1}^m$, the bridge parameter $\alpha$ cannot be chosen by minimizing empirical target prediction error.
Instead, it is an unobserved parameter governing the source--target weighting geometry.
With unlabeled target data, the hyperparameters must be selected directly from the bridge geometry or the robust optimization framework, rather than target risk.
We discuss three strategies for handling the bridge parameters without access
to target labels: dual profiling, a radius-profile diagnostic, and sensitivity
analysis.
For KL/HD, $\alpha$ denotes the effective coordinate in
\eqref{eqn:kh-effective-alpha} for the structural bridge and, at fixed
$(h,\nu)$, for the normalized loss-aware response described after that
identity. We set $\gamma=0$ in these canonical parameterizations; the
corresponding original dual balance in the exact KKT representation depends on
the current $h$ and $\nu$ through normalization.

\paragraph{Dual profiling.}
Because the bridge balance emerges from the Lagrangian dual variables, it
can be profiled through the robust dual objective.  For KL/KL, write
$\lambda_S=\nu(1-\alpha)$ and $\lambda_T=\nu\alpha$, where
$\nu>0$ and $\alpha\in[0,1]$.  After profiling out the normalization
multiplier, the smoothed plug-in profiled dual criterion is
\[
\widehat\Psi_{\mathrm{KL}}(h;\alpha,\nu)
=
\nu\left[
  (1-\alpha)\rho_S+\alpha\rho_T
  +\log\int
    \exp\left\{\frac{\widehat G_h(x)}{\nu}\right\}
    \widehat p_S(x)^{1-\alpha}\widehat p_T(x)^\alpha
    \d\mu_X(x)
\right].
\]
At $\alpha=0$ and $\alpha=1$, the density product in the integrand is
understood as $\widehat p_S$ and $\widehat p_T$, respectively.  The smoothed
conditional source risk $\widehat G_h$ appears because the adversary shifts
the covariate marginal while labels are supplied by the source conditional
law.  If the regularized empirical marginal ambiguity problem is strictly
feasible, this criterion is its robust dual after the normalization
multiplier has been profiled out.  More generally, the optimization below is
used only when the profiled criterion is bounded below; empirical strict
feasibility is a sufficient condition and supplies the robust-primal
interpretation.

When a minimizer exists, one may choose
\[
(\widehat h,\widehat\alpha)
\in
\argmin_{h\in\hs,\,\alpha\in[0,1]}
\inf_{\nu>0}\widehat\Psi_{\mathrm{KL}}(h;\alpha,\nu),
\]
and, if the infimum is finite but unattained, use approximate minimizers.
The closed interval accommodates a zero dual multiplier while retaining the
corresponding constraint in the primal problem. Appendix~\ref{sec:radius-indexed-selection}
instead optimizes the deterministic risk-bound envelope over the admissible
loss-agnostic structural bridge laws; its sufficient conditions can exclude
the two structural endpoints for that objective.

\paragraph{Radius-profile diagnostic.}
A second, heuristic approach decouples the bridge parameter from the
downstream prediction task and uses the geometry of the divergence
constraints. Suppose that the radii $\rho_S$ and $\rho_T$ are supplied from
external scientific information or treated as sensitivity inputs. Here
$\widehat P_D^X(\mathrm{d} x)=\widehat p_D(x)\mu_X(\mathrm{d} x)$, $D\in\{S,T\}$, are
regularized density anchors rather than raw atomic empirical laws. Put
\[
    \widehat c_{\bm\theta}^B
    =\int
      g_{\bm\theta}^B\{\widehat\LR(x)\}
      \widehat p_S(x)\d\mu_X(x)
\]
and define the normalized empirical bridge law by
\[
    \frac{\d\widehat P_{\bm\theta}^{B,X}}{\d\mu_X}(x)
    =
    \frac{
      g_{\bm\theta}^B\{\widehat\LR(x)\}\widehat p_S(x)
    }{
      \widehat c_{\bm\theta}^B
    }.
\]
The latter definition is used whenever
$\widehat c_{\bm\theta}^B>0$; the regularization below ensures this for every
sample.
For support-adapted anchors on the compact full-mass working region
$\mathcal X_0$, finite-sample positivity may, if needed, be enforced by a
common reference law $Q_0$ supported on $\mathcal X_0$ whose density is
bounded and bounded away from zero there.  Namely, replace each anchor by
\[
    \widehat P_{D,\epsilon}^X
    =(1-\epsilon_{n,m})\widehat P_D^X+\epsilon_{n,m}Q_0,
    \qquad D\in\{S,T\},
\]
where $0<\epsilon_{n,m}\downarrow0$ is deterministic and has no larger order
than the uniform density-estimation error.  This makes the empirical KL
criterion finite for every sample without changing that rate. Whenever this
regularization is used, the two fitted marginal anchors are mutually
absolutely continuous on $\mathcal X_0$, and $\widehat\LR$,
$\widehat c_{\bm\theta}^B$, and $\widehat P_{\bm\theta}^{B,X}$ in the
preceding definitions are all computed from the mixed anchors. On the event
$\mathcal E_{\delta_0}$ in~\eqref{eqn:density-high-probability-event}, under
$a_{S,n}(\delta_0)\leq c_S/2$ and $e_{n,m}(\delta_0)\leq r_-/2$ as in
Proposition~\ref{prop:bridge-weight-error-density-estimation}, one has
$\widehat p_S\geq c_S/2$ and $\widehat\LR\geq r_-/2$ on $\mathcal X_0$.
Thus both fitted densities are already bounded away from zero there, and the
regularization is asymptotically immaterial.

For diagnostic purposes, define an empirical radius-matching point by
\[
\widehat{\bm\theta}
\in
\argmin_{\bm\theta\in\Theta_B}
\left[
  \left\{
    D_{\phi_S}(\widehat P_{\bm\theta}^{B,X}\Vert\widehat P_S^X)-\rho_S
  \right\}^{2}
  +
  \left\{
    D_{\phi_T}(\widehat P_{\bm\theta}^{B,X}\Vert\widehat P_T^X)-\rho_T
  \right\}^{2}
\right],
\]
with the corresponding population version
\begin{equation}\label{eqn:radius-calibration-estimate-population}
\bm\theta_0
\in
\argmin_{\bm\theta\in\Theta_B}
\left[
  \left\{
    D_{\phi_S}(P_{\bm\theta}^{B,X}\Vert P_S^X)-\rho_S
  \right\}^{2}
  +
  \left\{
    D_{\phi_T}(P_{\bm\theta}^{B,X}\Vert P_T^X)-\rho_T
  \right\}^{2}
\right].
\end{equation}
This criterion projects the prescribed radius pair onto the marginal
radius-profile curve. Since the radii define inequality constraints, it need
not make both constraints active and does not minimize the deterministic
risk-bound envelope in Theorem~\ref{thm:radius-indexed-interior-optimum}.
For the symmetric KL/KL and HD/HD bridges with $\Theta_B=[0,1]$, if the
population admissible interval in~\eqref{eqn:admissible-interval} is nonempty,
the monotonicity in Proposition~\ref{prop:radius-profile} implies that every
population minimizer of the squared criterion belongs to that interval:
outside the interval, moving toward it decreases both squared residuals.

Unlabeled marginal information also does not determine the two
target-feasible radii. Under Assumption~\ref{ass:target-feasibility}, data
processing gives
\[
    D_{\phi_S}(P_T^X\Vert P_S^X)
    \leq D_{\phi_S}(P_T\Vert P_S)
    \leq\rho_S,
    \qquad
    \delta_T^\circ
    =D_{\phi_T}(P_T\Vert P_T^\circ)
    \leq\rho_T.
\]
The first marginal divergence can be estimated from the source and target
covariates under suitable density-estimation conditions, but it is only a
necessary lower bound for $\rho_S$: it does not determine
$D_{\phi_S}(P_T\Vert P_S)$, which also depends on the unobserved target
conditional law.  Likewise, $\delta_T^\circ$ depends on that conditional law,
so the target covariate sample alone cannot establish a target-feasible value
of $\rho_T$.

The marginal divergences do provide convenient sufficient certificates for
the joint-space strict-feasibility condition in
Theorem~\ref{thm:general-dual-representation}. By
\eqref{eqn:completed-marginal-divergences}, the completed laws associated
with $\nu=P_S^X$ and $\nu=P_T^X$ give, respectively,
\[
\rho_S>0
    \quad\text{and}\quad
\rho_T>D_{\phi_T}(P_S^X\Vert P_T^X),
\qquad\text{or}\qquad
\rho_T>0
    \quad\text{and}\quad
\rho_S>D_{\phi_S}(P_T^X\Vert P_S^X).
\]
These are sufficient Slater certificates, rather than lower bounds required
of the radii. More generally,
Theorem~\ref{thm:general-dual-representation} requires some strictly
positive anchor-dominated joint density satisfying both inequalities
strictly; that law need not belong to the structural bridge family.

For the symmetric bridges under the nonconstant-likelihood-ratio conditions
of Proposition~\ref{prop:radius-profile}, target feasibility and $\rho_T>0$
give a direct bridge-based certificate. Indeed, for $\alpha<1$ sufficiently
close to $1$, that proposition gives
\[
    \delta_{B,T}(\alpha)<\rho_T,
    \qquad
    \delta_{B,S}(\alpha)
    <\delta_{B,S}(1)
    =D_{\phi_S}(P_T^X\Vert P_S^X)
    \leq\rho_S.
\]
Thus, a small positive target radius is compatible with strict feasibility.

For HD/HD, the radius profiles concern completed bridge laws, which are
anchor dominated. Lemma~\ref{lem:full-vs-ac-risk-hdhd} separately controls
the difference between the full and anchor-dominated robust risks when
$\rho_-<2$.

\paragraph{Sensitivity analysis.}
Finally, when informative radii are unavailable, the bridge coordinate can be
treated as a sensitivity parameter. For KL/KL and HD/HD,
$\alpha\in[0,1]$ traces the full source--target bridge path. For KL/HD, the
canonical effective coordinate satisfies $\alpha\in[0,1)$, with the
normalized bridge law converging to $P_T^\circ$ as $\alpha\uparrow1$.
The structural offset $\gamma$ is fixed at zero and is not varied as a second
parameter. Evaluating the adapted predictor $\widehat m_\alpha(x)$ over a grid
of effective bridge coordinates gives a sensitivity profile of the
predictions with respect to the degree and geometry of target-marginal
weighting.

\section{Empirical Illustration}\label{sec:empirical}

\subsection{Theorem-Linked Bounded-Overlap Simulation}
\label{sec:bounded-overlap-simulation}

We first use a design in which the population bridge geometry is available in
closed form. Let $X\in[0,1]$, with
\[
  p_S^X(x)=1,
  \qquad
  p_T^X(x)=1+0.5(2x-1),
\]
and use Bernoulli conditionals
\[
  P_S(Y=1\mid X=x)=\operatorname{logit}^{-1}\{f(x)\},
  \qquad
  P_T(Y=1\mid X=x)=\operatorname{logit}^{-1}\{f(x)-0.05(2x-1)\},
\]
where $f(x)=3\{(2x-1)^2-1/2\}$. The working loss is the Brier loss and
\[
  \mathcal H
  =\left\{h_\theta(x)=\frac12+\theta(2x-1):
  \theta=-0.500,-0.495,\ldots,0.500\right\}.
\]
We use $n=m=2000$, $1000$ independent replications, and
$\alpha=0,0.02,\ldots,1$. We fit each predictor to the labeled source data using the known likelihood ratio and evaluate its risk under the true target law by quadrature. The source and target radii are chosen to contain the true target law while excluding the source endpoint. Full
implementation details are given in
Section~\ref{supp:bounded-overlap-simulation}.

\begin{table}[t]
\centering
\small
\caption{Population quantities and Monte Carlo summaries for the
bounded-overlap simulation. The envelope minimizer is taken over the
radius-admissible interval $[\alpha_-(\rho_T),1]$.}
\label{tab:bounded-overlap-summary}
\begin{tabular}{lrr}
\toprule
Quantity & KL/KL & HD/HD \\
\midrule
Conditional-shift discrepancy $\delta_T^\circ$
  & $8.92\times10^{-5}$ & $4.46\times10^{-5}$ \\
Source radius $\rho_S$ & $0.0472$ & $0.0242$ \\
Target radius $\rho_T$ & $0.0114$ & $0.0055$ \\
Lower admissible endpoint $\alpha_-(\rho_T)$ & $0.485$ & $0.499$ \\
Monte Carlo mean-risk minimizer $\alpha_{\mathrm{risk}}$ & $0.64$ & $0.66$ \\
Admissible-envelope minimizer $\alpha_{\mathrm{env}}$ & $0.485$ & $0.499$ \\
Mean excess risk at $\alpha_{\mathrm{risk}}$
  & $1.29\times10^{-4}$ & $1.29\times10^{-4}$ \\
Mean ESS fraction at $\alpha_{\mathrm{risk}}$ & $0.965$ & $0.964$ \\
Endpoint-slope margin, $\mathfrak c'_{B,-}(1;\delta)-M\tau_B$ & $3.61$ & $3.93$ \\
\bottomrule
\end{tabular}
\end{table}

Figure~\ref{fig:bounded-overlap-simulation} shows the resulting radius,
risk, and ESS profiles. The minimizers of Monte Carlo mean target risk on the fitting grid are interior and lie well
inside the admissible ranges. The deterministic envelope is minimized at
$\alpha=0$ if it is optimized over the full interval $[0,1]$, but that point
violates the target-radius constraint. Over the radius-admissible interval
relevant to Theorem~\ref{thm:radius-indexed-interior-optimum}, its minimizer is
the lower endpoint $\alpha_-(\rho_T)$ for both geometries.
At $n=2000$ the spread term $\mathfrak s_B(\alpha)\lambda_{B,n}$ dominates. Even at its admissible minimum the envelope exceeds the unit range of the Brier loss. Its shape illustrates how the complexity cost changes with the bridge coordinate.
The positive endpoint-slope margins exclude $\alpha=1$, as predicted by the sufficient condition in Theorem~\ref{thm:radius-indexed-interior-optimum}.
The empirical ESS curves are visually indistinguishable from their population limits. At the coordinates minimizing Monte Carlo mean risk, the mean ESS exceeds $96\%$ of the source sample size.

\begin{figure}[t]
  \centering
  \includegraphics[width=6in]{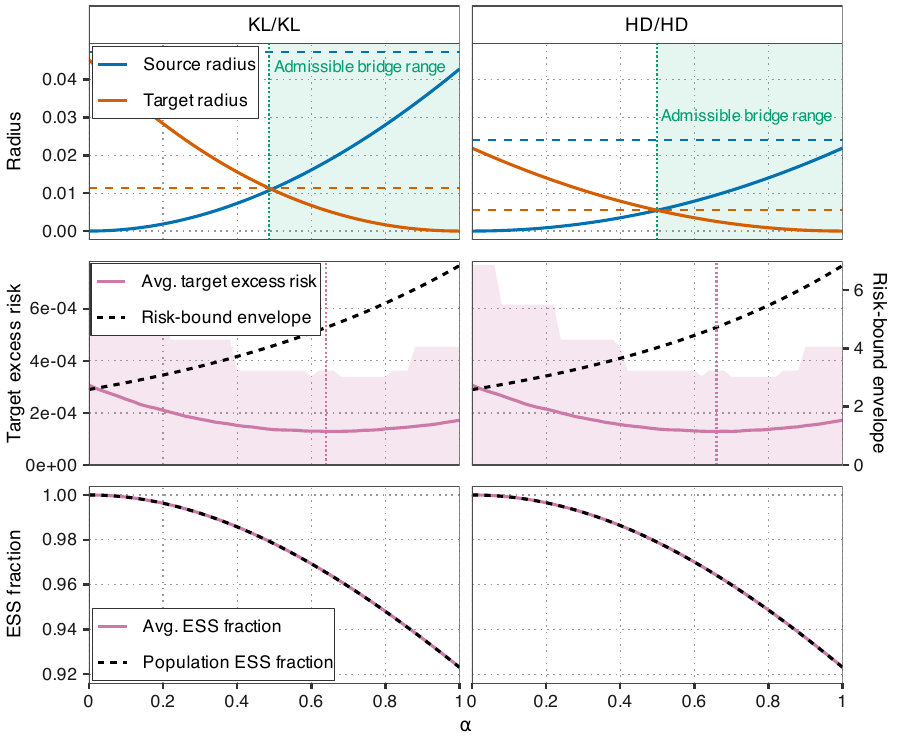}
  \caption{Theorem-linked bounded-overlap simulation for KL/KL (left) and
  HD/HD (right). The top row gives the exact radius profiles, prescribed
  radii, lower admissible endpoint, and shaded admissible range. The middle
  row compares Monte Carlo mean target excess risk (with 10--90\% bands) with
  the deterministic risk-bound envelope on the right-hand scale; the vertical
  magenta line marks the mean-risk minimizer. The bottom row compares the
  Monte Carlo mean ESS fraction with its population value.}
  \label{fig:bounded-overlap-simulation}
\end{figure}

\subsection{Fashion-MNIST Sensitivity Study}
\label{sec:fmnist}

As a secondary higher-dimensional illustration, we consider the Pullover and Coat classes of Fashion-MNIST \citep{Xiao2017}.
The experiment uses the bridge-weighted regression estimator of Section~\ref{sec:kernel-regression}. The source sample contains $12{,}000$ labeled training images and the target sample contains $2{,}000$ test images.
A target shift is induced by adding centered Laplace perturbations with pre-clipping standard deviation $0.15$ and clipping pixel intensities to $[0,1]$.
A source-trained autoencoder provides a ten-dimensional latent representation, Gaussian KDEs estimate the two covariate laws, and a weighted Nadaraya--Watson rule supplies the predictor.
The source and target samples each contain equal numbers of the two classes. We hold the source-trained representation and the images fixed across the 100 redraws of the target perturbations. Target labels are used only for the descriptive accuracy curves in Figure~\ref{fig:fmnist-sensitivity-accuracy}.

For each structural path the weights are normalized to have empirical mean one. At $\alpha=1$, the two symmetric paths use the same estimated likelihood-ratio weights. The predicted Pullover proportion measures class balance, while the summed binary entropy measures the sharpness of the predicted probabilities. The construction of the target shift is shown in Appendix~\ref{supp:implementation}.

The structural paths in Figure~\ref{fig:fmnist-sensitivity-accuracy} illustrate the concentration differences predicted by the three bridge maps: KL/KL concentrates most rapidly, while the canonical KL/HD path retains a larger plug-in ESS over the displayed range.
In the loss-aware diagnostic replication, all configurations have zero capped fraction, but only KL/HD has stable ESS, entropy, and predicted class proportion.
The full implementation, cap diagnostic, and reproducibility details are given in Section~\ref{supp:implementation} and in the public code repository at
\url{https://github.com/dakep/doubly-anchored-dro-for-da}.

\begin{figure}[t]
    \centering
    \includegraphics[width=5.5in]{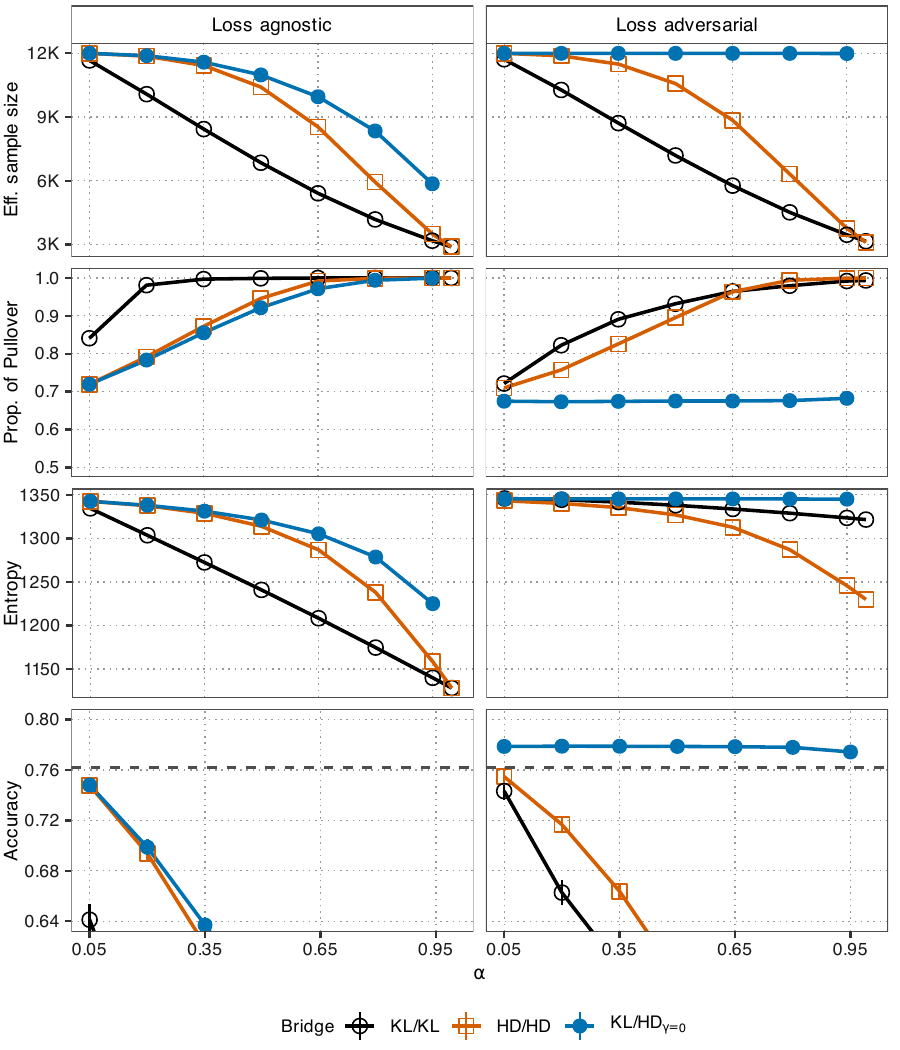}
    \caption{Sensitivity analysis and target accuracy of the domain-adapted NW estimators applied to the Fashion-MNIST data set with shifted target images.
    The left column uses loss-agnostic weights and the right uses the
    loss-aware iterative refinement; KL/HD has $\gamma=0$.
    The sensitivity analysis in rows 1--3 uses one target-shift replication.
    The plug-in effective sample size in the first row and the summed binary entropy in the third row are given in~\eqref{eqn:fmnist-ess-and-summed-entropy}.
    The bottom row gives the mean target accuracy over $100$ independent redraws of the Laplace perturbations, with the images, source-trained autoencoder, and source--target split held fixed; error bars show plus or minus one standard deviation.
    }
    \label{fig:fmnist-sensitivity-accuracy}
\end{figure}

\FloatBarrier

\section{Discussion}

Doubly anchored DRO incorporates the observed target covariate law as an active constraint on source-supervised worst cases, and our results highlight its three main consequences.
First, the symmetric geometries recover paths connecting source-only learning to raw importance weighting through tempered weights. For positive total multiplier, the relative target-side multiplier selects the center of an effective ball in the loss-aware problem.
The KL/KL and HD/HD bridges recover the flattened and square-root interpolated weights already used in covariate-shift adaptation, but here with radii that fix their admissible range.
KL/HD, on the other hand, adds a Lambert-$W$ map that attenuates large likelihood ratios without truncation.
Second, under target feasibility the target radius determines how much source borrowing is admissible along a given symmetric path. The finite-sample bounds quantify the trade-off between bridge-to-target discrepancy and effective sample size.
The conditional shift contributes a separate term to the risk bound. The source radius enters the residual loss-aware budget even when it permits every symmetric structural center.
Third, the loss-agnostic and loss-aware bridges show how loss tilting interacts with the source--target structure: multiplicatively for the symmetric geometries, and through a likelihood-ratio-dependent coupling for KL/HD.%

The finite-sample theory separates target locality, source-sample complexity through the realized effective sample size, and error from estimating the likelihood ratio.
In the bounded-overlap simulation, the known likelihood ratio makes the weight-estimation term zero, so the remaining terms can be examined directly.
Its realized target-risk minimizers are interior, while at the reported sample size the conservative deterministic
envelope selects the lower admissible endpoint because its spread-dependent complexity term dominates.
This difference is consistent with the role of the envelope as a sufficient finite-sample upper bound rather than a risk-selection identity.
The Fashion-MNIST study provides a complementary higher-dimensional sensitivity analysis and also shows why numerical regularization must be reported when it materially changes a response map.

We also want to stress a few limitations of the current work.
First, the finite-sample guarantees derived in Section~\ref{sec:bridge-risk-bounds} require a bounded likelihood-ratio overlap, which excludes heavy-tailed regimes that motivate the Hellinger and Lambert-$W$ attenuation.
Whether the structural bridges have favorable finite-sample behavior when relaxing this condition is an open question.
Second, although the dual formulation couples the source radius $\rho_S$ into the residual robust budget alongside $\rho_T$ and $\alpha$, we do not isolate the source anchor's marginal contribution relative to a target-ball-only formulation.
A direct comparison, similar to the min-envelope bounds used for single-ball DRO, is left for future work.
Finally, the loss-aware fixed-point iterations in Section~\ref{sec:loss-adversarial-refinement} are presented only as a computational refinement; we do not develop convergence or finite-sample guarantees.

The NW theory shows that fixed
positive smooth bridge weights share the source-regression limit and
first-order pointwise variance, so first-order asymptotics alone cannot rank
the bridge parameters. A natural direction is therefore a finite-bandwidth or
second-order risk analysis of the tradeoff between stabilization from
tempering the weights and the bridge-to-target-anchor locality discrepancy.

\acks{%
This project was supported by resources provided by the Office of Research Computing at George Mason University (URL: \url{https://orc.gmu.edu}) and funded in part by grants from the National Science Foundation (Award Number 2018631).%
}

\clearpage
\appendix
\section{Full vs.\ anchor-dominated robust risk}\label{app:full-vs-ac-riks}

Here, we quantify the gap between the full and the anchor-dominated population robust risk for the HD/HD pairing.
Since $\phi_{\mathrm{HD}}^\infty(1)=1$, the extended squared Hellinger divergence can remain finite for laws $Q \not\ll P_S$, and hence $\Psi_{S,T}$ can be larger than $\Psi_{S,T}^\mathrm{ac}$.
Since squared Hellinger divergence lies in $[0,2]$, the condition $\min\{\rho_S,\rho_T\}<2$ excludes the automatically vacuous regime $\rho_S,\rho_T\geq2$ and yields the nontrivial singular-mass bound used below.

\begin{lemma}[Full vs.\ anchor-dominated risk for the HD/HD pairing]\label{lem:full-vs-ac-risk-hdhd}
Suppose Assumptions~\ref{ass:bounded-loss}--\ref{ass:common-domination} hold and consider the extended squared Hellinger distance divergence for both anchors, $\phi_S=\phi_T=\phi_{\mathrm{HD}}$.
Assume that $\as(P_S,P_T^\circ)\neq\varnothing$.
Define $\rho_- = \min\{\rho_S, \rho_T \}$ and assume $\rho_- < 2$.
Further, let $\tilde\rho = 1 - \left(1 - \rho_-/2\right)^2$.
Then $0\leq\tilde\rho<1$ and, for every $h\in\hs$,
\[
0 \leq \Psi_{S,T}(h) - \Psi_{S,T}^\mathrm{ac}(h) \leq M \tilde\rho \leq M \rho_-.
\]
\end{lemma}
\begin{proof}
Fix $Q \in \as(P_S, P_T^\circ)$ and let $Q=Q^a + Q^s$ be its Lebesgue decomposition relative to $P_S$, with absolutely continuous part $Q^a \ll P_S$ and singular part of mass $m_\perp=Q^s(\fs) \geq 0$.
From Assumption~\ref{ass:common-domination}, $P_S \sim P_T^\circ$, hence the two anchors have the same null sets.
Thus, a measure is absolutely continuous, or singular, with respect to $P_S$ if and only if it has the corresponding property with respect to $P_T^\circ$.
Uniqueness of the Lebesgue decomposition therefore gives the same $Q^a$ and $Q^s$ for both anchors, with $Q^a(\fs)=1-m_\perp$.

Let $P_-$ be an anchor whose associated radius is $\rho_-$, choosing either anchor in case of a tie, and write
\[
    r_{a,-} = \frac{\d Q^a}{\d P_-},
    \qquad
    I_- = \int \sqrt{r_{a,-}}\,\d P_-.
\]
The extended Hellinger divergence in~\eqref{eqn:extended-phi-divergence}, with $\phi^\infty_\text{HD}(1)=1$, is
\[
    D_{\phi_{\mathrm{HD}}}(Q\Vert P_-)
    =\int(\sqrt{r_{a,-}}-1)^{2}\d P_-+m_\perp
    =2-2I_-.
\]
By Cauchy--Schwarz, $I_- \leq \sqrt{1-m_\perp}$. Feasibility at the radius-minimizing anchor therefore gives
\[
    1-\frac{\rho_-}{2}
    \leq I_-
    \leq \sqrt{1-m_\perp}.
\]
Since $\rho_-<2$, the left-hand side is positive, and hence
\[
    m_\perp
    \leq
    1-\left(1-\frac{\rho_-}{2}\right)^2
    =\tilde\rho
    <1.
\]

Define the dominated law $\tilde Q = Q^a / (1-m_\perp)$.
For either anchor $A\in\{P_S,P_T^\circ\}$, write
\[
    r_A=\frac{\d Q^a}{\d A},
    \qquad
    I_A=\int\sqrt{r_A}\,\d A.
\]
Since $I_A\geq0$ and $0<\sqrt{1-m_\perp}\leq1$,
\[
    D_{\phi_{\mathrm{HD}}}(\tilde Q\Vert A)
    =2-\frac{2I_A}{\sqrt{1-m_\perp}}
    \leq 2-2I_A
    =D_{\phi_{\mathrm{HD}}}(Q\Vert A).
\]
Therefore, $\tilde Q \in \as_\mathrm{ac}(P_S, P_T^\circ)$.
In particular, $\as_\mathrm{ac}(P_S,P_T^\circ)$ is nonempty and $\Psi_{S,T}^\mathrm{ac}(h)\geq0$.

Finally, using the bounded loss $0 \leq L_h \leq M$,
\[
\begin{aligned}
\E_Q[L_h]
  &=(1-m_\perp)\E_{\tilde Q}[L_h]+\int L_h\,\d Q^s \\
  &\leq(1-m_\perp)\Psi_{S,T}^\mathrm{ac}(h)+Mm_\perp \\
  &=\Psi_{S,T}^\mathrm{ac}(h)
    +m_\perp\{M-\Psi_{S,T}^\mathrm{ac}(h)\} \\
  &\leq\Psi_{S,T}^\mathrm{ac}(h)+Mm_\perp \\
  &\leq\Psi_{S,T}^\mathrm{ac}(h)+M\tilde\rho.
\end{aligned}
\]
Taking the supremum over $Q \in \as(P_S, P_T^\circ)$ gives the upper bound.
The inclusion of $\as_\mathrm{ac}(P_S,P_T^\circ)$ in $\as(P_S,P_T^\circ)$ gives the lower bound.
Finally, $\tilde\rho=\rho_--\rho_-^2/4\leq\rho_-$.
\end{proof}

\section{Radius-indexed bridge admissibility and interior minimizers of a finite-sample bound}
\label{sec:radius-indexed-selection}

We first describe the generic admissible set. For the symmetric one-parameter
bridges, before imposing target feasibility, $\rho_S$ and $\rho_T$ constrain
the bridge path from opposite sides. Under target feasibility, the source constraint is satisfied
along the entire symmetric bridge path and therefore does not truncate it,
although it may bind at $\alpha=1$. We show that the target radius can exclude
the source endpoint and that a finite-sample endpoint-slope condition is
sufficient to exclude the importance-weighting endpoint. Under both
conditions, every minimizer of the deterministic risk-bound envelope lies in
$(0,1)$, although it may equal the lower boundary of the admissible interval.

To define the marginal radius map, recall from
Proposition~\ref{prop:structural-bridge-target-locality} the target-side
quantity
\[
    \delta_{B,T}(\bm\theta)
    =D_{\phi_T}(P_{\bm\theta}^{B,X}\Vert P_T^X),
\]
and introduce its source-side counterpart by
\begin{equation}\label{eqn:source-side-bridge-radius}
    \delta_{B,S}(\bm\theta)
    =
    D_{\phi_S}(P_{\bm\theta}^{B,X}\Vert P_S^X).
\end{equation}
By \eqref{eqn:structural-marginal-divergence} and \eqref{eqn:structural-marginal-source-divergence}, the completed bridge law
$P_{\bm\theta}^B(\di x, \di y)=P_{\bm\theta}^{B,X}(\di x) P_S^{Y\mid X}(\di y \mid x)$ satisfies
\[
    D_{\phi_S}(P_{\bm\theta}^B\Vert P_S)=\delta_{B,S}(\bm\theta),
    \qquad
    D_{\phi_T}(P_{\bm\theta}^B\Vert P_T^\circ)=\delta_{B,T}(\bm\theta),
\]
so that
\begin{equation}\label{eqn:bridge-in-ambiguity-set}
    P_{\bm\theta}^B\in\as(P_S,P_T^\circ)
    \iff
    \delta_{B,S}(\bm\theta)\leq\rho_S
    \text{ and }
    \delta_{B,T}(\bm\theta)\leq\rho_T .
\end{equation}
Both quantities depend only on the two covariate marginals $P_S^X$ and $P_T^X$.
We denote the admissible bridge parameters by
\begin{equation}\label{eqn:admissible-bridge-set}
    \widetilde\Theta_B(\rho_S,\rho_T)
    =
    \left\{
      \bm\theta\in\Theta_B:
      \delta_{B,S}(\bm\theta)\leq\rho_S,\
      \delta_{B,T}(\bm\theta)\leq\rho_T
    \right\}.
\end{equation}
By \eqref{eqn:completed-marginal-divergences}, if some
$\bm\theta\in\Theta_B$ satisfies both inequalities in
\eqref{eqn:admissible-bridge-set} strictly, then the density
$r_0=\di P_{\bm\theta}^B/\di P_S
=w_{\bm\theta}^B/\E_{P_S}[w_{\bm\theta}^B]$ is normalized and satisfies the
two strict divergence inequalities. Assumption~\ref{ass:common-domination}
gives $\LR>0$, $P_S$-almost surely, and the bridge maps are positive on
$(0,\infty)$, so $r_0>0$, $P_S$-almost surely. Thus $r_0$ satisfies the
Slater condition~\eqref{eqn:dual-slater-condition} required by
Theorem~\ref{thm:general-dual-representation}. The marginal radius map can
therefore be used both to check bridge admissibility and to construct a Slater
point for the dual representation.

For the symmetric one-parameter bridges, the map $\alpha\mapsto(\delta_{B,S},\delta_{B,T})$ is available in closed form and is monotone.

\begin{proposition}[Symmetric-bridge radius profiles] \label{prop:radius-profile}
For $B\in\{\mathrm{KL},\mathrm{HD}\}$, take $\Theta_B=[0,1]$.
Suppose Assumption~\ref{ass:common-domination} holds. Let $\mathcal X_0\subseteq\xspace$ be compact with $P_S^X(\mathcal X_0)=1$, and suppose there are constants $r_-$ and $r_+$ such that
\[
    0<r_-\leq\LR(x)\leq r_+<\infty,
    \qquad x\in\mathcal X_0.
\]
Suppose also that $\LR$ is not $P_S^X$-almost surely constant. Then:
\begin{enumerate}
\item \textbf{KL/KL:}
Let $P_\alpha^{\mathrm{KL},X}$ denote the KL/KL bridge marginal, and let
$\Lambda(\alpha)=\log\E_{P_S^X}[\LR^\alpha]$.
Then, $\Lambda$ is finite and strictly convex on $[0,1]$ with $\Lambda(0)=\Lambda(1)=0$, and
\begin{equation}\label{eqn:klkl-radius-profile}
    \delta_{\mathrm{KL},S}(\alpha)
    =\alpha\Lambda'(\alpha)-\Lambda(\alpha),
    \qquad
    \delta_{\mathrm{KL},T}(\alpha)
    =(\alpha-1)\Lambda'(\alpha)-\Lambda(\alpha),
\end{equation}
with
$\delta_{\mathrm{KL},S}'(\alpha)=\alpha\Lambda''(\alpha)\geq0$ and
$\delta_{\mathrm{KL},T}'(\alpha)=(\alpha-1)\Lambda''(\alpha)\leq0$,
where $\Lambda''(\alpha)=\Var_{P_\alpha^{\mathrm{KL},X}}(\log\LR)>0$.
In particular $\delta_{\mathrm{KL},S}$ is strictly increasing on $(0,1]$ from $0$ to $D_{\mathrm{KL}}(P_T^X\Vert P_S^X)$, and $\delta_{\mathrm{KL},T}$ is strictly decreasing on $[0,1)$ from $D_{\mathrm{KL}}(P_S^X\Vert P_T^X)$ to $0$.
\item \textbf{HD/HD:}
Let $\delta_H=D_{\phi_{\mathrm{HD}}}(P_S^X\Vert P_T^X)\in(0,2)$ and $Z(\alpha)=1-\alpha(1-\alpha)\delta_H$.
Then
\begin{equation}\label{eqn:hdhd-radius-profile}
    \delta_{\mathrm{HD},S}(\alpha)
    =2-\frac{2-\alpha\delta_H}{\sqrt{Z(\alpha)}},
    \qquad
    \delta_{\mathrm{HD},T}(\alpha)
    =2-\frac{2-(1-\alpha)\delta_H}{\sqrt{Z(\alpha)}},
\end{equation}
with $\delta_{\mathrm{HD},S}$ strictly increasing on $[0,1]$ from $0$ to $\delta_H$, and $\delta_{\mathrm{HD},T}$ strictly decreasing on $[0,1]$ from $\delta_H$ to $0$.
\end{enumerate}

Therefore, for $B\in\{\mathrm{KL},\mathrm{HD}\}$,
\begin{equation}\label{eqn:admissible-interval}
    \widetilde\Theta_B(\rho_S,\rho_T)
    =[\alpha_-(\rho_T),\alpha_+(\rho_S)]\cap[0,1],
\end{equation}
where $\alpha_+(\rho_S)=\sup\{\alpha\in[0,1]:\delta_{B,S}(\alpha)\leq\rho_S\}$ and $\alpha_-(\rho_T)=\inf\{\alpha\in[0,1]:\delta_{B,T}(\alpha)\leq\rho_T\}$.
The set is non-empty if and only if $\alpha_-(\rho_T)\leq\alpha_+(\rho_S)$, a condition determined by the two covariate marginals and the chosen radii.
Moreover
\[
\begin{aligned}
    \alpha_-(\rho_T)>0
    &\iff
    \rho_T<D_{\phi_T}(P_S^X\Vert P_T^X),
    \\
    \alpha_-(\rho_T)<1
    &\iff
    \rho_T>0,
    \\
    \alpha_+(\rho_S)<1
    &\iff
    \rho_S<D_{\phi_S}(P_T^X\Vert P_S^X).
\end{aligned}
\]

If Assumption~\ref{ass:target-feasibility} also holds, then
\begin{equation}\label{eqn:source-radius-inactive-under-feasibility}
    D_{\phi_S}(P_T^X\Vert P_S^X)
    \leq D_{\phi_S}(P_T\Vert P_S)
    \leq\rho_S,
    \qquad
    \alpha_+(\rho_S)=1,
\end{equation}
and hence
\begin{equation}\label{eqn:feasible-admissible-interval}
    \widetilde\Theta_B(\rho_S,\rho_T)
    =[\alpha_-(\rho_T),1].
\end{equation}

Lastly, as $\alpha\uparrow1$,
\begin{equation*}
    \delta_{B,T}(\alpha)
    =
    \varsigma_B(1-\alpha)^{2}+O\{(1-\alpha)^{3}\},
\end{equation*}
with $\varsigma_{\mathrm{KL}}=\tfrac12\Var_{P_T^X}(\log\LR)>0$ and $\varsigma_{\mathrm{HD}}=\delta_H(1-\delta_H/4)>0$.

\end{proposition}

\begin{proof}\textbf{(KL/KL)}
Write $\varphi=\log\LR$, which is bounded due to the assumption $\LR\in[r_-,r_+]$.
Then $\di P_\alpha^{\mathrm{KL},X}/\di P_S^X=\LR^\alpha/\E_{P_S^X}[\LR^\alpha] =\exp\{\alpha\varphi-\Lambda(\alpha)\}$, i.e.\ a one-parameter exponential family in $\alpha$ with cumulant function $\Lambda$.
Therefore,
$\Lambda'(\alpha)=\E_{P_\alpha^{\mathrm{KL},X}}[\varphi]$ and $\Lambda''(\alpha)=\Var_{P_\alpha^{\mathrm{KL},X}}(\varphi) > 0 $ (because $\varphi$ is not a.s.\ constant).
Also $\Lambda(0)=0$ and $\Lambda(1)=\log\E_{P_S^X}[\LR]=0$ by Assumption~\ref{ass:common-domination}.
Therefore
\[
    \delta_{\mathrm{KL},S}(\alpha)
    =\E_{P_\alpha^{\mathrm{KL},X}}\left[\log\frac{\di P_\alpha^{\mathrm{KL},X}}{\di P_S^X}\right]
    =\alpha\Lambda'(\alpha)-\Lambda(\alpha),
\]
and, since $\di P_T^X/\di P_S^X=\LR=e^{\varphi}$,
\[
    \delta_{\mathrm{KL},T}(\alpha)
    =\E_{P_\alpha^{\mathrm{KL},X}}
     \left[\alpha\varphi-\Lambda(\alpha)-\varphi\right]
    =(\alpha-1)\Lambda'(\alpha)-\Lambda(\alpha).
\]
The stated derivatives follow immediately.
The endpoint values follow from $\Lambda(0)=\Lambda(1)=0$, $\Lambda'(0)=\E_{P_S^X}[\varphi] =-D_{\mathrm{KL}}(P_S^X\Vert P_T^X)$ and $\Lambda'(1)=\E_{P_T^X}[\varphi]=D_{\mathrm{KL}}(P_T^X\Vert P_S^X)$.

\textbf{(HD/HD)}
Let $\mu$ be a common dominating measure for the two covariate marginals, whose densities are $p_S^X,p_T^X$.
Because $\LR$ is not almost surely constant, the two marginals are distinct
and $\delta_H>0$. Their equivalence gives
$\int\sqrt{p_S^Xp_T^X}\,\d\mu>0$, and hence $\delta_H<2$.
The unnormalized HD/HD weight is $\{(1-\alpha)+\alpha\sqrt\LR\}^{2}$, so $\sqrt{p_\alpha^{\mathrm{HD},X}}
=\frac{(1-\alpha)\sqrt{p_S^X}+\alpha\sqrt{p_T^X}}{\sqrt{Z(\alpha)}}$ with
\[
\begin{aligned}
    Z(\alpha)
    &=\int\left\{(1-\alpha)\sqrt{p_S^X}+\alpha\sqrt{p_T^X}\right\}^{2}\d\mu
    =(1-\alpha)^{2}+\alpha^{2}
     +2\alpha(1-\alpha)\left(1-\frac{\delta_H}{2}\right) \\
    &=1-\alpha(1-\alpha)\delta_H,
\end{aligned}
\]
where we used $\int\sqrt{p_S^Xp_T^X}\,\d\mu=1-\delta_H/2$.
Therefore,
\[
    \int\sqrt{p_\alpha^{\mathrm{HD},X}p_T^X}\,\d\mu
    =\frac{(1-\alpha)(1-\delta_H/2)+\alpha}{\sqrt{Z(\alpha)}}
    =\frac{1-(1-\alpha)\delta_H/2}{\sqrt{Z(\alpha)}}.
\]
The identity $\delta_{\mathrm{HD},T}(\alpha)=2-2\int\sqrt{p_\alpha^{\mathrm{HD},X}p_T^X}\d\mu$ gives the target-side equation in \eqref{eqn:hdhd-radius-profile}.
The source-side follows by switching the roles of the two marginals, which corresponds to $\alpha\mapsto1-\alpha$ but leaves $Z$ invariant.

To show monotonicity, let $k=\delta_H/2\in(0,1)$, $u(\alpha)=1-(1-\alpha)k$ and $G=u/\sqrt Z$, so that $\delta_{\mathrm{HD},T}=2-2G$.
Then $u'=k$, $Z'=-2(1-2\alpha)k$,
\[
    G'=\frac{u'Z-\tfrac12uZ'}{Z^{3/2}},
\]
and
\[
    u'Z-\tfrac12uZ'
    =k\left\{1-2\alpha(1-\alpha)k
      +(1-2\alpha)\left[1-(1-\alpha)k\right]\right\}
    =k(1-\alpha)(2-k).
\]
Since $k\in(0,1)$, this is strictly positive for $\alpha<1$, hence $G$ is strictly increasing and $\delta_{\mathrm{HD},T}$ is strictly decreasing on $[0,1]$.
The source-side statement for $\delta_{\mathrm{HD},S}$ follows again from the substitution $\alpha\mapsto1-\alpha$.

\textbf{(Admissible set)}
From the previous parts for $B \in \{\mathrm{KL}, \mathrm{HD}\}$, $\delta_{B,S}$ is continuous and strictly increasing and $\delta_{B,T}$ is continuous and strictly decreasing on $[0,1]$.
Therefore, the two sublevel sets are intervals of the stated form, and their intersection is non-empty if and only if $\alpha_-\leq\alpha_+$.
The first and third equivalences follow from strict monotonicity combined with $\delta_{B,S}(1)=D_{\phi_S}(P_T^X\Vert P_S^X)$ and $\delta_{B,T}(0)=D_{\phi_T}(P_S^X\Vert P_T^X)$. The second follows from continuity, strict decrease, and $\delta_{B,T}(1)=0$.
Under Assumption~\ref{ass:target-feasibility}, the data-processing inequality for the projection $(x,y)\mapsto x$ gives
\[
    D_{\phi_S}(P_T^X\Vert P_S^X)
    \leq D_{\phi_S}(P_T\Vert P_S)
    \leq\rho_S.
\]
Since $\delta_{B,S}$ is increasing and has endpoint value $\delta_{B,S}(1)=D_{\phi_S}(P_T^X\Vert P_S^X)$, the source constraint holds along the entire path. This proves~\eqref{eqn:source-radius-inactive-under-feasibility} and~\eqref{eqn:feasible-admissible-interval}.

\textbf{(Limit)}
Let $\varepsilon=1-\alpha$.
For KL/KL, boundedness of $\varphi$ makes $\Lambda$ smooth on a neighborhood
of $[0,1]$. Taylor expansion at $\alpha=1$ gives
\[
\begin{aligned}
\Lambda(1) & =0,\\
\Lambda(1-\varepsilon) &=-\varepsilon\Lambda'(1) +\tfrac12\varepsilon^{2}\Lambda''(1)+O(\varepsilon^{3}),\\
\Lambda'(1-\varepsilon) & =\Lambda'(1)-\varepsilon\Lambda''(1)+O(\varepsilon^{2}),
\end{aligned}
\]
and therefore
\[
    \delta_{\mathrm{KL},T}(1-\varepsilon)
    =-\varepsilon\Lambda'(1-\varepsilon)-\Lambda(1-\varepsilon)
    =\tfrac12\varepsilon^{2}\Lambda''(1)+O(\varepsilon^{3}),
\]
and $\Lambda''(1)=\Var_{P_T^X}(\log\LR)$.

For HD/HD, with
\[
u(1-\varepsilon)=1-\varepsilon k,\quad\text{and}\quad
Z(1-\varepsilon)=1-2\varepsilon k+2\varepsilon^{2}k,
\]
a second-order expansion gives $u/\sqrt Z=1-\varepsilon^{2}k(1-k/2)+O(\varepsilon^{3})$, hence
\[
\delta_{\mathrm{HD},T}=2\varepsilon^{2}k(1-k/2)+O(\varepsilon^{3})
=\varepsilon^{2}\delta_H(1-\delta_H/4)+O(\varepsilon^{3}).
\]
\end{proof}

\paragraph{Covering the structural bridge classes.}

Let \(I=[r_-,r_+]\) with \(0<r_-<1<r_+<\infty\), assume
\(\LR(x)\in I\) for every \(x\in\mathcal X_B\), put \(R=r_+/r_-\), and let
\(\bar\alpha_{\mathrm{KL}}=\bar\alpha_{\mathrm{HD}}=1\), while
\(\bar\alpha_{\mathrm{KH}}<1\).

\begin{lemma}[Explicit entropy of the bridge classes]
\label{lem:supp-bridge-entropy}
For \(B\in\{\mathrm{KL},\mathrm{HD},\mathrm{KH}\}\), define
\[
u_\alpha^B(x)
=
\frac{g_\alpha^B\{\LR(x)\}}
     {g_\alpha^B(r_+)},
\qquad
0\le\alpha\le\bar\alpha_B.
\]
Then
\[
\sup_x|\partial_\alpha u_\alpha^B(x)|
\le A_B,
\]
where
\[
A_{\mathrm{KL}}=\log R,\qquad
A_{\mathrm{HD}}=2(\sqrt{r_+}-\sqrt{r_-}),\qquad
A_{\mathrm{KH}}
=
\frac{2(\sqrt{r_+}-\sqrt{r_-})}
     {(1-\bar\alpha_{\mathrm{KH}})^2}.
\]
Consequently,
\[
\mathcal N_\infty(\epsilon,\mathcal F_{0,B},N)
\le
1+\frac{\bar\alpha_BA_B}{2\epsilon},
\qquad 0<\epsilon\le1.
\]
If the normalized loss class is VC-type with characteristics
\((A_H,V)\), then the weighted class is VC-type with entropy exponent
\(V+1\).
\end{lemma}

\begin{proof}
For KL/KL,
\(u_\alpha(u)=(u/r_+)^\alpha\), and
\[
|\partial_\alpha u_\alpha(u)|
=
u_\alpha(u)|\log(u/r_+)|
\le\log R.
\]
For HD/HD, with
\(N_\alpha(u)=1-\alpha+\alpha\sqrt u\) and
\(D_\alpha=N_\alpha(r_+)\),
\[
\partial_\alpha u_\alpha(u)
=
\frac{2N_\alpha(u)(\sqrt u-\sqrt{r_+})}{D_\alpha^3},
\]
whose absolute value is at most
\(2(\sqrt{r_+}-\sqrt{r_-})\).
For KL/HD, write \(a_\alpha=\alpha/(1-\alpha)\) and
\(f(w)=w/(1+w)\).  Differentiation of
\(
u_\alpha(u)
=
\exp\{2W_0(a_\alpha\sqrt u)-2W_0(a_\alpha\sqrt{r_+})\}
\)
gives
\[
|\partial_\alpha u_\alpha(u)|
\le
\frac{2(\sqrt{r_+}-\sqrt{r_-})}{(1-\alpha)^2}.
\]
A grid in \(\alpha\) with mesh \(2\epsilon/A_B\) gives the covering bound.
The product-class statement follows from
\(|u\ell-u'\ell'|\le|u-u'|+|\ell-\ell'|\).
\end{proof}

For the two symmetric families used in the interiority result, the decreasing
bridge-to-target-reference component of the target-locality bound is
accompanied by growth in the ratios $C_B/c_B$ and $L_B/c_B$ entering the
statistical-complexity term. The next lemma quantifies these ratios and also
records corresponding constants for the KL/HD bridge.

\begin{lemma}[Bridge constants for bounded ratios] \label{lem:explicit-bridge-constants}
Let $I=[r_-,r_+]$ with $r_-,r_+$ as in
Proposition~\ref{prop:radius-profile}, and define $R=r_+/r_-$. Since
$\E_{P_S^X}[\LR]=1$ and $\LR$ is not almost surely constant, every such
bounding interval satisfies $r_-<1<r_+$, and hence $R>1$. For
$\alpha\in[0,1]$ in the two symmetric families and $\alpha\in[0,1)$ for
KL/HD, the following are valid choices:
\begin{align*}
\text{\textbf{KL/KL:}}\quad
&c_{\mathrm{KL}}(\alpha)=r_-^\alpha,\quad
 C_{\mathrm{KL}}(\alpha)=r_+^\alpha,\quad
 L_{\mathrm{KL}}(\alpha)=\alpha r_-^{\alpha-1},
\\
\text{\textbf{HD/HD:}}\quad
&c_{\mathrm{HD}}(\alpha)=\{(1-\alpha)+\alpha\sqrt{r_-}\}^{2},\quad
 C_{\mathrm{HD}}(\alpha)=\{(1-\alpha)+\alpha\sqrt{r_+}\}^{2},\\
 &L_{\mathrm{HD}}(\alpha) =\alpha\left(\frac{1-\alpha}{\sqrt{r_-}}+\alpha\right),
\\
\text{\textbf{KL/HD:}}\quad
&c_{\mathrm{KH}}(\alpha)=e^{2W_-},\quad
 C_{\mathrm{KH}}(\alpha)=e^{2W_+},
\\
&L_{\mathrm{KH}}(\alpha) =
    \frac{C_{\mathrm{KH}}(\alpha)}{r_-}
    \frac{W_+}{1+W_+}, \\
&\text{with }
 W_\pm=W_0\left(\frac{\alpha}{1 - \alpha}\sqrt{r_\pm}\right).
\end{align*}

Here $c_B(\alpha)$ and $C_B(\alpha)$ are the bridge-map bounds used in
Assumption~\ref{ass:bridge-weight-regularity}, and $L_B(\alpha)$ is a valid
Lipschitz constant for
Proposition~\ref{prop:bridge-weight-error-density-estimation}.
On any KL/HD parameter interval $[0,\alpha_{\mathrm{KH},\max}]$ with
$\alpha_{\mathrm{KH},\max}<1$, these
pointwise choices also yield the uniform finite constants required in
Assumption~\ref{ass:bridge-weight-regularity}. For the mean-value step in
Proposition~\ref{prop:bridge-weight-error-density-estimation}, the displayed
Lipschitz constants apply on any event on which both $\LR$ and $\widehat\LR$
take values in $I$. In Theorem~\ref{thm:buffered-radius-indexed-bound}, this
is ensured on the density-estimation event by the buffered likelihood-ratio
condition together with $e_{n,m}(\delta/2)\leq\zeta$; no clipping of the
plug-in estimator is required.

Moreover, for the two symmetric families the spread $\mathfrak s_B(\alpha) = C_B(\alpha)/c_B(\alpha)$ is strictly increasing in $\alpha$ with $\mathfrak s_B(0) = 1$ and $\mathfrak s_B(1)=R$.
For the KL/HD bridge, the spread $\mathfrak s_{\mathrm{KH}}(\alpha)=C_{\mathrm{KH}}(\alpha)/c_{\mathrm{KH}}(\alpha)$ is strictly increasing in $\alpha$, with $\mathfrak s_{\mathrm{KH}}(0)=1$ and $\lim_{\alpha \uparrow 1}\mathfrak s_{\mathrm{KH}}(\alpha)=R$.
For the two symmetric families the plug-in sensitivity simplifies to
\begin{equation}\label{eqn:plugin-sensitivity-ratios}
    \frac{L_{\mathrm{KL}}(\alpha)}{c_{\mathrm{KL}}(\alpha)}
    =\frac{\alpha}{r_-},
    \quad
    \frac{L_{\mathrm{HD}}(\alpha)}{c_{\mathrm{HD}}(\alpha)}
    =\frac{\alpha}
          {\sqrt{r_-}(1-\alpha)+\alpha r_-},
\end{equation}
and they are strictly increasing from $0$ at $\alpha=0$ to $1/r_-$ at $\alpha=1$.
\end{lemma}

\begin{proof}
All three maps are strictly increasing in $u$ on $I$ for $\alpha>0$.
For KL/KL and HD/HD this is immediate, and for KL/HD it follows from~\eqref{eqn:appendix-kh-derivative-u} combined with positivity of $W_0$ on $(0,\infty)$.
For $\alpha>0$, the minimum and maximum are attained at $u=r_-$ and
$u=r_+$, respectively; at $\alpha=0$, all three maps are constant. This gives
$c_B$ and $C_B$.
For KL/HD we use the representation
$g_{\alpha}^{\mathrm{KH}}(u)=\exp[2W_0\{\alpha\sqrt u/(1-\alpha)\}]$
from the proof of Lemma~\ref{lem:bridge-maps-lipschitz}.

For the Lipschitz constants, $\partial_ug_\alpha^{\mathrm{KL}} =\alpha u^{\alpha-1}$ is decreasing in $u$ for $\alpha\leq1$, so its supremum on $I$ is $\alpha r_-^{\alpha-1}$.
For HD/HD, $\partial_ug_\alpha^{\mathrm{HD}}=\alpha\{(1-\alpha)u^{-1/2}+\alpha\}$ is decreasing in $u$, with supremum
$\alpha\{(1-\alpha)r_-^{-1/2}+\alpha\}$.
Finally, for KL/HD,~\eqref{eqn:appendix-kh-derivative-u} gives
$\partial_ug_\alpha^{\mathrm{KH}}=(g_\alpha^{\mathrm{KH}}/u)W/(1+W)
\leq\{C_{\mathrm{KH}}(\alpha)/r_-\}W_+/(1+W_+)$.

For the spread of KL/KL, $C_{\mathrm{KL}}/c_{\mathrm{KL}}=R^\alpha$ is strictly increasing from $1$ to $R$.
For HD/HD, we write $a=\sqrt{r_+}-1\geq0\geq b=\sqrt{r_-}-1$, and
\[
    \frac{C_{\mathrm{HD}}(\alpha)}{c_{\mathrm{HD}}(\alpha)}
    =\left\{\frac{1+a\alpha}{1+b\alpha}\right\}^{2},
    \qquad
    \frac{\di}{\di\alpha}\frac{1+a\alpha}{1+b\alpha}
    =\frac{a-b}{(1+b\alpha)^{2}}>0.
\]
Therefore, the spread increases strictly from $1$ to $(\sqrt{r_+}/\sqrt{r_-})^{2}=R$.
For KL/HD,
\[
    \frac{C_{\mathrm{KH}}(\alpha)}{c_{\mathrm{KH}}(\alpha)}
    =\exp\{2(W_+-W_-)\}.
\]
For $\alpha\in(0,1)$, let
$\xi(\alpha)=\frac{\alpha}{1 - \alpha}\sqrt{r_-}$. The identity
$W_0'(t)=W_0(t)/[t\{1+W_0(t)\}]$ \citep{Corless1996} gives
\[
    \frac{\di}{\di\xi(\alpha)}\{W_0(\xi(\alpha)\sqrt R)-W_0(\xi(\alpha))\}
    =\frac1{\xi(\alpha)}
     \left\{
       \frac{W_+}{1+W_+}-\frac{W_-}{1+W_-}
     \right\}>0,
\]
because both $W_0$ and $t\mapsto t/(1+t)$ are increasing.
Since $\xi(\alpha)$ is strictly increasing in $\alpha$, the spread is strictly increasing in $\alpha$, equals $1$ at $\alpha=0$, and tends to $R$ as $\alpha\uparrow1$, using $W_0(t)=\log t-\log\log t+o(1)$.

The identities in~\eqref{eqn:plugin-sensitivity-ratios} follow immediately
from the factors established in this proof. The KL/KL ratio is plainly
strictly increasing, while
\[
\frac{\di}{\di\alpha}
\frac{\alpha}{\sqrt{r_-}(1-\alpha)+\alpha r_-}
=
\frac{\sqrt{r_-}}
{\{\sqrt{r_-}(1-\alpha)+\alpha r_-\}^{2}}
>0.
\]
This also proves the asserted strict increase for HD/HD.
\end{proof}

We now combine Proposition~\ref{prop:radius-profile} and Lemma~\ref{lem:explicit-bridge-constants} to provide a generalization bound for $R_T(\widehat h)$ for the two symmetric bridges with the radii explicit.
We further give conditions under which every minimizer of the deterministic bound lies in $(0,1)$. Such a minimizer may lie at the lower boundary of the admissible interval.
Write $\psi^{\mathrm{KL}}_T(x)=\sqrt{x/2}$ and $\psi^{\mathrm{HD}}_T(x)=\sqrt x$.
Then Proposition~\ref{prop:structural-bridge-target-locality} says $\Delta_B(\bm\theta)\leq M\{\psi^B_T(\delta_{B,T}(\bm\theta))+\psi^B_T(\delta_T^\circ)\}$.
For $B\in\{\mathrm{KL},\mathrm{HD}\}$ define the bound function
\begin{equation}\label{eqn:bound-function}
    \mathfrak B_B(\alpha;\delta)
    = 2M \psi^B_T\{\delta_{B,T}(\alpha)\}
     + 2M\psi^B_T(\delta_T^\circ)
     + 2 \mathfrak c_B(\alpha;\delta),
\end{equation}
where $\mathfrak c_B(\alpha;\delta)$ is a deterministic upper bound on $\mathfrak C_{n,m}^B(\alpha;\delta)$ in Proposition~\ref{prop:vc-structural-bridge-complexity}. For a known likelihood ratio, \eqref{eqn:deterministic-structural-bridge-complexity} gives
$\mathfrak c_B=M(C_B/c_B)\lambda_{B,n}$ with zero plug-in error.

\begin{theorem}[Radius-admissible bound and endpoint exclusion]
\label{thm:radius-indexed-interior-optimum}
\leavevmode\par\noindent
Let $B\in\{\mathrm{KL},\mathrm{HD}\}$.  Suppose
Assumptions~\ref{ass:bounded-loss}--\ref{ass:bridge-weight-regularity} hold
with $\Theta_B=[0,1]$, that $\LR$ satisfies the bounded-overlap condition of
Proposition~\ref{prop:radius-profile}, and that the conditions of
Proposition~\ref{prop:vc-structural-bridge-complexity}, including $n\geq16$,
hold.  Let $\widehat\alpha$ be any measurable $[0,1]$-valued parameter and let
$\widehat h$ satisfy the approximate minimization condition in Theorem~\ref{thm:empirical-structural-bridge-generalization} at
$\widehat\alpha$.  On the common event underlying
Corollary~\ref{cor:parameter-indexed-structural-bridge-generalization},
\begin{equation}\label{eqn:radius-indexed-target-risk}
R_T(\widehat h)
\leq
\inf_{h\in\hs}R_T(h)
+\mathfrak B_B(\widehat\alpha;\delta)+\xi_n.
\end{equation}
If, in addition, $\widehat\alpha\in\widetilde\Theta_B(\rho_S,\rho_T)$ almost
surely, then on the same event
\begin{equation}\label{eqn:radius-indexed-target-risk-admissible}
R_T(\widehat h)
\leq
\inf_{h\in\hs}R_T(h)
+2M\psi_T^B(\rho_T)+2M\psi_T^B(\delta_T^\circ)
+2\mathfrak c_B(\widehat\alpha;\delta)+\xi_n.
\end{equation}
If
\begin{equation}\label{eqn:general-source-exclusion}
0<\rho_T<D_{\phi_T}(P_S^X\Vert P_T^X),
\end{equation}
then $\alpha_-(\rho_T)\in(0,1)$.  Put
\[
\tau_{\mathrm{KL}}
=\frac12\sqrt{\Var_{P_T^X}(\log\LR)},
\qquad
\tau_{\mathrm{HD}}
=\sqrt{\delta_H(1-\delta_H/4)}.
\]
If $\rho_T>0$, $\mathfrak c_B$ is left differentiable at $1$, and
\begin{equation}\label{eqn:general-slope-exclusion}
\mathfrak c'_{B,-}(1;\delta)>M\tau_B,
\end{equation}
then $\alpha=1$ does not minimize
$\mathfrak B_B(\cdot;\delta)$ over the admissible interval.  Under both
\eqref{eqn:general-source-exclusion} and
\eqref{eqn:general-slope-exclusion}, every minimizer belongs to
$[\alpha_-(\rho_T),1)\subset(0,1)$.
\end{theorem}

\begin{proof}
On the common event,~Corollary~\ref{cor:parameter-indexed-structural-bridge-generalization}, the bound
$\Delta_B(\widehat\alpha)
\leq M[\psi_T^B\{\delta_{B,T}(\widehat\alpha)\}+\psi_T^B(\delta_T^\circ)]$
from Proposition~\ref{prop:structural-bridge-target-locality}, and
$\mathfrak C_{n,m}^B\leq\mathfrak c_B$
give~\eqref{eqn:radius-indexed-target-risk}; admissibility is not used.  If
$\widehat\alpha$ is admissible, then $\delta_{B,T}(\widehat\alpha)\leq\rho_T$,
and monotonicity of $\psi_T^B$
gives~\eqref{eqn:radius-indexed-target-risk-admissible}.  The first endpoint
statement follows from the strict monotonicity in
Proposition~\ref{prop:radius-profile}.  By the endpoint expansion in Proposition~\ref{prop:radius-profile},
$\alpha\mapsto2M\psi_T^B\{\delta_{B,T}(\alpha)\}$ has left derivative
$-2M\tau_B$ at $1$, so
\[
\mathfrak B'_{B,-}(1;\delta)
=-2M\tau_B+2\mathfrak c'_{B,-}(1;\delta)>0.
\]
Hence
\[
\mathfrak B_B(1;\delta)-\mathfrak B_B(\alpha;\delta)
=(1-\alpha)\{\mathfrak B'_{B,-}(1;\delta)+o(1)\}>0
\]
for admissible $\alpha<1$ sufficiently close to $1$; such $\alpha$ exist
because $\rho_T>0$ and $\delta_{B,T}$ is continuous with
$\delta_{B,T}(1)=0$.  This excludes the target endpoint.
\end{proof}

The two endpoint conclusions concern minimization of the deterministic envelope over feasible structural centers. The source center is infeasible when the target radius is smaller than $D_{\phi_T}(P_S^X\Vert P_T^X)$. The target center is feasible, and the left-slope condition $\mathfrak c'_{B,-}(1;\delta)>M\tau_B$ gives a nearby feasible center with a smaller envelope. Section~\ref{sec:bounded-overlap-simulation} evaluates this condition.

For the explicit specialization below, the buffered likelihood-ratio and density-estimation conditions
in Theorem~\ref{thm:buffered-radius-indexed-bound} give, by Propositions~\ref{prop:vc-structural-bridge-complexity} and~\ref{prop:bridge-weight-error-density-estimation}, the deterministic complexity term
\begin{equation}\label{eqn:alpha-indexed-complexity}
    \mathfrak c_B(\alpha;\delta)
    =
    M\frac{C_B(\alpha)}{c_B(\alpha)}
    \lambda_{B,n}(\delta)
    +
    2M\frac{L_B(\alpha)}{c_B(\alpha)}
    e_{n,m}(\delta/2).
\end{equation}
Here
\[
e_{n,m}(\delta_0)=2a_{T,m}(\delta_0)/c_S+2c_Ta_{S,n}(\delta_0)/c_S^{2}
\]
is the likelihood-ratio error in~\eqref{eqn:lr-high-probability-bound}.

\begin{theorem}[Explicit bounds for buffered likelihood-ratio estimates]
\label{thm:buffered-radius-indexed-bound}
\leavevmode\par\noindent
Fix $\delta\in(0,1)$ and let $B\in\{\mathrm{KL},\mathrm{HD}\}$. Suppose Assumptions~\ref{ass:bounded-loss}--\ref{ass:target-feasibility} and the conditions in Proposition~\ref{prop:radius-profile} hold. Take $\Theta_B=[0,1]$ and $\mathcal X_B=\mathcal X_0$, suppose $n\geq16$, and assume that the normalized bridge and loss classes satisfy the measurability and covering conditions of Proposition~\ref{prop:vc-structural-bridge-complexity}.

Suppose further that the density bounds~\eqref{eqn:density-bounds-section4} and the high-probability density-estimation condition~\eqref{eqn:density-high-probability-event} hold on $\mathcal X_0$ at $\delta_0=\delta/2$, with $a_{S,n}(\delta/2)\leq c_S/2$. Assume that the likelihood-ratio bounds in Proposition~\ref{prop:radius-profile} can be chosen so that, for some $\zeta>0$,
\begin{equation}\label{eqn:buffered-likelihood-ratio-range}
    r_-+\zeta\leq\LR(x)\leq r_+-\zeta,
    \qquad x\in\mathcal X_0,
\end{equation}
and that
\begin{equation}\label{eqn:appendix-density-error-condition}
    e_{n,m}(\delta/2)\leq\min\{\zeta,r_-/2\}.
\end{equation}
Let $\widehat\LR$ be the likelihood-ratio estimator formed from these density
estimators as in Section~\ref{sec:nonparametric-density-estimation}, and take
the estimated bridge weights to be
\[
    \widehat w_\alpha^B(x)=g_\alpha^B\{\widehat\LR(x)\}.
\]
Assume that this estimated-weight family is measurable as required by
Assumption~\ref{ass:bridge-weight-regularity}.
Because the bounds in Proposition~\ref{prop:radius-profile} need not be sharp,
they may be chosen with slack; thus the buffer in
\eqref{eqn:buffered-likelihood-ratio-range} does not strengthen that
proposition's bounded-overlap condition. The selected $r_-$ and $r_+$ are held
fixed when defining and comparing the deterministic envelope. Looser choices
remain valid but yield looser constants and may change the sufficient
endpoint-slope condition.
Let $\widehat\alpha$ be any measurable, possibly data-dependent parameter satisfying $\widehat\alpha\in\widetilde\Theta_B(\rho_S,\rho_T)$ almost surely, and let $\widehat h$ be as in Theorem~\ref{thm:empirical-structural-bridge-generalization}.
Then the following hold:
\begin{enumerate}
\item With probability at least $1-\delta$,
\begin{align*}
    R_T(\widehat h)
    &\leq
    \inf_{h\in\hs}R_T(h)
    +\mathfrak B_B(\widehat\alpha;\delta)
    +\xi_n
    \\
    &\leq
    \inf_{h\in\hs}R_T(h)
    +2M\psi^B_T(\rho_T)
    +2M\psi^B_T(\delta_T^\circ)
    +2\sup_{\alpha\in\widetilde\Theta_B(\rho_S,\rho_T)}
      \mathfrak c_B(\alpha;\delta)
    +\xi_n .
\end{align*}
Moreover,
\begin{equation}\label{eqn:admissible-complexity-maximum}
    \sup_{\alpha\in\widetilde\Theta_B(\rho_S,\rho_T)}
      \mathfrak c_B(\alpha;\delta)
    =\mathfrak c_B(1;\delta)
    =M\left\{R\lambda_{B,n}(\delta)
       +\frac{2}{r_-}e_{n,m}(\delta/2)\right\}.
\end{equation}
\item By target feasibility, $\alpha_+(\rho_S)=1$ and
\[
    \widetilde\Theta_B(\rho_S,\rho_T)
    =[\alpha_-(\rho_T),1].
\]
If
\begin{equation}\label{eqn:target-radius-excludes-source-endpoint}
    0<\rho_T<D_{\phi_T}(P_S^X\Vert P_T^X),
\end{equation}
then $\alpha_-(\rho_T)\in(0,1)$.
\item
Let $\varsigma_B$ be as in Proposition~\ref{prop:radius-profile} and write
\[
    \tau_{\mathrm{KL}}
    =\sqrt{\frac{\varsigma_{\mathrm{KL}}}{2}}
    =\frac12\sqrt{\Var_{P_T^X}(\log\LR)},
    \qquad
    \tau_{\mathrm{HD}}
    =\sqrt{\varsigma_{\mathrm{HD}}}
    =\sqrt{\delta_H\left(1-\frac{\delta_H}{4}\right)} .
\]
Then $\alpha\mapsto2M\psi^B_T\{\delta_{B,T}(\alpha)\}$ has left derivative $-2M\tau_B<0$ at $\alpha=1$, and if the complexity term satisfies
\begin{equation}\label{eqn:interior-optimum-condition}
    \mathfrak c'_{B,-}(1;\delta)
    >
    M\tau_B,
\end{equation}
where the prime denotes differentiation with respect to $\alpha$ from the left,
the bounding function $\mathfrak B_B(\cdot;\delta)$ is strictly increasing on a left neighborhood of $\alpha=1$.
For the two symmetric pairings, condition~\eqref{eqn:interior-optimum-condition} is
\begin{align}
    \text{\textbf{KL/KL: }}
    &\lambda_{B,n}(\delta) R\log R
     +\frac{2}{r_-}e_{n,m}(\delta/2)
     >\tau_{\mathrm{KL}},
    \label{eqn:interior-condition-klkl}
    \\
    \text{\textbf{HD/HD: }}
    &\frac{2\lambda_{B,n}(\delta)\sqrt R(\sqrt{r_+}-\sqrt{r_-})}{r_-}+\frac{2}{r_-^{3/2}}e_{n,m}(\delta/2)
      >\tau_{\mathrm{HD}},
    \label{eqn:interior-condition-hdhd}
\end{align}
where $\lambda_{B,n}(\delta)$ is defined in
Proposition~\ref{prop:vc-structural-bridge-complexity}.
If $\rho_T>0$ and~\eqref{eqn:interior-optimum-condition} holds, no minimizer of $\mathfrak B_B(\cdot;\delta)$ over $\widetilde\Theta_B(\rho_S,\rho_T)$ equals $1$. If, in addition,~\eqref{eqn:target-radius-excludes-source-endpoint} holds, then every minimizer belongs to
\[
    [\alpha_-(\rho_T),1)\subset(0,1).
\]
Thus every minimizer is interior relative to the full bridge range $[0,1]$, although it may equal the lower boundary $\alpha_-(\rho_T)$ of the admissible interval.
\end{enumerate}
\end{theorem}

\begin{proof}
\textbf{(1)}
On the density-estimation event~\eqref{eqn:density-high-probability-event}, Proposition~\ref{prop:bridge-weight-error-density-estimation} gives
\[
    \|\widehat\LR-\LR\|_{\infty,\mathcal X_0}
    \leq e_{n,m}(\delta/2)\leq\zeta.
\]
It follows from~\eqref{eqn:buffered-likelihood-ratio-range} that both $\LR$
and $\widehat\LR$ take values in $I=[r_-,r_+]$. Lemma~\ref{lem:explicit-bridge-constants},
together with the same mean-value argument used in the proof of
Proposition~\ref{prop:bridge-weight-error-density-estimation}, therefore shows
that we may take, simultaneously for every $\alpha\in[0,1]$,
\[
    \varepsilon_{n,m}^B(\alpha;\delta/2)
    =L_B(\alpha)e_{n,m}(\delta/2).
\]
Moreover,~\eqref{eqn:plugin-sensitivity-ratios} and~\eqref{eqn:appendix-density-error-condition} imply
\[
    \sup_{\alpha\in[0,1]}
    \frac{\varepsilon_{n,m}^B(\alpha;\delta/2)}{c_B(\alpha)}
    =\frac{e_{n,m}(\delta/2)}{r_-}
    \leq\frac12.
\]
Thus the density event supplies the plug-in component of Assumption~\ref{ass:bridge-weight-regularity} at level $\delta/2$. Combining it with the simultaneous empirical-Bernstein event at level $\delta/2$ gives the event of Proposition~\ref{prop:vc-structural-bridge-complexity} with probability at least $1-\delta$.

On this event, Corollary~\ref{cor:parameter-indexed-structural-bridge-generalization}, Proposition~\ref{prop:structural-bridge-target-locality}, and~\eqref{eqn:deterministic-structural-bridge-complexity} give
\[
    R_T(\widehat h)
    \leq\inf_{h\in\hs}R_T(h)
      +\mathfrak B_B(\widehat\alpha;\delta)+\xi_n.
\]
The event is simultaneous in $\alpha$, so the conclusion remains valid at the data-dependent $\widehat\alpha$. Its admissibility gives $\delta_{B,T}(\widehat\alpha)\leq\rho_T$, and monotonicity of $\psi_T^B$ gives the second inequality in part~(1).

By~\eqref{eqn:feasible-admissible-interval}, $1$ belongs to the admissible interval. Lemma~\ref{lem:explicit-bridge-constants} shows that both $C_B/c_B$ and $L_B/c_B$ are increasing in $\alpha$. Therefore, the supremum of $\mathfrak c_B$ over the admissible interval is attained at $1$, and the endpoint values in that lemma give~\eqref{eqn:admissible-complexity-maximum}.

\textbf{(2)}
The first conclusion is~\eqref{eqn:source-radius-inactive-under-feasibility}--\eqref{eqn:feasible-admissible-interval}. Under~\eqref{eqn:target-radius-excludes-source-endpoint}, Proposition~\ref{prop:radius-profile} gives both $\alpha_-(\rho_T)>0$ and $\alpha_-(\rho_T)<1$.

\textbf{(3)}
By Proposition~\ref{prop:radius-profile},
\[
\delta_{B,T}(\alpha)=\varsigma_B(1-\alpha)^{2}\{1+O(1-\alpha)\}.
\]
Writing $\varepsilon=1-\alpha$, we have
\[
    \psi_T^B\{\delta_{B,T}(1-\varepsilon)\}
    =\varepsilon\tau_B\{1+O(\varepsilon)\}.
\]
Hence $2M\psi^B_T\{\delta_{B,T}(\alpha)\}$ has left derivative $-2M\tau_B$ at $\alpha=1$.
Moreover, the closed-form profiles in
Proposition~\ref{prop:radius-profile} imply that
\[
    \frac{\delta_{B,T}(\alpha)}{(1-\alpha)^2}
\]
extends to a positive continuously differentiable function at $\alpha=1$.
Consequently, the derivative of the square-root locality term extends
continuously to $1$ from the left.
Since $\mathfrak B_B(\alpha;\delta)$ differs from this term by $2\mathfrak c_B(\alpha;\delta)$ and a constant, its left derivative at $\alpha=1$ is $-2M\tau_B+2\mathfrak c'_{B,-}(1;\delta)$, which is $>0$ under condition~\eqref{eqn:interior-optimum-condition}.
The two explicit conditions follow by differentiating the ratios of Lemma~\ref{lem:explicit-bridge-constants} at $\alpha=1$:
\[
\begin{aligned}
  \left.\frac{\di}{\di\alpha} R^\alpha\right|_{\alpha=1} & = R\log R, \\
  \frac{\di}{\di\alpha} \frac{\alpha}{r_-} & =1/r_-, \\
  \left.\frac{\di}{\di\alpha}
  \left\{\frac{(1-\alpha)+\alpha\sqrt{r_+}}
  {(1-\alpha)+\alpha\sqrt{r_-}}\right\}^{2}\right|_{\alpha=1}
  &=2\sqrt R(\sqrt{r_+}-\sqrt{r_-})/r_-, \\
  \left.\frac{\di}{\di\alpha}
  \frac{\alpha}{\sqrt{r_-}(1-\alpha)+\alpha r_-}\right|_{\alpha=1}
  &=\sqrt{r_-}/r_-^{2}=r_-^{-3/2}.
\end{aligned}
\]
The ratios in Lemma~\ref{lem:explicit-bridge-constants} are smooth near
$\alpha=1$, so the derivative of $\mathfrak B_B$ is continuous on a left
neighborhood of $1$. Hence, under~\eqref{eqn:interior-optimum-condition}, there
is an $\varepsilon_0>0$ such that
\[
    \mathfrak B_B(\alpha;\delta)
    <\mathfrak B_B(1;\delta),
    \qquad \alpha\in[1-\varepsilon_0,1).
\]
If $\rho_T>0$, then $\alpha_-(\rho_T)<1$, so this left neighborhood contains admissible parameters and $1$ cannot minimize the envelope over $[\alpha_-(\rho_T),1]$. If~\eqref{eqn:target-radius-excludes-source-endpoint} also holds, then $\alpha_-(\rho_T)>0$. Continuity of $\mathfrak B_B$ on the compact admissible interval gives existence of a minimizer, and every minimizer lies in $[\alpha_-(\rho_T),1)\subset(0,1)$.

\end{proof}

\section{Asymptotic analysis of loss-agnostic bridge weights}
\label{app:asymptotics-loss-agnostic-bridge-weights}

Throughout this appendix, $p_S$ and $p_T$ denote the marginal densities of
the source and target covariates, respectively, and
\[
    \LR(x)=\frac{p_T(x)}{p_S(x)}.
\]
Thus the bridge weights studied below are the canonical unnormalized
covariate-level structural weights introduced in
Section~\ref{sec:loss-agnostic-density-bridges}. We first establish uniform
stability of the bridge maps and their plug-in estimators, and then derive
pointwise central limit theorems for the estimated likelihood ratio and the
resulting bridge weights. Unless otherwise stated, the two-sample limits in
this appendix are taken along sequences for which $n,m\to\infty$.

The uniform argument has two steps. First, uniform consistency of the two
density estimators keeps the estimated source density bounded away from zero
and yields a uniform bound for $\widehat\LR-\LR$. Second, strong overlap places
both likelihood ratios in a fixed compact subset of $(0,\infty)$, on which the
bridge maps are jointly Lipschitz in the likelihood ratio and the bridge
parameter. Combining these two deterministic steps gives uniform weight
consistency, including at an estimated bridge parameter.

\subsection{Uniform consistency}

\begin{assumption}[Working region and density-estimation input]
\label{ass:kde-input}
The set $\cX_0\subseteq\R^d$ is compact. There exist constants
$0<c_S<\infty$ and $0<c_T<\infty$ such that
\[
    p_S(x)\geq c_S,
    \qquad
    0\leq p_T(x)\leq c_T,
    \qquad x\in\cX_0.
\]
There exist deterministic sequences $a_{S,n}\to0$ and $a_{T,m}\to0$ such
that
\[
    \|\widehat p_S-p_S\|_{\infty,\cX_0}=\Op(a_{S,n}),
    \qquad
    \|\widehat p_T-p_T\|_{\infty,\cX_0}=\Op(a_{T,m}),
\]
where
\[
    \|f\|_{\infty,\cX_0}
    =\sup_{x\in\cX_0}|f(x)|.
\]
The estimators $\widehat p_S$ and $\widehat p_T$ are nonnegative and
measurable on $\cX_0$.
\end{assumption}

Throughout the uniform analysis, define
\[
    \widehat\LR(x)
    =
    \begin{cases}
      \widehat p_T(x)/\widehat p_S(x),&\widehat p_S(x)>0,\\
      0,&\widehat p_S(x)=0.
    \end{cases}
\]
The convention in the second case is immaterial on the events below, on
which $\widehat p_S$ is uniformly bounded away from zero.

\begin{assumption}[Strong overlap for uniform bridge-map stability]
\label{ass:strong-overlap}
There exist constants $0<r_-<r_+<\infty$ such that
\[
    r_-\leq\LR(x)\leq r_+,
    \qquad x\in\cX_0.
\]
\end{assumption}

\begin{remark}[Why strong overlap appears here]
Assumption~\ref{ass:strong-overlap} is stronger than what is needed for an
ordinary Nadaraya--Watson estimator. Pointwise Nadaraya--Watson estimation
only requires the source design density to be positive near the evaluation
point. Here the structural bridge weights are nonlinear functions of
$\LR=p_T/p_S$. Uniform stability of the three bridge families is therefore
most naturally obtained when $\LR$ remains in a compact subset of
$(0,\infty)$. For $0<\tau<1$, one may instead define the trimmed population
and empirical ratios
\[
    \LR_\tau(x)
    =\{\LR(x)\vee\tau\}\wedge\tau^{-1},
    \qquad
    \widehat\LR_\tau(x)
    =\{\widehat\LR(x)\vee\tau\}\wedge\tau^{-1}
\]
and analyze the resulting modified trimmed bridge, or restrict the analysis
to a region of adequate source--target overlap. Fixed trimming changes the
estimand whenever $\LR$ lies outside $[\tau,\tau^{-1}]$; recovering the
untrimmed bridge with a varying threshold requires separate control.
\end{remark}

For $B\in\{\mathrm{KL},\mathrm{HD},\mathrm{KH}\}$, write
\[
    w_{\bm\theta}^B(x)
    =g_{\bm\theta}^B\{\LR(x)\},
    \qquad
    \widehat w_{\bm\theta}^B(x)
    =g_{\bm\theta}^B\{\widehat\LR(x)\}.
\]
The parameter sets are
\[
    \Theta_{\mathrm{KL}}=\Theta_{\mathrm{HD}}=[0,1],
    \qquad
    \Theta_{\mathrm{KH}}
    =[0,\alpha_{\mathrm{KH},\max}],
\]
where $0\leq\alpha_{\mathrm{KH},\max}<1$. Thus
$\bm\theta=(\alpha)$ is scalar for each of
the three structural families in this appendix.

\begin{lemma}[Uniform regularity of the bridge maps]
\label{lem:bridge-maps-lipschitz}
For the three structural bridge geometries,
\begin{align*}
    g_\alpha^{\mathrm{KL}}(u)
    &=u^\alpha,
    &&\alpha\in[0,1],
    \\
    g_\alpha^{\mathrm{HD}}(u)
    &=\{(1-\alpha)+\alpha\sqrt u\}^{2},
    &&\alpha\in[0,1],
    \\
    g_{\alpha}^{\mathrm{KH}}(u)
    &=
    \left[
      \frac{\beta_\alpha(u)}
           {W_0\{\beta_\alpha(u)\}}
    \right]^2,
    &&\alpha\in\Theta_{\mathrm{KH}},
\end{align*}
where
\[
    \beta_\alpha(u)
    =\frac{\alpha}{1-\alpha}\sqrt u,
\]
and the KL/HD map is extended continuously at $\alpha=0$ by
\[
    g_{0}^{\mathrm{KH}}(u)=1.
\]
Let $I=[r_-,r_+]\subset(0,\infty)$ be compact. For every bridge family $B$,
there are finite positive constants
$c_{B,I},C_{B,I},L_{B,u}$, and $L_{B,\theta}$ such that
\begin{equation}\label{eqn:appendix-bridge-map-bounds}
    0<c_{B,I}
    \leq g_{\bm\theta}^B(u)
    \leq C_{B,I}<\infty
\end{equation}
for all $u\in I$ and $\bm\theta\in\Theta_B$, and
\begin{equation}\label{eqn:appendix-joint-bridge-lipschitz}
\left|
  g_{\bm\theta}^B(u)-g_{\bm\vartheta}^B(v)
\right|
\leq
L_{B,u}|u-v|
+L_{B,\theta}\|\bm\theta-\bm\vartheta\|
\end{equation}
for all $u,v\in I$ and
$\bm\theta,\bm\vartheta\in\Theta_B$.
\end{lemma}

\begin{proof}
For the KL/KL map,
\[
    \frac{\partial}{\partial u}g_\alpha^{\mathrm{KL}}(u)
    =\alpha u^{\alpha-1},
    \qquad
    \frac{\partial}{\partial\alpha}g_\alpha^{\mathrm{KL}}(u)
    =u^\alpha\log u.
\]
Both derivatives are uniformly bounded on $I\times[0,1]$ because $I$ is
bounded away from zero. For the HD/HD map,
\begin{align*}
    \frac{\partial}{\partial u}g_\alpha^{\mathrm{HD}}(u)
    &=
    \frac{\alpha\{(1-\alpha)+\alpha\sqrt u\}}{\sqrt u},
    \\
    \frac{\partial}{\partial\alpha}g_\alpha^{\mathrm{HD}}(u)
    &=
    2\{(1-\alpha)+\alpha\sqrt u\}(\sqrt u-1).
\end{align*}
These derivatives are likewise uniformly bounded on the corresponding
compact set.

For the KL/HD map, put
\[
    a_\alpha=\frac{\alpha}{1-\alpha},
    \qquad
    x_{\alpha}(u)
    =a_\alpha\sqrt u.
\]
The identity $x/W_0(x)=e^{W_0(x)}$ gives the equivalent representation
\begin{equation*}
    g_{\alpha}^{\mathrm{KH}}(u)
    =
    \exp\left[
      2W_0\{x_{\alpha}(u)\}
    \right].
\end{equation*}
This formula is well defined at $\alpha=0$ because
$x_{0}(u)=0$ and $W_0(0)=0$, and it gives
$g_{0}^{\mathrm{KH}}(u)=1$.
For $\alpha>0$, writing
$W=W_0\{x_{\alpha}(u)\}$ and using
$W_0'(x)=W_0(x)/[x\{1+W_0(x)\}]$, we obtain
\begin{align}
    \frac{\partial}{\partial u}
    g_{\alpha}^{\mathrm{KH}}(u)
    &=
    \frac{g_{\alpha}^{\mathrm{KH}}(u)}{u}
    \frac{W}{1+W},
    \label{eqn:appendix-kh-derivative-u}
    \\
    \frac{\partial}{\partial\alpha}
    g_{\alpha}^{\mathrm{KH}}(u)
    &=
    \frac{2g_{\alpha}^{\mathrm{KH}}(u)W}
         {(1+W)\alpha(1-\alpha)}.
    \label{eqn:appendix-kh-derivative-alpha}
\end{align}
The derivative with respect to $u$ extends continuously to zero at
$\alpha=0$. Since
$W_0(x)\sim x$ as $x\downarrow0$, the derivative in
\eqref{eqn:appendix-kh-derivative-alpha} has the continuous extension
\[
    \left.
    \frac{\partial}{\partial\alpha}
    g_{\alpha}^{\mathrm{KH}}(u)
    \right|_{\alpha=0}
    =2\sqrt u.
\]
Consequently, the extended KL/HD map is continuously differentiable on the
compact set $I\times\Theta_{\mathrm{KH}}$. Its first derivatives are
uniformly bounded there.

All three maps are continuous and strictly positive on their compact domains,
which proves~\eqref{eqn:appendix-bridge-map-bounds}. The joint Lipschitz bound
\eqref{eqn:appendix-joint-bridge-lipschitz} follows from the mean-value
theorem and the uniform derivative bounds.
\end{proof}

The next lemma records the stochastic-order consequence of the
likelihood-ratio bound in
Proposition~\ref{prop:bridge-weight-error-density-estimation}.

\begin{lemma}[Uniform likelihood-ratio consistency]
\label{lem:lr-uniform}
Under Assumption~\ref{ass:kde-input},
\[
    \|\widehat\LR-\LR\|_{\infty,\cX_0}
    =\Op(a_{S,n}+a_{T,m}).
\]
The corresponding deterministic and high-probability bounds are
\eqref{eqn:lr-high-probability-bound} in
Proposition~\ref{prop:bridge-weight-error-density-estimation}.
\end{lemma}

\begin{proof}
Assumption~\ref{ass:kde-input} and
Proposition~\ref{prop:bridge-weight-error-density-estimation} give the result.
\end{proof}

\begin{theorem}[Uniform consistency of bridge weights]
\label{thm:unif-consistency-bridge-weights}
Suppose Assumptions~\ref{ass:kde-input} and~\ref{ass:strong-overlap} hold. For
every bridge family $B$,
\begin{equation}\label{eqn:appendix-uniform-weight-op}
    \sup_{\bm\theta\in\Theta_B}
    \sup_{x\in\cX_0}
    \left|
      \widehat w_{\bm\theta}^B(x)-w_{\bm\theta}^B(x)
    \right|
    =\Op(a_{S,n}+a_{T,m}).
\end{equation}
If the high-probability input
\eqref{eqn:density-high-probability-event} holds at level $\delta_0$,
$a_{S,n}(\delta_0)\leq c_S/2$, and
\[
    e_{n,m}(\delta_0)\leq r_-/2,
\]
then, with probability at least $1-\delta_0$,
\begin{equation}\label{eqn:appendix-uniform-weight-high-probability}
\sup_{\bm\theta\in\Theta_B}
\sup_{x\in\cX_0}
\left|
  \widehat w_{\bm\theta}^B(x)-w_{\bm\theta}^B(x)
\right|
\leq
L_{B,u}^\star e_{n,m}(\delta_0),
\end{equation}
where $L_{B,u}^\star$ is a uniform Lipschitz constant for the bridge map on
\[
    I_\star
    =\left[\frac{r_-}{2},r_++\frac{r_-}{2}\right].
\]
In particular, for every fixed $\bm\theta\in\Theta_B$ and
$x\in\cX_0$,
\[
    \widehat w_{\bm\theta}^B(x)
    \longrightarrow
    w_{\bm\theta}^B(x)
    \qquad\text{in probability.}
\]
\end{theorem}

\begin{proof}
By Lemma~\ref{lem:lr-uniform},
$\|\widehat\LR-\LR\|_{\infty,\cX_0}
=\Op(a_{S,n}+a_{T,m})$. Since the latter rate tends to zero, the event
\[
    \left\{
      \|\widehat\LR-\LR\|_{\infty,\cX_0}\leq r_-/2
    \right\}
\]
has probability tending to one. On this event, both $\LR(x)$ and
$\widehat\LR(x)$ belong to $I_\star$ uniformly over $x\in\cX_0$.
Lemma~\ref{lem:bridge-maps-lipschitz} therefore yields
\[
\sup_{\bm\theta\in\Theta_B}
\sup_{x\in\cX_0}
\left|
  g_{\bm\theta}^B\{\widehat\LR(x)\}
  -g_{\bm\theta}^B\{\LR(x)\}
\right|
\leq
L_{B,u}^\star
\|\widehat\LR-\LR\|_{\infty,\cX_0}.
\]
This proves~\eqref{eqn:appendix-uniform-weight-op}. On the event in
\eqref{eqn:density-high-probability-event},
Proposition~\ref{prop:bridge-weight-error-density-estimation} gives
\eqref{eqn:lr-high-probability-bound}; the same Lipschitz inequality then
proves~\eqref{eqn:appendix-uniform-weight-high-probability}.
\end{proof}

\begin{remark}[Uniform weight bounds]
Under Assumption~\ref{ass:strong-overlap},
Lemma~\ref{lem:bridge-maps-lipschitz} gives constants
$0<c_B\leq C_B<\infty$ such that
\[
    c_B\leq w_{\bm\theta}^B(x)\leq C_B
\]
uniformly over $x\in\cX_0$ and $\bm\theta\in\Theta_B$. On the event
$\|\widehat\LR-\LR\|_{\infty,\cX_0}\leq r_-/2$, the same conclusion holds
for the estimated weights with constants obtained from the enlarged interval
$I_\star$. If $P_S^X(\mathcal X_0)=1$, these bounds verify the bridge-weight
boundedness used in the self-normalized empirical-process analysis of
Section~\ref{sec:bridge-risk-bounds}. Without this full-mass condition,
they establish only the corresponding local bridge-weight regularity on
$\mathcal X_0$.
\end{remark}

The joint parameter regularity in Lemma~\ref{lem:bridge-maps-lipschitz} also
controls bridge weights evaluated at estimated parameters.

\begin{corollary}[Estimated bridge-weight control]
\label{cor:estimated-parameter-bridge-weight-rate}
Suppose the assumptions of
Theorem~\ref{thm:unif-consistency-bridge-weights} hold. Let
$\bm\theta_0\in\Theta_B$ and let
$\widehat{\bm\theta}\in\Theta_B$ satisfy
\[
    \|\widehat{\bm\theta}-\bm\theta_0\|
    =\Op(r_{\bm\theta,n,m})
\]
for a deterministic sequence $r_{\bm\theta,n,m}\to0$. Then
\begin{equation*}
    \sup_{x\in\cX_0}
    \left|
      \widehat w_{\widehat{\bm\theta}}^B(x)
      -w_{\bm\theta_0}^B(x)
    \right|
    =
    \Op\left(
      a_{S,n}+a_{T,m}+r_{\bm\theta,n,m}
    \right).
\end{equation*}
\end{corollary}

\begin{proof}
On the event
\[
    \|\widehat\LR-\LR\|_{\infty,\cX_0}\leq r_-/2,
\]
whose probability tends to one, both $\LR$ and $\widehat\LR$ take values in
$I_\star$. Let $L_{B,u}^\star$ and $L_{B,\theta}^\star$ be joint Lipschitz
constants from Lemma~\ref{lem:bridge-maps-lipschitz} on
$I_\star\times\Theta_B$. On this event, adding and subtracting
$g_{\widehat{\bm\theta}}^B\{\LR(x)\}$ and applying
\eqref{eqn:appendix-joint-bridge-lipschitz} gives
\[
\begin{aligned}
&\sup_{x\in\cX_0}
\left|
  g_{\widehat{\bm\theta}}^B\{\widehat\LR(x)\}
  -g_{\bm\theta_0}^B\{\LR(x)\}
\right|
\\
&\qquad\leq
L_{B,u}^\star
\|\widehat\LR-\LR\|_{\infty,\cX_0}
+
L_{B,\theta}^\star
\|\widehat{\bm\theta}-\bm\theta_0\|.
\end{aligned}
\]
The conclusion follows from Lemma~\ref{lem:lr-uniform}.
\end{proof}

We next give a concrete support-adapted density estimator that supplies the
uniform input above when the compact working region carries full mass.  An
ordinary convolution KDE is not used for this purpose: if a density is
positive at the boundary of a compact support, its uncorrected KDE generally
has a nonvanishing boundary bias in the supremum norm.

The construction uses boundary scaling spaces to remove this boundary bias.
Bernstein's inequality for the empirical projection coefficients, together
with the bounded local overlap of the scaling functions, gives the stochastic
term. Taking the positive part and renormalizing preserves the same uniform
rate.

Fix $\beta>0$.

\begin{assumption}[Boundary-corrected projection system]
\label{ass:boundary-projection-system}
After a fixed affine change of coordinates,
$\cX_0=[0,1]^d$. For $j\geq j_0$, let $V_j^{\mathrm{bc}}$ be the tensor-product
boundary scaling space generated by compactly supported
Cohen--Daubechies--Vial scaling functions whose regularity and approximation
order exceed $\beta$, and let
\[
    K_j^{\mathrm{bc}}(x,y)
    =
    \sum_{k\in\mathcal I_j}
       \varphi_{j,k}^{\mathrm{bc}}(x)
       \varphi_{j,k}^{\mathrm{bc}}(y)
\]
be its orthogonal projection kernel. The spaces reproduce constants, so
$1\in V_j^{\mathrm{bc}}$.
\end{assumption}

For $D\in\{S,T\}$, put $N_S=n$, $N_T=m$, let $h_D=2^{-j_D}$, and define the
raw projection estimator
\begin{equation}\label{eqn:raw-boundary-projection-density}
    \widetilde p_D(x)
    =
    \frac1{N_D}\sum_{i=1}^{N_D}
       K_{j_D}^{\mathrm{bc}}(x,X_i^D).
\end{equation}
It integrates to one but need not be nonnegative. Define the density estimator
used in the likelihood ratio by
\begin{equation}\label{eqn:positive-normalized-boundary-projection-density}
    \widehat p_D(x)
    =
    \frac{\{\widetilde p_D(x)\}_+}
         {\int_{\cX_0}\{\widetilde p_D(u)\}_+\d u}.
\end{equation}
The denominator is at least one and hence is always positive.

Use the boundary-projection rate function $\omega_{N,h}(\delta_0)$ defined
in Section~\ref{sec:nonparametric-density-estimation}.

\begin{corollary}[Boundary-projection rates]
\label{cor:unif-consistency-bridge-weights-boundary-projection}
\mbox{}\par\noindent
Assume the boundary projection system above, and suppose both covariate laws
are supported on $\cX_0$. Assume the two densities belong to
bounded intrinsic $\beta$-H\"older--Zygmund balls on $\cX_0$ and satisfy
\[
    p_S\geq c_S>0,
    \qquad
    p_T\leq c_T<\infty.
\]
Suppose also that Assumption~\ref{ass:strong-overlap} holds. Then there are finite
constants $A_S,A_T$, independent of $n,m,j_S,j_T$, and $\delta_0$, such that the
estimators in~\eqref{eqn:positive-normalized-boundary-projection-density}
satisfy~\eqref{eqn:density-high-probability-event} with
\begin{equation}\label{eqn:appendix-boundary-projection-high-probability-rate}
\begin{split}
    a_{S,n}(\delta_0)
    &=A_S\{h_S^\beta+\omega_{n,h_S}(\delta_0)\},\\
    a_{T,m}(\delta_0)
    &=A_T\{h_T^\beta+\omega_{m,h_T}(\delta_0)\}.
\end{split}
\end{equation}

If $h_S,h_T\to0$ and
\[
    \frac{nh_S^d}{\log(1/h_S)}\to\infty,
    \qquad
    \frac{mh_T^d}{\log(1/h_T)}\to\infty,
\]
then Assumption~\ref{ass:kde-input} holds with
\begin{align*}
    a_{S,n}
    &=h_S^\beta+\sqrt{\frac{\log(1/h_S)}{nh_S^d}},\\
    a_{T,m}
    &=h_T^\beta+\sqrt{\frac{\log(1/h_T)}{mh_T^d}}.
\end{align*}
Consequently, for every bridge family $B$,
\begin{align}
&\sup_{\bm\theta\in\Theta_B}
 \sup_{x\in\cX_0}
 \left|
   \widehat w_{\bm\theta}^B(x)-w_{\bm\theta}^B(x)
 \right|
\notag\\
&\quad=
\Op\left(
  h_S^\beta+h_T^\beta
  +\sqrt{\frac{\log(1/h_S)}{nh_S^d}}
  +\sqrt{\frac{\log(1/h_T)}{mh_T^d}}
\right).
\label{eqn:appendix-explicit-boundary-projection-op-rate}
\end{align}
If, at a fixed confidence level,
$a_{S,n}(\delta_0)\leq c_S/2$ and $e_{n,m}(\delta_0)\leq r_-/2$, substitution of
\eqref{eqn:appendix-boundary-projection-high-probability-rate} into
\eqref{eqn:appendix-uniform-weight-high-probability} gives the corresponding
simultaneous high-probability bridge-weight bound.
\end{corollary}

\begin{proof}
Let $\Pi_j^{\mathrm{bc}}$ denote the orthogonal projector with kernel
$K_j^{\mathrm{bc}}$. Since the projection space reproduces constants,
\[
    \int_{\cX_0}K_j^{\mathrm{bc}}(x,y)\d x
    =(\Pi_j^{\mathrm{bc}}1)(y)=1,
\]
and hence $\int\widetilde p_D=1$. The boundary-wavelet approximation property
gives
\[
    \|\Pi_{j_D}^{\mathrm{bc}}p_D-p_D\|_{\infty,\cX_0}
    \leq C_Dh_D^\beta.
\]
Compact support implies that only a fixed number of tensor-product scaling
functions are nonzero at any $x$, while
$\|\varphi_{j,k}^{\mathrm{bc}}\|_\infty\lesssim2^{jd/2}$ and
$|\mathcal I_j|\lesssim2^{jd}$. Each coefficient summand
$\varphi_{j,k}^{\mathrm{bc}}(X_i^D)$ has variance bounded by
$\|p_D\|_\infty$, by $L_2$ normalization, and its centered envelope is of
order $2^{jd/2}$. Applying Bernstein's inequality coefficientwise, taking a
union bound over $\mathcal I_{j_D}$, and then using the bounded local overlap
therefore gives, with probability at least $1-\delta_0/2$,
\[
\|\widetilde p_D-\E\widetilde p_D\|_{\infty,\cX_0}
\leq C_D'
\left[
  \sqrt{\frac{\log(1/h_D)+\log(4/\delta_0)}{N_Dh_D^d}}
  +\frac{\log(1/h_D)+\log(4/\delta_0)}{N_Dh_D^d}
\right].
\]
The required boundary-basis and approximation properties are developed in
\citet{Cohen1993} and
\citet[Sections~4.3.5--4.3.6]{Gine2015}; the preceding stochastic argument is
the standard tensor-product extension of the projection-kernel bounds in
\citet[Section~5.1]{Gine2015}. A further union bound makes the density bounds
simultaneous for $D=S,T$ with probability at least $1-\delta_0$.

It remains to verify that the positive-part normalization does not change the
rate. Put
$b_D=\int_{\cX_0}(\widetilde p_D)_+$. Since
$\int\widetilde p_D=1$, we have $b_D\geq1$ and
\[
    0\leq b_D-1
    =\int_{\cX_0}(\widetilde p_D)_-
    \leq
    \|\widetilde p_D-p_D\|_{\infty,\cX_0},
\]
where the final inequality uses that $\cX_0$ has unit volume. Moreover,
$|a_+-b|\leq|a-b|$ for $b\geq0$. Therefore,
\begin{align*}
\|\widehat p_D-p_D\|_{\infty,\cX_0}
&\leq
\|\widetilde p_D-p_D\|_{\infty,\cX_0}
+\|p_D\|_{\infty,\cX_0}(b_D-1)\\
&\leq
\{1+\|p_D\|_{\infty,\cX_0}\}
\|\widetilde p_D-p_D\|_{\infty,\cX_0}.
\end{align*}
This proves~\eqref{eqn:appendix-boundary-projection-high-probability-rate}.
The stochastic-order statements follow under the bandwidth conditions, and
the bridge-weight conclusions follow from
Theorem~\ref{thm:unif-consistency-bridge-weights}.
\end{proof}

The construction is a boundary-corrected analogue of kernel mollification.
It is used only to give one concrete sufficient input for the uniform theory;
other density estimators can be substituted when they satisfy
\eqref{eqn:density-high-probability-event}. A direct likelihood-ratio
estimator may instead be used when it satisfies the corresponding uniform
ratio-error bound.
An ordinary mollifier can be handled in fixed-$p$ $L_p$ loss by isolating a
shrinking boundary layer, but that argument alone does not give the
full-support $L_\infty$ input used above.

For the pointwise limit theory below, we return to ordinary convolution KDEs.
These calculations are made at interior points, away from an uncorrected
support boundary. The density-estimation kernels are distinct from the
regression kernel used in Appendix~\ref{app:asymptotics-nw}.

\begin{assumption}[Kernels for pointwise density estimation]
\label{ass:kde-kernel}
For $D\in\{S,T\}$, the kernel $K_D:\R^d\to\R$ is bounded, measurable, and
integrable, integrates to one, has finite squared integral
\[
    R(K_D)=\int_{\R^d}K_D(u)^2\d u<\infty.
\]
The kernels may be signed, as is possible for higher-order density kernels.
When the $O(h_D^\beta)$ pointwise bias bound below is invoked, $K_D$ is
additionally assumed to be compactly supported and of integer order
$q>\beta$, so that
\[
    \int_{\R^d}u^{\bm k}K_D(u)\d u=0,
    \qquad 1\leq|\bm k|\leq q-1,
\]
for every multi-index $\bm k$ in the stated range.
\end{assumption}

Classical treatments of KDE rates include \citet{Silverman1986} and
\citet{Wand1994}; uniform empirical-process treatments for ordinary KDEs away
from uncorrected support boundaries include \citet{Einmahl2005},
\citet{Gine2002}, \citet{Tsybakov2009}, and \citet{Gine2015}.

\subsection{Asymptotic normality}

We now derive pointwise central limit theorems for the estimated covariate
likelihood ratio and the resulting fixed-parameter bridge weights. Let
$(X_1^S,\ldots,X_n^S)$ be an i.i.d. sample from $P_S^X$, and let
$(X_1^T,\ldots,X_m^T)$ be an independent i.i.d. sample from $P_T^X$. For
possibly different kernels and bandwidths, define
\begin{align*}
    \widehat p_S(x)
    &=
    \frac{1}{nh_S^d}
    \sum_{i=1}^n
    K_S\left(\frac{x-X_i^S}{h_S}\right),
    \\
    \widehat p_T(x)
    &=
    \frac{1}{mh_T^d}
    \sum_{j=1}^m
    K_T\left(\frac{x-X_j^T}{h_T}\right),
\end{align*}
and put
\[
    \overline p_S(x)=\E\{\widehat p_S(x)\},
    \qquad
    \overline p_T(x)=\E\{\widehat p_T(x)\},
    \qquad
    \overline\LR_{n,m}(x)
    =\frac{\overline p_T(x)}{\overline p_S(x)}.
\]
For this pointwise subsection, define
\[
    \widehat\LR(x)
    =
    \begin{cases}
      \widehat p_T(x)/\widehat p_S(x),
        &\widehat p_S(x)>0\ \text{and}\ \widehat p_T(x)>0,\\
      1,&\text{otherwise}.
    \end{cases}
\]
The fallback value is arbitrary. The assumptions below make the exceptional
event have probability tending to zero, while this convention permits signed
higher-order kernels without evaluating a bridge map at a negative argument.

\begin{assumption}[Pointwise KDE CLT input]
\label{ass:pointwise-kde-clt}
For a fixed $x$ in the interior of the common density-estimation region, the
following conditions hold:
\begin{enumerate}
\item $p_S(x)>0$ and $p_T(x)>0$.
\item The expected density estimators satisfy
\[
    \overline p_S(x)\to p_S(x),
    \qquad
    \overline p_T(x)\to p_T(x).
\]
\item The bandwidths satisfy
\[
    h_S\to0,
    \qquad h_T\to0,
    \qquad nh_S^d\to\infty,
    \qquad mh_T^d\to\infty.
\]
\item The standard pointwise KDE central limit theorems hold:
\[
    \sqrt{nh_S^d}
    \{\widehat p_S(x)-\overline p_S(x)\}
    \dto
    N\{0,p_S(x)R(K_S)\},
\]
and
\[
    \sqrt{mh_T^d}
    \{\widehat p_T(x)-\overline p_T(x)\}
    \dto
    N\{0,p_T(x)R(K_T)\}.
\]
\end{enumerate}
\end{assumption}

Assumptions~\ref{ass:kde-kernel} and~\ref{ass:pointwise-kde-clt} are standard
for pointwise KDE asymptotics. They hold, for example, under continuity of the
densities at $x$, compact support of the kernels, the stated kernel conditions,
and divergence of the local sample sizes, with $x$ kept away from an
uncorrected support boundary. The
centered CLTs and the expectation limits give
$\widehat p_D(x)\xrightarrow{P}p_D(x)>0$, so the fallback event in the definition of
$\widehat\LR(x)$ has probability tending to zero. Centering at $p_S(x)$ and
$p_T(x)$ rather than at the expected KDEs requires the usual bias correction
or undersmoothing condition.

Define
\begin{equation*}
\begin{aligned}
    s_{n,m}^2(x)
    ={}&
    R(K_T)
    \frac{p_T(x)}{mh_T^d p_S(x)^2}
    \\
    &+
    R(K_S)
    \frac{p_T(x)^2}{nh_S^d p_S(x)^3}.
\end{aligned}
\end{equation*}
Under Assumption~\ref{ass:pointwise-kde-clt}, $s_{n,m}(x)>0$ for all
sufficiently large $n,m$ and $s_{n,m}(x)\to0$.

At the fixed interior point $x$, the source and target KDE fluctuations are
asymptotically independent. A first-order expansion of the ratio map
$(u,v)\mapsto v/u$ yields the likelihood-ratio variance and bias, while the
quadratic remainder is negligible relative to $s_{n,m}(x)$. A second
delta-method step transfers this limit through the bridge map. When the bridge
parameter is estimated, its first-order contribution enters through the
parameter derivative of that map, so a complete limiting distribution
requires a joint limit for $\widehat\LR(x)$ and
$\widehat{\bm\theta}$.

\begin{lemma}[Likelihood-ratio CLT]
\label{lem:pointwise-lr-clt}
Suppose Assumptions~\ref{ass:kde-kernel} and~\ref{ass:pointwise-kde-clt}
hold. Then
\begin{equation}\label{eqn:appendix-smoothed-lr-clt}
    \frac{
      \widehat\LR(x)-\overline\LR_{n,m}(x)
    }{
      s_{n,m}(x)
    }
    \dto N(0,1).
\end{equation}
Let
\[
    B_S(x)=\overline p_S(x)-p_S(x),
    \qquad
    B_T(x)=\overline p_T(x)-p_T(x),
\]
and define the first-order likelihood-ratio bias
\begin{equation}\label{eqn:appendix-first-order-lr-bias}
    B_{\LR}^{(1)}(x)
    =
    \frac{B_T(x)}{p_S(x)}
    -
    \frac{p_T(x)}{p_S(x)^2}B_S(x).
\end{equation}
If
\begin{equation}\label{eqn:appendix-lr-second-order-bias-condition}
    \left|
      \overline\LR_{n,m}(x)-\LR(x)-B_{\LR}^{(1)}(x)
    \right|
    =o\{s_{n,m}(x)\},
\end{equation}
then
\begin{equation}\label{eqn:appendix-bias-corrected-lr-clt}
    \frac{
      \widehat\LR(x)-\LR(x)-B_{\LR}^{(1)}(x)
    }{
      s_{n,m}(x)
    }
    \dto N(0,1).
\end{equation}
If, in addition,
$B_{\LR}^{(1)}(x)=o\{s_{n,m}(x)\}$, then
\begin{equation}\label{eqn:appendix-uncentered-lr-clt}
    \frac{\widehat\LR(x)-\LR(x)}{s_{n,m}(x)}
    \dto N(0,1).
\end{equation}
\end{lemma}

\begin{proof}
By Assumption~\ref{ass:pointwise-kde-clt}, the event
$\{\widehat p_S(x)>0,\widehat p_T(x)>0\}$ has probability tending to one.
We work on this event; the arbitrary fallback convention has no effect on
the limit.
Put
\[
    \Delta_S(x)=\widehat p_S(x)-\overline p_S(x),
    \qquad
    \Delta_T(x)=\widehat p_T(x)-\overline p_T(x).
\]
Because $p_S(x)>0$ and $\overline p_S(x)\to p_S(x)$,
$\overline p_S(x)$ is bounded away from zero for all sufficiently large $n$.
A Taylor expansion of $(u,v)\mapsto v/u$ at
$(\overline p_S(x),\overline p_T(x))$ yields
\begin{equation}\label{eqn:appendix-stochastic-ratio-expansion}
\begin{aligned}
    \widehat\LR(x)-\overline\LR_{n,m}(x)
    ={}&
    \frac{\Delta_T(x)}{\overline p_S(x)}
    -
    \frac{\overline p_T(x)}{\overline p_S(x)^2}\Delta_S(x)
    \\
    &+R_{n,m}(x),
\end{aligned}
\end{equation}
where, on an event whose probability tends to one,
\[
    |R_{n,m}(x)|
    \leq
    C\{\Delta_S(x)^2+|\Delta_S(x)\Delta_T(x)|\}.
\]
Assumption~\ref{ass:pointwise-kde-clt} gives
\[
    \Delta_S(x)=\Op\{(nh_S^d)^{-1/2}\},
    \qquad
    \Delta_T(x)=\Op\{(mh_T^d)^{-1/2}\}.
\]
Consequently,
\[
    R_{n,m}(x)
    =
    \Op\left
    \{
      (nh_S^d)^{-1}
      +(nh_S^d mh_T^d)^{-1/2}
    \right\}.
\]
Since
\[
    s_{n,m}(x)
    \asymp
    \left
    \{
      (nh_S^d)^{-1}+(mh_T^d)^{-1}
    \right\}^{1/2},
\]
it follows that $R_{n,m}(x)/s_{n,m}(x)=\op(1)$.

It remains to analyze the linear part of
\eqref{eqn:appendix-stochastic-ratio-expansion}. By the pointwise KDE CLTs
and independence of the source and target samples, the two standardized
terms converge jointly to independent standard normal variables. The
corresponding asymptotic variance scale for their linear combination, with
$\overline p_S$ and $\overline p_T$ in the coefficients, is
\[
\begin{aligned}
    \overline s_{n,m}^2(x)
    ={}&
    R(K_T)
    \frac{p_T(x)}{mh_T^d\overline p_S(x)^2}
    \\
    &+
    R(K_S)
    \frac{\overline p_T(x)^2p_S(x)}
         {nh_S^d\overline p_S(x)^4}.
\end{aligned}
\]
Since $\overline p_S(x)\to p_S(x)$ and
$\overline p_T(x)\to p_T(x)$,
$\overline s_{n,m}(x)/s_{n,m}(x)\to1$. The relative contributions of the
source and target terms need not converge along the original sequence. Along
every subsequence, however, one may select a further subsequence on which
their variance shares converge. The joint Gaussian limit then gives a
standard normal limit on that further subsequence because the limiting shares
sum to one. This subsequence argument proves that the normalized linear term,
and hence the left-hand side of
\eqref{eqn:appendix-smoothed-lr-clt}, converges to $N(0,1)$ along the full
sequence.

For the deterministic bias, a Taylor expansion of
$\overline p_T/\overline p_S$ at $(p_S,p_T)$ gives
\[
    \overline\LR_{n,m}(x)-\LR(x)
    =
    B_{\LR}^{(1)}(x)+\rho_{B,n,m}(x),
\]
where
\[
    \rho_{B,n,m}(x)
    =
    \overline\LR_{n,m}(x)-\LR(x)-B_{\LR}^{(1)}(x).
\]
Under~\eqref{eqn:appendix-lr-second-order-bias-condition}, adding this
deterministic expansion to
\eqref{eqn:appendix-smoothed-lr-clt} proves
\eqref{eqn:appendix-bias-corrected-lr-clt}. If
$B_{\LR}^{(1)}(x)=o\{s_{n,m}(x)\}$ as well,
\eqref{eqn:appendix-uncentered-lr-clt} follows.
\end{proof}

\begin{remark}[Bias and undersmoothing]
If $p_S$ and $p_T$ are $\beta$-H\"older smooth near $x$ and the kernels are
of sufficient order, then
\[
    B_D(x)=O(h_D^\beta),
    \qquad D\in\{S,T\}.
\]
The condition
\[
    h_S^\beta+h_T^\beta=o\{s_{n,m}(x)\}
\]
therefore makes both the first-order ratio bias and its deterministic
second-order remainder negligible. Without this undersmoothing condition,
the CLT remains valid after centering at the smoothed ratio or after
subtracting the first-order bias in
\eqref{eqn:appendix-first-order-lr-bias} under
\eqref{eqn:appendix-lr-second-order-bias-condition}.
\end{remark}

We now transfer the likelihood-ratio limit through the structural bridge maps.
For fixed $\bm\theta\in\Theta_B$, define
\[
    \overline w_{\bm\theta,n,m}^B(x)
    =g_{\bm\theta}^B\{\overline\LR_{n,m}(x)\},
    \qquad
    w_{\bm\theta}^B(x)
    =g_{\bm\theta}^B\{\LR(x)\}.
\]
The derivatives with respect to the likelihood-ratio argument are
\begin{align}
    \dot g_\alpha^{\mathrm{KL}}(u)
    &=\alpha u^{\alpha-1},
    \label{eqn:appendix-kl-weight-derivative}
    \\
    \dot g_\alpha^{\mathrm{HD}}(u)
    &=
    \frac{\alpha\{(1-\alpha)+\alpha\sqrt u\}}{\sqrt u},
    \label{eqn:appendix-hd-weight-derivative}
\end{align}
and, for the KL/HD bridge,
\begin{equation*}
    \dot g_{\alpha}^{\mathrm{KH}}(u)
    =
    \frac{g_{\alpha}^{\mathrm{KH}}(u)}{u}
    \frac{W_0\{\beta_{\alpha}(u)\}}
         {1+W_0\{\beta_{\alpha}(u)\}}.
\end{equation*}
For every $\alpha>0$ in the stated parameter range, these derivatives are
strictly positive. At $\alpha=0$, they vanish.

\begin{theorem}[Fixed-parameter bridge-weight CLT]
\label{thm:pw-clt-bridge-weights}
Suppose Assumptions~\ref{ass:kde-kernel} and~\ref{ass:pointwise-kde-clt}
hold. Let $g_{\bm\theta}^B$ be one of the structural bridge maps, and suppose
it is continuously differentiable in a neighborhood of
$u_0=\LR(x)$ with
$\dot g_{\bm\theta}^B(u_0)\neq0$. Then
\begin{equation*}
    \frac{
      \widehat w_{\bm\theta}^B(x)
      -\overline w_{\bm\theta,n,m}^B(x)
    }{
      |\dot g_{\bm\theta}^B(u_0)|s_{n,m}(x)
    }
    \dto N(0,1).
\end{equation*}
If
\begin{equation}\label{eqn:appendix-weight-second-order-bias-condition}
\left|
  \overline w_{\bm\theta,n,m}^B(x)
  -w_{\bm\theta}^B(x)
  -\dot g_{\bm\theta}^B(u_0)B_{\LR}^{(1)}(x)
\right|
=o\{s_{n,m}(x)\},
\end{equation}
then
\begin{equation*}
    \frac{
      \widehat w_{\bm\theta}^B(x)
      -w_{\bm\theta}^B(x)
      -\dot g_{\bm\theta}^B(u_0)B_{\LR}^{(1)}(x)
    }{
      |\dot g_{\bm\theta}^B(u_0)|s_{n,m}(x)
    }
    \dto N(0,1).
\end{equation*}
If, in addition,
$B_{\LR}^{(1)}(x)=o\{s_{n,m}(x)\}$, then
\begin{equation*}
    \frac{
      \widehat w_{\bm\theta}^B(x)-w_{\bm\theta}^B(x)
    }{
      |\dot g_{\bm\theta}^B(u_0)|s_{n,m}(x)
    }
    \dto N(0,1).
\end{equation*}
\end{theorem}

\begin{proof}
Lemma~\ref{lem:pointwise-lr-clt} gives
\[
    \widehat\LR(x)-\overline\LR_{n,m}(x)
    =\Op\{s_{n,m}(x)\}
    =\op(1).
\]
A Taylor expansion at $\overline\LR_{n,m}(x)$ yields
\[
\begin{aligned}
    \widehat w_{\bm\theta}^B(x)
    -\overline w_{\bm\theta,n,m}^B(x)
    ={}&
    \dot g_{\bm\theta}^B\{\overline\LR_{n,m}(x)\}
    \{\widehat\LR(x)-\overline\LR_{n,m}(x)\}
    \\
    &+r_{n,m}(x),
\end{aligned}
\]
where $r_{n,m}(x)/s_{n,m}(x)=\op(1)$. Since
$\overline\LR_{n,m}(x)\to u_0$,
\[
    \dot g_{\bm\theta}^B\{\overline\LR_{n,m}(x)\}
    \longrightarrow
    \dot g_{\bm\theta}^B(u_0).
\]
The first claim follows from
Lemma~\ref{lem:pointwise-lr-clt} and Slutsky's theorem.

For the second claim, decompose
\[
\begin{aligned}
    \widehat w_{\bm\theta}^B(x)-w_{\bm\theta}^B(x)
    ={}&
    \{\widehat w_{\bm\theta}^B(x)
      -\overline w_{\bm\theta,n,m}^B(x)\}
    \\
    &+
    \{\overline w_{\bm\theta,n,m}^B(x)
      -w_{\bm\theta}^B(x)\}.
\end{aligned}
\]
Condition~\eqref{eqn:appendix-weight-second-order-bias-condition} and the
first part give the bias-corrected CLT. The uncentered version follows when
the first-order bias is also negligible.
\end{proof}

\begin{remark}[Endpoint and degeneracy]
At $\alpha=0$, all three unnormalized structural weights are identically one.
In all three cases,
the normalized structural bridge law is the source law and the derivative
with respect to the likelihood ratio is zero. The first-order weight CLT is
therefore degenerate at the source-only endpoint. This is not a pathology: the
bridge does not use the estimated likelihood ratio at that endpoint.
\end{remark}

The preceding theorem holds for fixed bridge parameters. The next expansion
records how parameter estimation enters the bridge weight at first order.

\begin{proposition}[Estimated-parameter bridge-weight expansion]
\label{prop:estimated-parameter-weight-expansion}
Fix $x\in\xspace$, and suppose
$\widehat\LR(x)\to\LR(x)$ and
$\widehat{\bm\theta}\to\bm\theta_0$ in probability. Assume that
$(u,\bm\theta)\mapsto g_{\bm\theta}^B(u)$ admits a continuously
differentiable extension to a neighborhood of
$(\LR(x),\bm\theta_0)$.
Then
\begin{align}
&\widehat w_{\widehat{\bm\theta}}^B(x)
-w_{\bm\theta_0}^B(x)
\notag\\
&\quad=
\dot g_{\bm\theta_0}^B\{\LR(x)\}
\{\widehat\LR(x)-\LR(x)\}
+
\left[
  \left.
  \nabla_{\bm\theta}g_{\bm\theta}^B\{\LR(x)\}
  \right|_{\bm\theta=\bm\theta_0}
\right]^{\mathsf T}
(\widehat{\bm\theta}-\bm\theta_0)
\notag\\
&\qquad+
\op\left(
  |\widehat\LR(x)-\LR(x)|
  +\|\widehat{\bm\theta}-\bm\theta_0\|
\right).
\label{eqn:appendix-estimated-parameter-weight-expansion}
\end{align}
\end{proposition}

\begin{proof}
This is the first-order multivariate Taylor expansion of
$(u,\bm\theta)\mapsto g_{\bm\theta}^B(u)$ at
$(\LR(x),\bm\theta_0)$, evaluated at
$(\widehat\LR(x),\widehat{\bm\theta})$.
\end{proof}

A joint central limit theorem for an estimated-parameter bridge requires the
joint limit of $\widehat\LR(x)$ and $\widehat{\bm\theta}$. The expansion
\eqref{eqn:appendix-estimated-parameter-weight-expansion} identifies the two
first-order contributions without imposing a particular bridge-parameter
estimation method.

\section{Technical details for bridge-weighted Nadaraya--Watson}
\label{app:asymptotics-nw}

\subsection{Notation and assumptions}

We first introduce notation used throughout the proofs. The response $Y$ is
real valued. We fix a regular conditional law $P_S^{Y\mid X}$ and use the
resulting Borel versions of the conditional moment functions on the local
neighborhood under consideration. For a regression kernel $K$ and bandwidth
$b=b_n\to0$, write
\[
    K_b(x-X)
    =b^{-d}K\left(\frac{x-X}{b}\right).
\]
Within this appendix, $p_S$ denotes a fixed Borel version of the marginal
density of the source covariates, and
\[
    m_S(x)=\E_{P_S}[Y\mid X=x],
    \qquad
    \sigma_S^2(x)=\Var_{P_S}[Y\mid X=x].
\]
For a $P_S$-integrable measurable function $f=f(X,Y)$, put
\[
    P_nf=\frac1n\sum_{i=1}^n f(X_i^S,Y_i^S),
    \qquad
    Pf=\E_{P_S}[f(X,Y)].
\]

Fix a bridge family $B\in\{\mathrm{KL},\mathrm{HD},\mathrm{KH}\}$ and a
population bridge parameter $\bm\theta_0\in\Theta_B$. Let
$\widehat{\bm\theta}$ be a measurable $\Theta_B$-valued estimator. We take
$\widehat\LR$ and the resulting random functions to be jointly measurable and
assume that the displayed suprema below are measurable. We abbreviate
\[
    w_0(x)=w_{\bm\theta_0}^B(x),
    \qquad
    \widehat w_{\widehat{\bm\theta}}^B(x)
    =g_{\widehat{\bm\theta}}^B\{\widehat\LR(x)\}.
\]
Because every estimator below is self-normalized, multiplying all bridge
weights by the same positive constant does not change the estimator.

Define the oracle bridge-weighted Nadaraya--Watson numerator and denominator
by
\[
    \widehat N_0(x)
    =P_n\{w_0(X)K_b(x-X)Y\},
    \qquad
    \widehat D_0(x)
    =P_n\{w_0(X)K_b(x-X)\},
\]
and let
\[
    \widehat m_0(x)
    =
    \begin{cases}
      \widehat N_0(x)/\widehat D_0(x),&\widehat D_0(x)>0,\\
      0,&\widehat D_0(x)=0.
    \end{cases}
\]
The corresponding population-smoothed numerator and denominator are
\[
    N_b(x)=P\{w_0(X)K_b(x-X)Y\},
    \qquad
    D_b(x)=P\{w_0(X)K_b(x-X)\},
\]
and the smoothed oracle regression target is
\[
    m_b(x)
    =
    \begin{cases}
      N_b(x)/D_b(x),&D_b(x)>0,\\
      0,&D_b(x)=0.
    \end{cases}
\]
The plug-in bridge-weighted estimator is
\[
    \widehat m_{\widehat w,\widehat{\bm\theta}}^B(x)
    =
    \begin{cases}
      \widehat N_{\widehat w}^B(x)/
      \widehat D_{\widehat w}^B(x),
      &\widehat D_{\widehat w}^B(x)>0,\\
      0,&\widehat D_{\widehat w}^B(x)=0,
    \end{cases}
\]
where
\[
    \widehat N_{\widehat w}^B(x)
    =P_n\{\widehat w_{\widehat{\bm\theta}}^B(X)K_b(x-X)Y\},
    \qquad
    \widehat D_{\widehat w}^B(x)
    =P_n\{\widehat w_{\widehat{\bm\theta}}^B(X)K_b(x-X)\}.
\]
The population quantities above are used only at the fixed point or compact
evaluation set specified below and for all sufficiently small $b$. Compact
kernel support and the local moment assumptions then ensure their finiteness.
The denominator-control results below show that the zero-denominator
conventions are asymptotically immaterial at the points and sets under
consideration.

\begin{assumption}[Regression kernel]
\label{ass:nw-kernel}
The regression kernel $K:\R^d\to[0,\infty)$ is bounded and measurable, has
compact support, integrates to one, and satisfies
\[
    R(K)=\int_{\R^d}K(u)^2\d u<\infty.
\]
For some $0<s\leq2$, if $1<s\leq2$, assume additionally that
\[
    \int_{\R^d}uK(u)\d u=0,
\]
as holds, for example, when $K$ is symmetric.
\end{assumption}

\begin{assumption}[Local regularity]
\label{ass:nw-regularity}
Fix either a point $x_0\in\operatorname{int}(\cX)$ or a compact set
$\mathcal K\Subset\operatorname{int}(\cX)$. There is a bounded open
neighborhood $\cN$ of the point or of $\mathcal K$ such that
$\overline{\cN}$ is compact and contained in
$\operatorname{int}(\cX)$, and:
\begin{enumerate}
\item The source design density is bounded away from zero and infinity on
      $\cN$:
      \[
          0<c_p\leq p_S(u)\leq C_p<\infty.
      \]
\item The oracle bridge weight is bounded away from zero and infinity on
      $\cN$:
      \[
          0<c_w\leq w_0(u)\leq C_w<\infty.
      \]
\item The functions $p_S$, $w_0$, and $m_S$ are locally $s$-H\"older on
      $\cN$: for $0<s\leq1$ they are H\"older of order $s$, while for
      $1<s\leq2$ they are continuously differentiable with
      $(s-1)$-H\"older gradients. The corresponding constants are uniform
      when the compact set $\mathcal K$ is considered. In particular, the
      products $w_0p_S$ and $w_0m_Sp_S$ have the same local smoothness.
\item For the pointwise CLT, $\sigma_S^2$ is continuous at $x_0$; for the
      consistency results it is locally bounded on $\cN$.
\end{enumerate}
\end{assumption}

\begin{assumption}[Moment condition]
\label{ass:nw-moments}
There exists $q_0>0$ such that
\[
    \sup_{u\in\cN}
    \E\left[|Y|^{2+q_0}\mid X=u\right]<\infty.
\]
Equivalently, one may impose the same condition on the regression error
$\varepsilon=Y-m_S(X)$ together with
$\sup_{u\in\cN}|m_S(u)|<\infty$. Thus,
boundedness of $Y$ is not required.
\end{assumption}

\begin{assumption}[Regression bandwidth]
\label{ass:nw-bandwidth}
The deterministic regression bandwidth $b=b_n>0$ satisfies
\[
    b\to0,
    \qquad
    nb^d\to\infty.
\]
For uniform-in-$x$ rates, assume additionally that
\[
    \frac{nb^d}{\log n}\to\infty.
\]
\end{assumption}

\begin{assumption}[Plug-in bridge-weight error]
\label{ass:nw-weight-error}
Define the local uniform weight error
\[
    \delta_{n,m}
    =
    \left\|
      \widehat w_{\widehat{\bm\theta}}^B
      -w_{\bm\theta_0}^B
    \right\|_{\infty,\cN}.
\]
Assume
\[
    \delta_{n,m}=\Op(\bar\delta_{n,m}),
    \qquad
    \bar\delta_{n,m}\to0.
\]
For density-estimated bridge weights,
Corollary~\ref{cor:estimated-parameter-bridge-weight-rate} gives the following
sufficient specification when all its hypotheses hold on a compact
density-estimation region containing $\overline{\cN}$; in particular, the
strong-overlap condition holds there. If, for a deterministic sequence
$r_{\bm\theta,n,m}\to0$,
\[
    \|\widehat{\bm\theta}-\bm\theta_0\|
    =\Op(r_{\bm\theta,n,m}),
\]
then one may take
\[
    \bar\delta_{n,m}
    =
    a_{S,n}+a_{T,m}+r_{\bm\theta,n,m}.
\]
For a fixed bridge parameter, $r_{\bm\theta,n,m}=0$.
No particular construction of $\bm\theta_0$ or
$\widehat{\bm\theta}$ is imposed. The parameter rate is an additional input
and is not implied by the radius-profile criterion in
\eqref{eqn:radius-calibration-estimate-population}.
\end{assumption}

\begin{assumption}[Uniform process input]
\label{ass:nw-uniform-ep}
For the uniform-in-$x$ version of
Theorem~\ref{thm:nw-consistency-rate}, assume that, for some deterministic
sequence $\zeta_{n,b}\to0$,
\[
    \sup_{x\in\mathcal K}
    \left|
      (P_n-P)\{w_0(X)K_b(x-X)Y\}
    \right|
    =\Op(\zeta_{n,b}),
\]
\[
    \sup_{x\in\mathcal K}
    \left|
      (P_n-P)\{w_0(X)K_b(x-X)\}
    \right|
    =\Op(\zeta_{n,b}),
\]
and
\[
    \sup_{x\in\mathcal K}
    P_n\{|K_b(x-X)|(1+|Y|)\}
    =\Op(1).
\]
For a compact interior set $\mathcal K$, VC-type or Lipschitz kernel
translates, and a locally bounded response envelope, one may take
\[
    \zeta_{n,b}
    \asymp
    \sqrt{\frac{\log n}{nb^d}},
\]
up to constants; see, for example, \citet{Gine2002,Gine2015}. Under the moment condition in
Assumption~\ref{ass:nw-moments} without a bounded envelope, we retain the
abstract sequence $\zeta_{n,b}$; a truncation argument may add a
moment-dependent tail term.
\end{assumption}

\begin{remark}[Localization by the regression kernel]
Because $K$ has compact support, Assumption~\ref{ass:nw-bandwidth} implies
that, for all sufficiently large $n$, every observation contributing to a
kernel average at the point or set under consideration has its covariate in
$\cN$. The local sup-norm error in Assumption~\ref{ass:nw-weight-error}
therefore controls every plug-in weight entering the corresponding estimator.
Ordinary convolution density estimators may be used when
$\overline{\cN}$ lies away from a support boundary. Support adaptation is
needed only when uniform control is required on a region reaching that
boundary, as in Section~\ref{sec:nonparametric-density-estimation}.
\end{remark}

\begin{remark}[What enters through $\widehat{\bm\theta}$]
The error $\delta_{n,m}$ contains both the density-ratio estimation error and the
bridge-parameter estimation error. The results below first treat the regime
in which this full plug-in error is negligible at the local NW scale. The
condition $\sqrt{nb^d}\,\bar\delta_{n,m}\to0$ used below is sufficient for
oracle equivalence and is not asserted to be necessary. If the plug-in error
is not negligible, Proposition~\ref{prop:estimated-parameter-weight-expansion}
identifies its pointwise first-order density-ratio and parameter contributions.
A limit theorem for the resulting NW estimator would additionally require a
joint local-process analysis.
\end{remark}

\subsection{Technical lemmas}

\begin{lemma}[Population denominator and smoothing bias]
\label{lem:denom-bias}
Under Assumptions~\ref{ass:nw-kernel} and~\ref{ass:nw-regularity}, for every
fixed interior point $x_0$,
\[
    D_b(x_0)
    =w_0(x_0)p_S(x_0)+O(b^s),
\]
and hence, for all sufficiently small $b$,
\[
    D_b(x_0)
    \geq\frac12 w_0(x_0)p_S(x_0)>0.
\]
Moreover,
\[
    m_b(x_0)-m_S(x_0)=O(b^s).
\]
Uniformly over a compact interior set $\mathcal K$,
\[
    \sup_{x\in\mathcal K}
    |D_b(x)-w_0(x)p_S(x)|
    =O(b^s),
\]
and
\[
    \sup_{x\in\mathcal K}
    |m_b(x)-m_S(x)|
    =O(b^s).
\]
\end{lemma}

\begin{proof}
We prove the pointwise statement; the uniform statement follows from the same
argument using the corresponding uniform smoothness on $\cN$. First,
\[
    D_b(x_0)
    =\int w_0(u)K_b(x_0-u)p_S(u)\d u.
\]
With the change of variables $u=x_0-bt$,
\[
    D_b(x_0)
    =\int K(t)w_0(x_0-bt)p_S(x_0-bt)\d t.
\]
Put $f=w_0p_S$. If $0<s\leq1$, the H\"older condition gives
$|f(x_0-bt)-f(x_0)|\leq Cb^s\|t\|^s$. If $1<s\leq2$, the first-order
Taylor expansion has a remainder bounded by $Cb^s\|t\|^s$, and the linear
term vanishes after integration because $\int tK(t)\d t=0$. Compact support
of $K$ therefore gives
\[
    D_b(x_0)
    =w_0(x_0)p_S(x_0)+O(b^s).
\]
The lower bound follows because $w_0(x_0)p_S(x_0)>0$.

The same argument applied to $w_0m_Sp_S$ gives
\[
    N_b(x_0)
    =\int w_0(u)K_b(x_0-u)m_S(u)p_S(u)\d u
\]
and therefore
\[
    N_b(x_0)
    =w_0(x_0)m_S(x_0)p_S(x_0)+O(b^s).
\]
Consequently,
\[
\begin{aligned}
    m_b(x_0)
    &=\frac{N_b(x_0)}{D_b(x_0)}\\
    &=\frac{
      w_0(x_0)m_S(x_0)p_S(x_0)+O(b^s)
    }{
      w_0(x_0)p_S(x_0)+O(b^s)
    }
    =m_S(x_0)+O(b^s).
\end{aligned}
\]
\end{proof}

\begin{lemma}[Oracle denominator control]
\label{lem:oracle-denom}
Under Assumptions~\ref{ass:nw-kernel}, \ref{ass:nw-regularity},
and~\ref{ass:nw-bandwidth}, for every fixed interior point $x_0$,
\[
    \widehat D_0(x_0)-D_b(x_0)
    =\Op\left((nb^d)^{-1/2}\right).
\]
Consequently, $\widehat D_0(x_0)\geq c>0$ with probability tending to one.
Under Assumption~\ref{ass:nw-uniform-ep},
\[
    \sup_{x\in\mathcal K}
    |\widehat D_0(x)-D_b(x)|
    =\Op(\zeta_{n,b}),
\]
and
\[
    \inf_{x\in\mathcal K}\widehat D_0(x)\geq c>0
\]
with probability tending to one.
\end{lemma}

\begin{proof}
For fixed $x_0$,
\[
    \widehat D_0(x_0)-D_b(x_0)
    =(P_n-P)\{w_0(X)K_b(x_0-X)\}.
\]
The summands are i.i.d. and centered, and
\[
\begin{aligned}
    \Var\{w_0(X)K_b(x_0-X)\}
    &\leq
    \E\{w_0(X)^2K_b(x_0-X)^2\}\\
    &=O(b^{-d}),
\end{aligned}
\]
where compact support localizes the integral to $\cN$ for all sufficiently
small $b$. Hence
\[
    \widehat D_0(x_0)-D_b(x_0)
    =\Op\left((nb^d)^{-1/2}\right)
    =\op(1).
\]
Lemma~\ref{lem:denom-bias} gives a positive lower bound for $D_b(x_0)$, so
$\widehat D_0(x_0)$ is bounded away from zero with probability tending to
one. The uniform statements follow from
Assumption~\ref{ass:nw-uniform-ep} and the uniform lower bound in
Lemma~\ref{lem:denom-bias}.
\end{proof}

\begin{lemma}[Oracle numerator order]
\label{lem:oracle-numer-order}
Under Assumptions~\ref{ass:nw-kernel}, \ref{ass:nw-regularity},
\ref{ass:nw-moments}, and~\ref{ass:nw-bandwidth}, for every fixed interior
point $x_0$,
\[
    (P_n-P)
    \left[
      w_0(X)K_b(x_0-X)\{Y-m_b(x_0)\}
    \right]
    =\Op\left((nb^d)^{-1/2}\right).
\]
Under Assumption~\ref{ass:nw-uniform-ep},
\[
\begin{aligned}
&\sup_{x\in\mathcal K}
\left|
  (P_n-P)
  \left[
    w_0(X)K_b(x-X)\{Y-m_b(x)\}
  \right]
\right|\\
&\qquad=\Op(\zeta_{n,b}).
\end{aligned}
\]
\end{lemma}

\begin{proof}
For fixed $x_0$, define
\[
    \xi_n(X,Y)
    =w_0(X)K_b(x_0-X)\{Y-m_b(x_0)\}.
\]
By the definition of $m_b$,
\[
    P\xi_n
    =N_b(x_0)-m_b(x_0)D_b(x_0)=0.
\]
The local conditional second-moment bound and the compact support of $K$
give
\[
\begin{aligned}
    \E\xi_n^2
    &=\int w_0(u)^2K_b(x_0-u)^2
      \E[\{Y-m_b(x_0)\}^2\mid X=u]p_S(u)\d u\\
    &=O(b^{-d}).
\end{aligned}
\]
Therefore,
\[
    (P_n-P)\xi_n
    =\Op\left((nb^d)^{-1/2}\right).
\]
For the uniform statement, write
\[
    w_0K_b(Y-m_b)=w_0K_bY-m_bw_0K_b
\]
and apply Assumption~\ref{ass:nw-uniform-ep},
Lemma~\ref{lem:denom-bias}, and local boundedness of $m_b$ on $\mathcal K$.
\end{proof}

\begin{lemma}[Local kernel-average control]
\label{lem:local-avg-y}
Under Assumptions~\ref{ass:nw-kernel}, \ref{ass:nw-regularity},
\ref{ass:nw-moments}, and~\ref{ass:nw-bandwidth}, for every fixed interior
point $x_0$,
\[
    P_n\{|K_b(x_0-X)|(1+|Y|)\}
    =\Op(1).
\]
Under Assumption~\ref{ass:nw-uniform-ep},
\[
    \sup_{x\in\mathcal K}
    P_n\{|K_b(x-X)|(1+|Y|)\}
    =\Op(1).
\]
\end{lemma}

\begin{proof}
For fixed $x_0$, Markov's inequality gives the result once the expectation is
shown to be bounded. Compact support ensures that, for sufficiently small
$b$, only $u\in\cN$ contribute. Hence
\[
\begin{aligned}
&\E\{|K_b(x_0-X)|(1+|Y|)\}\\
&\quad=
\int |K_b(x_0-u)|
\{1+\E(|Y|\mid X=u)\}p_S(u)\d u\\
&\quad\leq
C\int b^{-d}
\left|K\left(\frac{x_0-u}{b}\right)\right|\d u
=C\int|K(t)|\d t<\infty.
\end{aligned}
\]
The uniform statement is part of
Assumption~\ref{ass:nw-uniform-ep}.
\end{proof}

\subsection{Proofs of the main NW theorems}

\paragraph{Idea of proof for Theorem~\ref{thm:nw-consistency-rate}.}
The oracle error is split at the population-smoothed target $m_b$.
Lemma~\ref{lem:denom-bias} controls the smoothing bias, while
Lemmas~\ref{lem:oracle-denom} and~\ref{lem:oracle-numer-order} control the
centered empirical ratio. The plug-in estimator is then compared pathwise
with the oracle estimator through the exact centered identity
\eqref{eqn:appendix-nw-plugin-identity}; the same argument after taking
suprema gives the compact-set result.

\begin{proof}\textbf{of Theorem~\ref{thm:nw-consistency-rate}.}
We first prove the oracle pointwise rate. By
Lemma~\ref{lem:oracle-denom}, $\widehat D_0(x_0)>0$ on an event whose
probability tends to one. On this event, the identity
\[
    \widehat m_0(x_0)-m_b(x_0)
    =
    \frac{
      (P_n-P)
      [w_0(X)K_b(x_0-X)\{Y-m_b(x_0)\}]
    }{
      \widehat D_0(x_0)
    }
\]
is exact. The arbitrary zero-denominator convention is immaterial.
Lemma~\ref{lem:oracle-numer-order} gives a numerator of order
$\Op((nb^d)^{-1/2})$, while
Lemma~\ref{lem:oracle-denom} gives a denominator bounded away from zero with
probability tending to one. Therefore,
\[
    \widehat m_0(x_0)-m_b(x_0)
    =\Op\left((nb^d)^{-1/2}\right).
\]
Lemma~\ref{lem:denom-bias} then yields
\[
    \widehat m_0(x_0)-m_S(x_0)
    =\Op\left(b^s+(nb^d)^{-1/2}\right).
\]

We next control the plug-in perturbation. On the event, established below,
that both empirical denominators are positive, the exact centered identity is
\begin{equation}
\label{eqn:appendix-nw-plugin-identity}
\begin{aligned}
&\widehat m_{\widehat w,\widehat{\bm\theta}}^B(x_0)
-\widehat m_0(x_0)\\
&\quad=
\frac{
P_n\left[
  \{\widehat w_{\widehat{\bm\theta}}^B-w_0\}(X)
  K_b(x_0-X)
  \{Y-\widehat m_0(x_0)\}
\right]
}{
\widehat D_{\widehat w}^B(x_0)
}.
\end{aligned}
\end{equation}
Indeed,
\[
    P_n\left[
      w_0(X)K_b(x_0-X)
      \{Y-\widehat m_0(x_0)\}
    \right]=0.
\]
For all sufficiently small $b$, the kernel localization remark allows us to
use the norm in Assumption~\ref{ass:nw-weight-error}. On an event whose
probability tends to one,
\[
\begin{aligned}
    \widehat D_{\widehat w}^B(x_0)
    &\geq
    \widehat D_0(x_0)
    -\delta_{n,m}P_nK_b(x_0-X)\\
    &\geq c>0,
\end{aligned}
\]
by Lemmas~\ref{lem:oracle-denom} and~\ref{lem:local-avg-y} and the fact that
$\delta_{n,m}=\op(1)$. Moreover,
\[
\begin{aligned}
&\left|
P_n\left[
  \{\widehat w_{\widehat{\bm\theta}}^B-w_0\}(X)
  K_b(x_0-X)
  \{Y-\widehat m_0(x_0)\}
\right]
\right|\\
&\quad\leq
\delta_{n,m}
P_n\left[
  |K_b(x_0-X)|
  \{|Y|+|\widehat m_0(x_0)|\}
\right].
\end{aligned}
\]
Lemma~\ref{lem:local-avg-y} gives
$P_n\{|K_b(x_0-X)|(1+|Y|)\}=\Op(1)$, and the already established oracle
rate gives $\widehat m_0(x_0)=\Op(1)$. Thus
\[
    \widehat m_{\widehat w,\widehat{\bm\theta}}^B(x_0)
    -\widehat m_0(x_0)
    =\Op(\bar\delta_{n,m}).
\]
Combining the oracle and plug-in bounds proves the pointwise assertions.
The plug-in comparison is pathwise on the weight-error event and does not
require the estimated weights to be independent of the labeled source sample.

For the uniform assertions, the oracle numerator and denominator bounds
follow from Assumption~\ref{ass:nw-uniform-ep}. With probability tending to
one, both denominators are then bounded away from zero uniformly over
$\mathcal K$, and the same centered identity
gives
\[
\begin{aligned}
&\sup_{x\in\mathcal K}
\left|
  \widehat m_{\widehat w,\widehat{\bm\theta}}^B(x)
  -\widehat m_0(x)
\right|\\
&\quad\leq
C\delta_{n,m}
\sup_{x\in\mathcal K}
P_n\left[
  |K_b(x-X)|
  \{1+|Y|+|\widehat m_0(x)|\}
\right]
=\Op(\bar\delta_{n,m}),
\end{aligned}
\]
where the last step uses Assumption~\ref{ass:nw-uniform-ep}, the uniform
oracle rate, and local boundedness of $m_S$. This completes the proof.
\end{proof}

\paragraph{Idea of proof for Theorem~\ref{thm:nw-oracle-clt}.}
Centering at $m_b(x_0)$ produces an exactly centered triangular array whose
variance is of order $b^{-d}$. The conditional moment assumption verifies
Lyapunov's condition, the empirical denominator converges to
$w_0(x_0)p_S(x_0)$, and Slutsky's theorem yields the stated limit. The bridge
weight cancels from the limiting variance; undersmoothing then permits
centering at $m_S(x_0)$.

\begin{proof}\textbf{of Theorem~\ref{thm:nw-oracle-clt}.}
Define the triangular-array function and variables
\[
    \xi_n(x,y)
    =w_0(x)K_b(x_0-x)\{y-m_b(x_0)\},
    \qquad
    \xi_{i,n}=\xi_n(X_i^S,Y_i^S).
\]
Then $P\xi_n=\E\xi_{i,n}=0$, and, on the event
$\{\widehat D_0(x_0)>0\}$,
\[
    \widehat m_0(x_0)-m_b(x_0)
    =\frac{P_n\xi_n}{\widehat D_0(x_0)}.
\]
Lemma~\ref{lem:oracle-denom} shows that this event has probability tending
to one.
Because $m_b(x_0)\to m_S(x_0)$ and the relevant functions are continuous at
$x_0$, the identity
$\E[\{Y-m_b(x_0)\}^2\mid X=u]
=\sigma_S^2(u)+\{m_S(u)-m_b(x_0)\}^2$
gives
\[
\begin{aligned}
    b^d\E[\xi_{i,n}^2]
    &=\int w_0(x_0-bt)^2K(t)^2\\
    &\qquad\times
    \E[\{Y-m_b(x_0)\}^2\mid X=x_0-bt]
    p_S(x_0-bt)\d t\\
    &\longrightarrow
    w_0(x_0)^2p_S(x_0)\sigma_S^2(x_0)R(K).
\end{aligned}
\]
Hence
\[
\Var\left(\sqrt{nb^d}P_n\xi_n\right)
\longrightarrow
w_0(x_0)^2p_S(x_0)\sigma_S^2(x_0)R(K).
\]

We next verify Lyapunov's condition. By
Assumption~\ref{ass:nw-moments}, for its $q_0>0$,
\[
    \E|\xi_{i,n}|^{2+q_0}
    =O\left(b^{-d(1+q_0)}\right).
\]
If
\[
    s_n^2=n\Var(\xi_{i,n})\asymp nb^{-d},
\]
then
\[
    \frac{n\E|\xi_{i,n}|^{2+q_0}}{s_n^{2+q_0}}
    =O\left((nb^d)^{-q_0/2}\right)
    \longrightarrow0.
\]
Lyapunov's theorem therefore gives
\[
    \sqrt{nb^d}P_n\xi_n
    \dto
    N\left(
      0,
      w_0(x_0)^2p_S(x_0)\sigma_S^2(x_0)R(K)
    \right).
\]
By Lemmas~\ref{lem:denom-bias} and~\ref{lem:oracle-denom},
\[
    \widehat D_0(x_0)
    \xrightarrow{P}
    w_0(x_0)p_S(x_0).
\]
Slutsky's theorem yields
\[
    \sqrt{nb^d}\{\widehat m_0(x_0)-m_b(x_0)\}
    \dto
    N\left(
      0,
      \frac{R(K)\sigma_S^2(x_0)}{p_S(x_0)}
    \right).
\]
Finally, if $\sqrt{nb^d}\,b^s\to0$, Lemma~\ref{lem:denom-bias} gives
\[
    \sqrt{nb^d}\{m_b(x_0)-m_S(x_0)\}\to0,
\]
and the undersmoothed conclusion follows.
\end{proof}

\paragraph{Idea of proof for Theorem~\ref{thm:nw-plugin-clt}.}
The pathwise comparison used for Theorem~\ref{thm:nw-consistency-rate}
makes the plug-in estimator differ from its oracle counterpart by
$\Op(\bar\delta_{n,m})$. The assumed rate makes this difference negligible
on the $\sqrt{nb^d}$ scale, so the oracle limit transfers by Slutsky's theorem;
the additional undersmoothing condition removes the deterministic bias.

\begin{proof}\textbf{of Theorem~\ref{thm:nw-plugin-clt}.}
Equation~\eqref{eqn:appendix-nw-plugin-identity} and the argument in the proof
of Theorem~\ref{thm:nw-consistency-rate} give
\[
    \widehat m_{\widehat w,\widehat{\bm\theta}}^B(x_0)
    -\widehat m_0(x_0)
    =\Op(\bar\delta_{n,m}).
\]
Therefore, if $\sqrt{nb^d}\,\bar\delta_{n,m}\to0$,
\[
\begin{aligned}
&\sqrt{nb^d}
\left\{
  \widehat m_{\widehat w,\widehat{\bm\theta}}^B(x_0)
  -\widehat m_0(x_0)
\right\}
=\op(1).
\end{aligned}
\]
It follows that
\[
\begin{aligned}
&\sqrt{nb^d}
\left\{
  \widehat m_{\widehat w,\widehat{\bm\theta}}^B(x_0)
  -m_b(x_0)
\right\}\\
&\quad=
\sqrt{nb^d}\{\widehat m_0(x_0)-m_b(x_0)\}
+\op(1).
\end{aligned}
\]
The first conclusion follows from
Theorem~\ref{thm:nw-oracle-clt}. If additionally
$\sqrt{nb^d}\,b^s\to0$, then
\[
    \sqrt{nb^d}\{m_b(x_0)-m_S(x_0)\}\to0,
\]
and the second conclusion follows by Slutsky's theorem.
\end{proof}

\section{Empirical Details and Reproducibility}
\label{supp:implementation}

\subsection{Theorem-Linked Bounded-Overlap Simulation}
\label{supp:bounded-overlap-simulation}

The simulation in Section~\ref{sec:bounded-overlap-simulation} uses
\[
 p_S^X(x)=1,
 \qquad
 p_T^X(x)=1+0.5(2x-1),
 \qquad x\in[0,1],
\]
and Bernoulli conditional probabilities
\[
 p_S(x)=\operatorname{logit}^{-1}\!\left[3\{(2x-1)^2-1/2\}\right],
\]
\[
 p_T(x)=\operatorname{logit}^{-1}\!\left[3\{(2x-1)^2-1/2\}-0.05(2x-1)\right].
\]
We use the Brier loss and the finite class
\[
 h_\theta(x)=\frac12+\theta(2x-1),
 \qquad
 \theta=-0.500,-0.495,\ldots,0.500,
\]
so $N_H=201$.
This relatively fine grid avoids an artificial concentration of the fitted minimizer on only two or three candidate values.
The bridge grid is $\alpha=0,0.02,\ldots,1$, so $N_\alpha=51$.
We draw $n=m=2000$ observations in each of $1000$ independent replications.
The likelihood ratio is known and all target risks are evaluated by univariate quadrature.

For each geometry, the source radius is
$\rho_S=1.10D_B(P_T\Vert P_S)$ and the target radius is
\[
 \rho_T=\delta_T^{\circ,B}
 +0.25\{\delta_{B,T}(0)-\delta_T^{\circ,B}\}.
\]
The fitted predictors use the 51-point grid. For the deterministic envelope we cover the full coordinate interval $[0,1]$ at radius $1/n$. With $r_-=1/2$ and $r_+=3/2$, Lemma~\ref{lem:supp-bridge-entropy} gives Lipschitz constants $\log 3$ for KL/KL and $2(\sqrt{3/2}-\sqrt{1/2})$ for HD/HD, both below $1.1$. The midpoints of 1,100 equal subintervals form a common cover since $1.1/(2\cdot1100)=1/2000$. For each fixed normalized loss, multiplication by that loss preserves the covering radius. Hence $\mathcal N_{B,n}\le1100N_H$ for both geometries. We use the resulting common upper bound
\[
 \log\{80\cdot1100\cdot201/0.05\}=19.6841\quad\text{(rounded)}
\]
for $\kappa_{B,n}(0.05)$ in the envelope and its endpoint derivative. The known likelihood ratio gives $e_{n,m}=0$.

Table~\ref{tab:bounded-overlap-summary} reports the resulting population and
Monte Carlo quantities. The unconstrained deterministic envelope is minimized
at $\alpha=0$, whereas its minimizer over the admissible interval is
$\alpha_-(\rho_T)$ for both geometries. This distinction is essential:
Theorem~\ref{thm:radius-indexed-interior-optimum} compares the envelope only
over the radius-admissible bridge range. The Monte Carlo mean-risk minima on the fitting grid are interior at $0.64$ and $0.66$. The endpoint-slope margins are positive,
so the target endpoint is excluded even though the envelope remains
conservative at the reported sample size.

\subsection{Fashion-MNIST Target Shift and Representation}

For each target image \(j\in\{1,\ldots,m\}\) and pixel
\(k\in\{1,\ldots,784\}\), let
\(\widetilde\xi_{jk}^T\in[0,1]\) be the original intensity.  We draw
independently
\[
\varepsilon_{jk}\sim\operatorname{Laplace}(0,0.15/\sqrt2),
\]
and set
\[
\xi_{jk}^T
=
\max\{0,\min\{1,\widetilde\xi_{jk}^T+\varepsilon_{jk}\}\}.
\]
Thus the pre-clipping noise has standard deviation \(0.15\).  The
source--target split and source-trained autoencoder are fixed, while the
Laplace perturbations are independently redrawn in each of the \(100\)
replications.

The source autoencoder has encoder widths
\(784\to256\to64\to32\to10\), ReLU hidden activations, and a linear latent
layer.  The decoder reverses the widths and has a sigmoid output layer.  The
training loss is the sum of reconstruction mean squared error and binary
cross-entropy from a logistic head on the latent representation.  Training
uses full-batch Adam for \(100\) epochs at learning rate \(10^{-3}\).
Latent coordinates are standardized by the source-sample means and
divisor-\(n\) standard deviations.

\begin{figure}[t]
\centering
\begin{subfigure}[t]{0.59\textwidth}
\centering
\includegraphics[width=\linewidth]{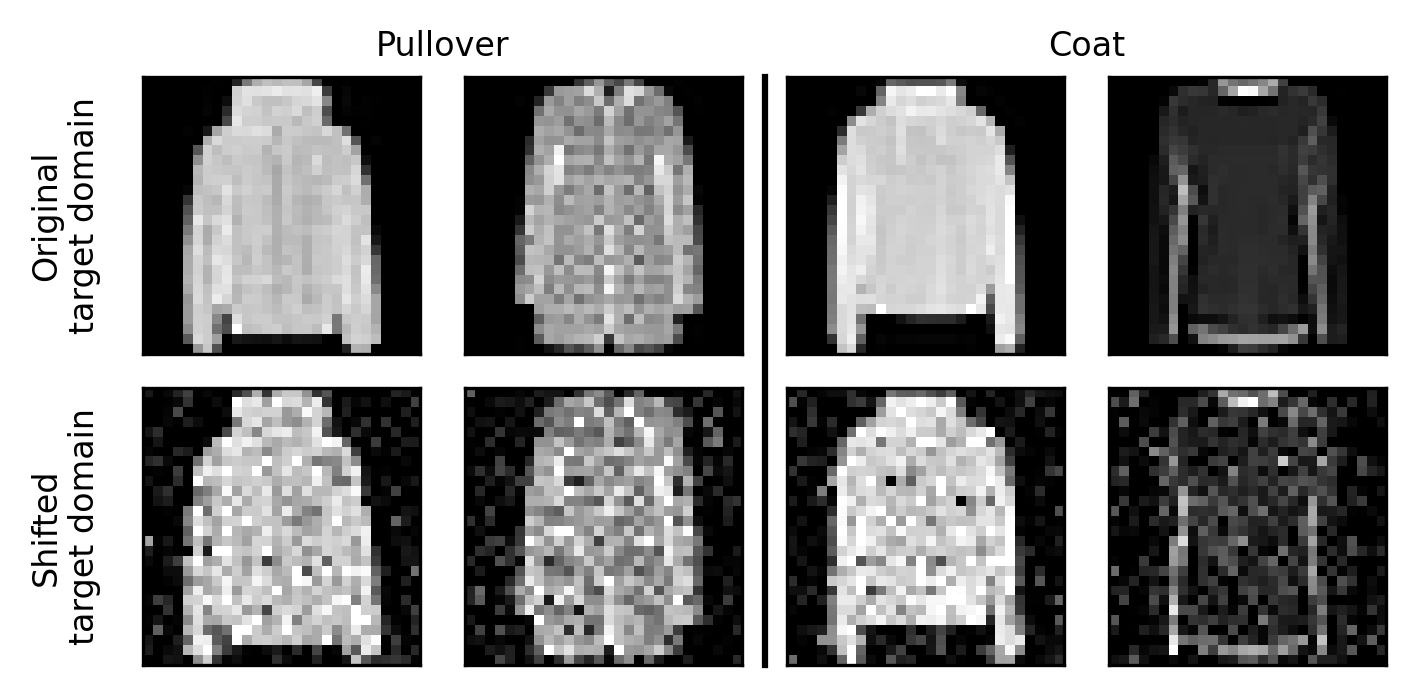}
\caption{Original and perturbed target images.}
\end{subfigure}
\hfill
\begin{subfigure}[t]{0.4\textwidth}
\centering
\includegraphics[width=\linewidth]{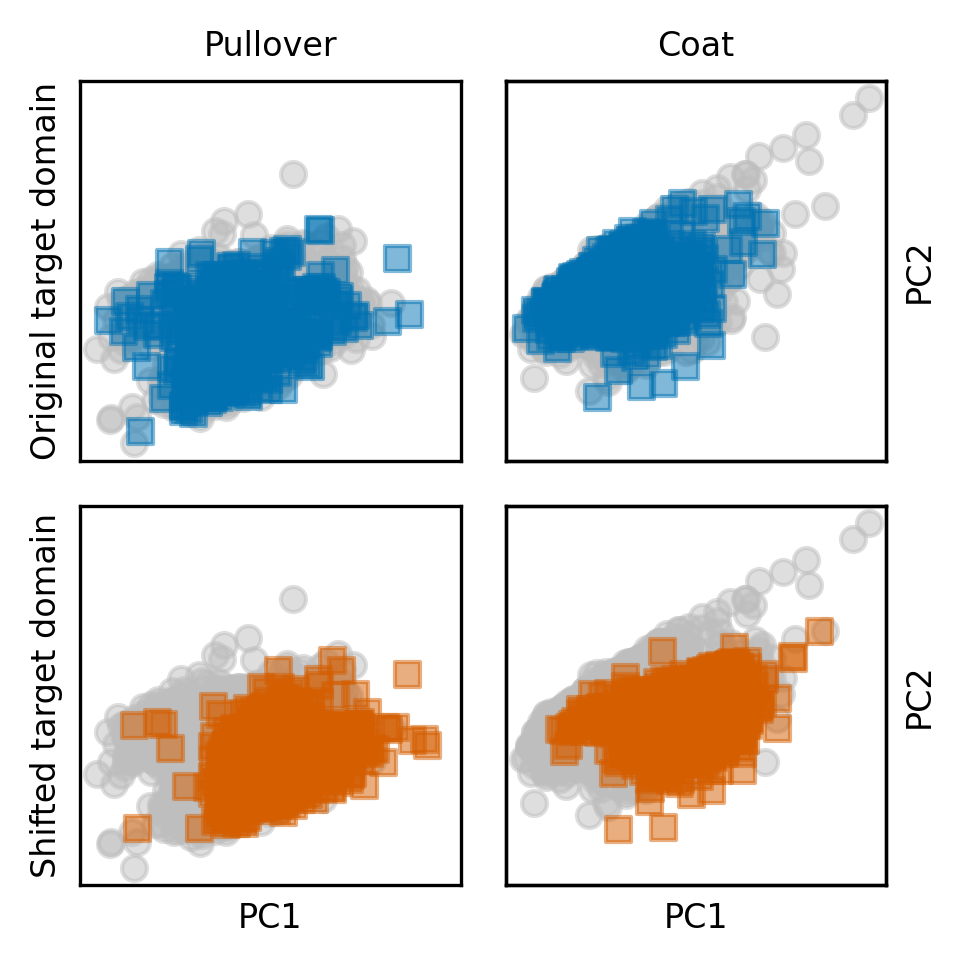}
\caption{Source, clean target, and shifted target latent representations.}
\end{subfigure}
\caption{Construction and latent-space effect of the Fashion-MNIST target
shift for one replication.}
\label{fig:supp-target-shift}
\end{figure}

\subsection{Likelihood-Ratio Estimation and Weighted Regression}

Let \(\widehat\Sigma_S,\widehat\Sigma_T\) be the empirical covariance matrices
of the standardized latent coordinates.  Gaussian KDEs use covariance
matrices
\[
H_S=h_S^2\widehat\Sigma_S,\qquad
H_T=h_T^2\widehat\Sigma_T,
\]
where
\[
h_S=\left\{\frac{n(d+2)}4\right\}^{-1/(d+4)},
\qquad
h_T=\left\{\frac{m(d+2)}4\right\}^{-1/(d+4)}.
\]
With
\[
K_H(u)
=
(2\pi)^{-d/2}|H|^{-1/2}
\exp(-u^\intercal H^{-1}u/2),
\]
the likelihood-ratio estimate is
\[
\widehat\LR_\tau(x)
=
\frac{n\sum_{j=1}^mK_{H_T}(x-X_j^T)}
     {m\sum_{i=1}^nK_{H_S}(x-X_i^S)}.
\]
Weighted Nadaraya--Watson regression uses a multivariate Epanechnikov kernel with bandwidth \(b=10\).
Every structural or loss-aware weight vector is normalized to sum to \(n\).

\subsection{Loss-Aware Iteration}

Fix \(B\) and \(\vartheta=(\alpha,\nu)\), with \(\gamma=0\) for KL/HD.
Initialize with the structural bridge vector
\(\bm r_\vartheta^{B,(0)}\).  At iteration \(t\), fit the weighted predictor,
smooth its binary cross-entropy loss, clipped to \([0,4]\), over the
source covariates to obtain \(\widehat G_{\widehat h^{(t-1)}}\), and evaluate the corresponding
plug-in response.  If
\(\widetilde{\bm r}_\vartheta^{B,(t)}\) denotes the proposed
empirical-mean-one vector, the damped update is
\[
\bm r_\vartheta^{B,(t)}
=
(1-a_t)\bm r_\vartheta^{B,(t-1)}
+a_t\widetilde{\bm r}_\vartheta^{B,(t)},
\qquad a_1=0.2.
\]
We stop when
\[
\|\bm r_\vartheta^{B,(t)}
-\bm r_\vartheta^{B,(t-1)}\|_2\le a_t.
\]
If the rule has not been met and, for \(t\ge2\),
\[
\|\bm r_\vartheta^{B,(t)}
-\bm r_\vartheta^{B,(t-2)}\|_2\le a_t,
\]
then \(a_{t+1}=a_t/2\); otherwise \(a_{t+1}=a_t\).
The maximum is \(100\) iterations, and all reported runs terminated before
this cap.

Put
\[
\widehat G_i^{(t-1)}
=
\widehat G_{\widehat h^{(t-1)}}(X_i^S),
\qquad
\widehat R_i=\widehat\LR_\tau(X_i^S).
\]
The canonical pre-normalization log responses are
\[
\ell_{i,B}^{(t)}
=
\begin{cases}
\widehat G_i^{(t-1)}/\nu+\alpha\log\widehat R_i,
&B=\mathrm{KL},\\[3pt]
2\log\{1-\alpha+\alpha\sqrt{\widehat R_i}\}
-2\log\{\nu-\widehat G_i^{(t-1)}\},
&B=\mathrm{HD},\\[3pt]
2\widehat G_i^{(t-1)}/\nu
+2W_0\!\left[
\frac{\alpha}{1-\alpha}\sqrt{\widehat R_i}
e^{-\widehat G_i^{(t-1)}/\nu}
\right],
&B=\mathrm{KH}.
\end{cases}
\]
The KL/HD value at \(\alpha=0\) is defined continuously.  Set
\[
q_{i,B}^{(t)}
=
\min\{\ell_{i,B}^{(t)},\log 10^8\},
\qquad
\log\widetilde r_{i,B}^{(t)}
=
q_{i,B}^{(t)}+\log n
-\log\sum_{k=1}^ne^{q_{k,B}^{(t)}}.
\]
Damping is then applied on the ordinary weight scale.
The configurations selected from the diagnostic plots in Figure~\ref{fig:fmnist-sensitivity-nu} have zero capped fraction in the diagnostic replication.

\subsection{Loss-Scale Grid and Reported Diagnostics}

For KL/KL and KL/HD, the candidate grid is \(\nu\in\{10^{-2+0.2k}:k=0,\ldots,7\}\), whereas for HD/HD it is \(\nu\in\{2/k+10^{-8}:k=1,1.5,2,2.5,3\}\).

For normalized weights \(r_i\), the plug-in effective sample size and the summed entropy are
\begin{equation}
\label{eqn:fmnist-ess-and-summed-entropy}
\operatorname{ESS}
=
\frac{(\sum_i r_i)^2}{\sum_i r_i^2},\qquad
\operatorname{ENT}=
-\sum_{j=1}^m
\{\widehat p_j\log\widehat p_j
 +(1-\widehat p_j)\log(1-\widehat p_j)\},
\end{equation}
respectively, with the convention that $0 \cdot\log0=0$.
The predicted Pullover proportion is
\(m^{-1}\sum_j\mathbf 1\{\widehat p_j\ge1/2\}\).
Target labels enter only in the reported accuracy.

From the diagnostic replication shown in Figure~\ref{fig:fmnist-sensitivity-nu} we can see that for KL/KL with $\nu \leq 0.025$ and KL/HD with $\nu \leq 0.063$ the vast majority of weights are capped, hence we ignore these configurations.
Similarly, for HD/HD we ignore configurations with
\(
\nu\le
\max_{i,t}\widehat G_{\widehat h^{(t)}}(X_i^S),
\)
for one or more $\alpha$ values in the grid (i.e., we consider only $\nu \geq 1$).
Among the remaining configurations, we select $\nu$ for which the ESS and entropy appear stable and the proportion of pullovers degrades the slowest.
These criteria yield \(\nu\approx0.04\), \(1+10^{-8}\), and \(0.1\) for KL/KL, HD/HD, and KL/HD, respectively.

\begin{figure}[t]
    \centering
    \includegraphics[width=0.9\linewidth]{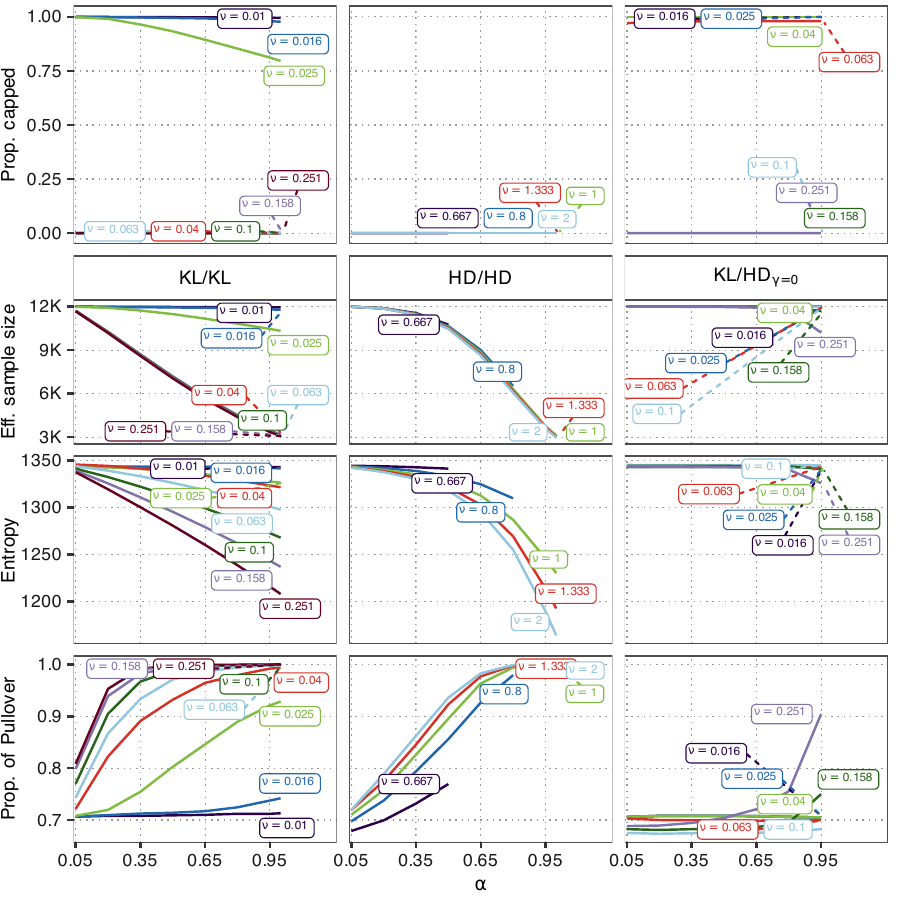}
    \caption{Exploratory sensitivity of the loss-aware refinement to $\nu$, plotted against $\alpha$ for KL/KL (left), HD/HD (center), and KL/HD with $\gamma=0$ (right), using the single target-shift replication used for tuning.
    Rows report the proportion of capped source weights, the plug-in effective sample size of the normalized source weights, the summed binary entropy of the fitted target probabilities, and the proportion of target images classified as Pullover.
    The HD/HD curves terminate when the interior-response check $\nu>\max_{i,t}\widehat G_{\widehat h^{(t)}}(X_i^S)$ fails.
    The retained values in these panels were $\nu\approx0.04$, $\nu=1+10^{-8}$, and $\nu=0.1$, respectively.}
    \label{fig:fmnist-sensitivity-nu}
\end{figure}

The theoretical analysis conditions on the feature representation. The experiment holds the source-trained encoder fixed across target-shift replications.
\FloatBarrier

\section{Dual Interpretation of the Symmetric Single-Ball Representation}
\label{supp:single-ball-dual}

Above the symmetric doubly anchored ambiguity sets are identified as intersections of KL or Hellinger balls centered on the corresponding barycentric bridge.
The robust value is the minimum over these effective single-ball problems. When the sum of the optimal dual multipliers is positive, the target-side multiplier divided by that sum gives a minimizing bridge coordinate.

For a probability law \(P\sim P_S\) and \(r\in L_+^1(P_S)\), let
\[
  \mathcal G_P(r)
  =
  \int
  \phi\!\left(\frac{\d(rP_S)}{\d P}\right)\dd P,
  \qquad
  \mathcal A(r)=\E_{P_S}[r].
\]
The generator \(\phi\) will be clear from context.

\begin{lemma}[Finite-measure barycentric identities]
\label{lem:supp-finite-measure-barycentric}
Let \(\alpha\in[0,1]\) and use the centers and normalizers from
Theorem~\ref{thm:symmetric-single-ball-representation}.

\begin{enumerate}
\item For KL/KL, for every \(r\in L_+^1(P_S)\),
\begin{equation}\label{eqn:supp-kl-finite-measure}
 (1-\alpha)\mathcal G_{P_S}(r)
 +\alpha\mathcal G_{P_T^\circ}(r)
 =
 \mathcal G_{P_\alpha^{\mathrm{KL}}}(r)
 -\Lambda(\alpha)\mathcal A(r),
\end{equation}
as an identity in \([0,\infty]\), with the natural extended-value
interpretation.
\item For HD/HD,
\begin{equation}\label{eqn:supp-hd-finite-measure}
 (1-\alpha)\mathcal G_{P_S}(r)
 +\alpha\mathcal G_{P_T^\circ}(r)
 =
 \sqrt{\ZH(\alpha)}\,
 \mathcal G_{P_\alpha^{\mathrm{HD}}}(r)
 +\{1-\sqrt{\ZH(\alpha)}\}\{\mathcal A(r)+1\}.
\end{equation}
\end{enumerate}
\end{lemma}

\begin{proof}
Put $u=\LR$. For KL let $v=u^\alpha e^{-\Lambda(\alpha)}$ and $\phi=\phi_{\mathrm{KL}}$. The pointwise identity
\[
 (1-\alpha)\phi(r)+\alpha u\phi(r/u)
 = v\phi(r/v)-r\Lambda(\alpha)+(1-\alpha)+\alpha u-v
\]
follows by expanding $t\log t-t+1$, including $r=0$ by continuity. The terms after $v\phi(r/v)$ are integrable, and $\E_{P_S}u=\E_{P_S}v=1$. Integrating thus gives~\eqref{eqn:supp-kl-finite-measure} even when a divergence is infinite, with the zero-multiplier convention at the endpoints.

For HD, $\mathcal G_P(r)=\mathcal A(r)+1-2\E_{P_S}\sqrt{r\,\d P/\d P_S}$. Substitution of
$(1-\alpha)+\alpha\sqrt u=\sqrt{\ZH(\alpha)}\sqrt{\d P_\alpha^{\mathrm{HD}}/\d P_S}$
gives~\eqref{eqn:supp-hd-finite-measure}.
\end{proof}

Use the single-ball values $\Psi_\alpha^{\mathrm{KL}}$ and $\Psi_\alpha^{\mathrm{HD},\mathrm{ac}}$ defined in Section~\ref{sec:symmetric-single-ball-geometry}. Whenever the two-ball intersection is nonempty, every effective radius is nonnegative by Theorem~\ref{thm:symmetric-single-ball-representation}.

\begin{proposition}[Dual profiling over the effective balls]
\label{prop:supp-dual-profiling}
\leavevmode\par\noindent
Under the hypotheses of
Theorem~\ref{thm:general-dual-representation},
\begin{align}
 \Psi_{S,T}^{\mathrm{ac}}(h)
 &=
 \min_{\alpha\in[0,1]}
 \Psi_\alpha^{\mathrm{KL}}(h),
 &&\text{for KL/KL},\label{eqn:supp-kl-value-profile}\\
 \Psi_{S,T}^{\mathrm{ac}}(h)
 &=
 \min_{\alpha\in[0,1]}
 \Psi_\alpha^{\mathrm{HD},\mathrm{ac}}(h),
 &&\text{for HD/HD}.\label{eqn:supp-hd-value-profile}
\end{align}
For KL/KL, the left side also equals the full robust value.  If an optimal
dual pair satisfies
\(\lambda_S^\star+\lambda_T^\star>0\), then a minimizing coordinate is
\[
 \alpha^\star
 =
 \frac{\lambda_T^\star}
      {\lambda_S^\star+\lambda_T^\star}.
\]
\end{proposition}

\begin{proof}
Write $D_S(Q)=D_{\phi_S}(Q\Vert P_S)$ and $D_T(Q)=D_{\phi_T}(Q\Vert P_T^\circ)$.
By Theorem~\ref{thm:symmetric-single-ball-representation}, the dominated intersection is contained in every effective ball on the class $Q\ll P_S$. Hence $\Psi_{S,T}^{\mathrm{ac}}(h)\le\Psi_\alpha^B(h)$ for every $\alpha$, where $B$ denotes the relevant symmetric geometry.

Take an optimal dual triple from Theorem~\ref{thm:general-dual-representation}. For $Q\ll P_S$, evaluate the pointwise conjugate at $r=\d Q/\d P_S$ and integrate. Because $\E_{P_S}r=1$, the normalization multiplier cancels, giving
\[
 \E_Q L_h \le \Psi_{S,T}^{\mathrm{ac}}(h)
 +\lambda_S^\star\{D_S(Q)-\rho_S\}
 +\lambda_T^\star\{D_T(Q)-\rho_T\}.
\]
An inactive divergence has zero contribution, including when it is infinite. If $\nu^\star=\lambda_S^\star+\lambda_T^\star>0$, set $\alpha^\star=\lambda_T^\star/\nu^\star$. On the effective ball at $\alpha^\star$, the identities in Theorem~\ref{thm:symmetric-single-ball-representation} give
\[
 (1-\alpha^\star)D_S(Q)+\alpha^\star D_T(Q)
 \le (1-\alpha^\star)\rho_S+\alpha^\star\rho_T.
\]
Thus $\E_Q L_h\le\Psi_{S,T}^{\mathrm{ac}}(h)$ throughout that ball on the dominated class. Taking the supremum gives equality at $\alpha^\star$. In the KL case every member of the effective ball is dominated because its KL radius is finite. If $\nu^\star=0$, the integrated inequality gives the same bound for every dominated $Q$, so equality holds at every coordinate.
\end{proof}

\begin{corollary}[Consequences of the profiling identity]
\label{cor:supp-single-ball-consequences}
\leavevmode\par\noindent
Under the KKT hypotheses of Proposition~\ref{prop:loss-adversarial-bridges}, the one-ball equations recover the symmetric loss-aware densities, namely an exponential tilt of
\(P_\alpha^{\mathrm{KL}}\) for KL/KL and a reciprocal-square tilt of
\(P_\alpha^{\mathrm{HD}}\) for HD/HD.  Moreover,
\(\alpha=1\) is the target-ball problem, while for \(\alpha<1\) the effective
radius contains the source-side slack.  Thus the source radius can lower the
loss-aware robust value even when its own constraint admits every structural center.
\end{corollary}

\clearpage

\bibliography{references}
\end{document}